\documentclass[10pt]{article}

\usepackage[ansinew]{inputenc}
\usepackage[none]{hyphenat}
\usepackage[top=2.0cm,bottom=2.5cm,left=2.0cm,right=2.0cm]{geometry}
\usepackage{amsmath,amssymb,amsfonts,latexsym,amsthm,bm,dsfont,mathrsfs,breqn,cancel}

\usepackage{algorithm}
\usepackage[noend]{algpseudocode}
\makeatletter
\def\BState{\State\hskip-\ALG@thistlm}

\usepackage{graphicx,graphics,float}
\usepackage[normal]{subfigure}

\usepackage{ctable,enumerate,color,multicol,colortbl}
\usepackage{array}
\usepackage{tabulary}
\newcolumntype{C}[1]{>{\centering\arraybackslash}p{#1}}

\definecolor{Gray}{gray}{0.9}
\newcolumntype{g}{>{\columncolor{Gray}}c}
\newcolumntype{R}{>{\columncolor{pink}}c}

\usepackage{cite,natbib,url}
\usepackage[toc,page]{appendix}
\usepackage{multirow}

\usepackage[pdftex,
			bookmarks, bookmarksnumbered=true, bookmarksopen=true,
			colorlinks, citecolor=blue, filecolor=black, linkcolor=black, urlcolor=green,
			pdfauthor={López-Lopera, Andrés Felipe}, pdftitle={PhDEMSE},
			pdftoolbar=true, pdfmenubar=true, pdffitwindow = true
			]{hyperref}

\theoremstyle{plain}
\newtheorem{lemma}{Lemma}[section]
\newtheorem{proposition}{Proposition}[section]
\newtheorem{Condition}{Condition}[section]

\newcommand{\Bc}{\textbf{c}}

\newcommand{\BK}{\textbf{K}}
\newcommand{\BR}{\textbf{R}}

\newcommand{\Bv}{\textbf{v}}
\newcommand{\Bx}{\textbf{x}} 
\newcommand{\By}{\textbf{y}} \newcommand{\BY}{\textbf{Y}}

\newcommand{\Bu}{\textbf{\textit{u}}} \newcommand{\Bl}{\textbf{\textit{l}}}

\newcommand{\Bzero}{\boldsymbol{0}}

\newcommand{\Beta}{\boldsymbol{\eta}}
\newcommand{\Bmu}{\boldsymbol{\mu}}
\newcommand{\Bnu}{\boldsymbol{\nu}}
\newcommand{\Bxi}{\boldsymbol{\xi}}
\newcommand{\Btheta}{\boldsymbol{\theta}} \newcommand{\BTheta}{\boldsymbol{\Theta}}
\newcommand{\BPhi}{\boldsymbol{\Phi}}
\newcommand{\BSigma}{\boldsymbol{\Sigma}}
\newcommand{\BGamma}{\boldsymbol{\Gamma}}
\newcommand{\BLambda}{\boldsymbol{\Lambda}}

\newcommand{\realset}[1]{\mathds{R}^{#1}}
\newcommand{\normrnd}[2]{\mathcal{N}\left({#1,#2}\right)}
\newcommand{\tnormrnd}[4]{\mathcal{TN}\left({#1,#2,#3,#4}\right)}

\newcommand{\expect}[1]{\mathds{E} \left\{ #1 \right\}}

\newcommand{\var}[1]{\operatorname{var} \left\{ #1 \right\}}

\definecolor{green}{rgb}{0, 0.5, 0}

\usepackage[none]{hyphenat}

\usepackage{authblk}

\RequirePackage[capitalize,nameinlink]{cleveref}[0.19]

\crefname{section}{section}{sections}
\crefname{subsection}{subsection}{subsections}

\Crefname{figure}{Figure}{Figures}

\crefformat{equation}{\textup{#2(#1)#3}}
\crefrangeformat{equation}{\textup{#3(#1)#4--#5(#2)#6}}
\crefmultiformat{equation}{\textup{#2(#1)#3}}{ and \textup{#2(#1)#3}}
{, \textup{#2(#1)#3}}{, and \textup{#2(#1)#3}}
\crefrangemultiformat{equation}{\textup{#3(#1)#4--#5(#2)#6}}%
{ and \textup{#3(#1)#4--#5(#2)#6}}{, \textup{#3(#1)#4--#5(#2)#6}}{, and \textup{#3(#1)#4--#5(#2)#6}}

\Crefformat{equation}{#2Equation~\textup{(#1)}#3}
\Crefrangeformat{equation}{Equations~\textup{#3(#1)#4--#5(#2)#6}}
\Crefmultiformat{equation}{Equations~\textup{#2(#1)#3}}{ and \textup{#2(#1)#3}}
{, \textup{#2(#1)#3}}{, and \textup{#2(#1)#3}}
\Crefrangemultiformat{equation}{Equations~\textup{#3(#1)#4--#5(#2)#6}}%
{ and \textup{#3(#1)#4--#5(#2)#6}}{, \textup{#3(#1)#4--#5(#2)#6}}{, and \textup{#3(#1)#4--#5(#2)#6}}

\crefdefaultlabelformat{#2\textup{#1}#3}

\begin{document}

%%Head--------------------------------------------------------------------------------------------
\title{
	\vspace{-2cm}
	\LARGE
	\textbf{Finite-dimensional Gaussian approximation with linear inequality constraints}
}

\author[1]{Andr\'es F. L\'opez-Lopera}
\author[2]{Fran\c{c}ois Bachoc}
\author[1,3]{Nicolas Durrande}
\author[1]{Olivier Roustant}
\affil[1]{Mines Saint-\'Etienne, UMR CNRS 6158, LIMOS, F-42023 Saint-\'Etienne, France.}
\affil[2]{Institut de Math\'ematiques de Toulouse, Université Paul Sabatier - Toulouse III, France.}
\affil[3]{PROWLER.io, 66-68 Hills Road, Cambridge, CB2 1LA, UK.}

\date{\vspace{-35pt}}

\maketitle

%%Introduction------------------------------------------------------------------------------------
\sloppy
\begin{abstract}
	Introducing inequality constraints in Gaussian process (GP) models can lead to more realistic uncertainties in learning a great variety of real-world problems. We consider the finite-dimensional Gaussian approach from Maatouk and Bay (2017) which can satisfy inequality conditions everywhere (either boundedness, monotonicity or convexity). Our contributions are threefold. First, we extend their approach in order to deal with general sets of linear inequalities. Second, we explore several Markov Chain Monte Carlo (MCMC) techniques to approximate the posterior distribution. Third, we investigate theoretical and numerical properties of the constrained likelihood for covariance parameter estimation. According to experiments on both artificial and real data, our full framework together with a Hamiltonian Monte Carlo-based sampler provides efficient results on both data fitting and uncertainty quantification.
\end{abstract}

\section{Introduction}
%1. Statements about the field of research to provide the reader with a setting or context for the problem to be investigated and to claim its centrality or importance.
Gaussian processes (GPs) are one of the most famous non-parametric Bayesian frameworks for modelling stochastic processes. In principle, GP models place prior distributions over function spaces, and prior assumptions (e.g. smoothness, stationarity, sparsity) are encoded in covariance functions \citep{Rasmussen2005GP,Paciorek04nonstationarycovariance,Snelson2006SparseGP}. Because GP provides a well-founded approach to learning, its properties have been explored in many decision tasks in regression (Kriging) and classification problems \citep{Rasmussen2005GP,Nickisch2008GPClassification}. Computer science, engineering, physics, biology, and neuroscience are some fields where GP models have been applied successfully \citep{Rasmussen2005GP,Murphy2012ML}.

%2. More specific statements about the aspects of the problem already studied by other researchers, laying a foundation of information already known.
Despite the reliable performance of GPs, they provide less realistic uncertainties when physical systems satisfy inequality constraints \citep{Maatouk2016GPineqconst,DaVeiga2012GPineqconst,Golchi2015MonotoneEmulation}. Quantifying properly the uncertainties is crucial for understanding real-world phenomena. For example, in nuclear safety criticality assessment, experimental settings typically demand expensive and risky procedures to evaluate neutron productions. Hence, emulators are required to infer these production rates and should assume a priori that the output is positive and usually monotonic with respect to a given set of input parameters. In this sense, to obtain more accurate predictions, both conditions have to be considered in the uncertainty quantification. Other test cases where data exhibit specific inequality constraints are given in computer networking (monotonicity) \citep{Golchi2015MonotoneEmulation}, social system analysis (monotonicity) \citep{Riihimaki2010GPwithMonotonicity}, and econometrics (monotonicity or positivity) \citep{Cousin2016KrigingFinancial}.

Several studies have shown that including inequality constraints in GP frameworks can lead to more realistic uncertainty quantifications in learning from real data \citep{DaVeiga2012GPineqconst,Golchi2015MonotoneEmulation,Riihimaki2010GPwithMonotonicity}. In most of the cases, it is assumed that the inequalities are satisfied on a finite set of input locations. Then, the posterior distribution is approximated given those constrained inputs. To the best of our knowledge, the framework from \citep{Maatouk2016GPineqconst} is the only Gaussian approach proposed in the literature which satisfies specific inequalities everywhere in the input space. There, the GP samples are approximated in the finite-dimensional space of functions such as piecewise linear functions. It is shown in \citep{Bay2016KimeldorfWahba} that the posterior mode converges to the one provided by thin plate splines. This approach has been applied on several real-data (e.g. econometrics, geostatistics) \citep{Maatouk2016GPineqconst,Cousin2016KrigingFinancial}, resulting in more realistic uncertainties than unconstrained Kriging.

The framework proposed in \citep{Maatouk2016GPineqconst} still presents some limitations. First, the focus is on either boundedness, monotonicity or convexity conditions. Second, the proposed rejection sampling method for estimating the posterior \citep{Maatouk2016RSM} results in costly computations when either the order of the finite approximation increases or the inequality constraints become more complex. Third, the proposed leave-one-out (LOO) technique for parameter estimation \citep{Maatouk2015CrossVal}, restricts the optimal values to be on a finite grid of possible values, and provides the same estimation of correlation parameters as for unconstrained GP parameters. In order to address these limitations, our contributions are threefold. First, we extend the framework to deal with general sets of linear inequality constraints. Second, we evaluate efficient Markov Chain Monte Carlo (MCMC) algorithms that can be used to approximate the posterior distribution. Third, we investigate theoretical and numerical properties of the conditional likelihood for covariance parameter estimation. According to experiments on both artificial and real data, the resulting framework provides efficient results on both data fitting and uncertainty quantification.

%5. Optional statement(s) that give a positive value or justification for carrying out the study.
This paper is organised as follows. In \Cref{sec:GPwithICs}, we briefly describe GP modelling with inequality constraints. In \Cref{sec:GPwithICLinear}, continuing with the finite-dimensional approach from \citep{Maatouk2016GPineqconst}, we propose a general formulation to deal with sets of linear inequalities. In \Cref{sec:simPosterior}, we apply several MCMC techniques to approximate the posterior distribution, and we compare their performances with respect to exact Monte Carlo (MC) algorithms. In \Cref{sec:ml}, we write the conditional likelihood for the covariance parameter estimation, providing theoretical and empirical properties. In \Cref{sec:2Dapp}, we assess our framework in two-dimensional Kriging tasks. Finally, in \Cref{sec:conclusions}, we summarize the conclusions, as well as the potential future works.

\section{Gaussian process modelling with inequality constraints}
\label{sec:GPwithICs}
\subsection{Finite-dimensional approximation}
\label{subsec:finApprox}
Let $Y$ be a zero-mean Gaussian process (GP) on $\realset{}$ with covariance function $k$. Consider $x \in \mathcal{D}$ with compact input space $\mathcal{D} = [0, 1]$, and a set of knots $t_1, \cdots, t_m \in \realset{}$. For simplicity, we will consider equally-spaced knots $t_j = j \Delta_m$ with $\Delta_m = 1/(m-1)$, but this assumption can be relaxed. Then, define a finite-dimensional GP, denoted by $Y_m$, as the piecewise linear interpolation of $Y$ at knots $t_1, \cdots, t_m$:
\begin{align}
Y_m (x) = \sum_{j=1}^{m} Y(t_j) \phi_j (x),
\label{eq:finApprox}
\end{align}
where $\phi_1 \cdots, \phi_m$ are hat basis functions given by
\begin{align}
\phi_j (x) :=
\begin{cases}
1 - \left|\frac{x - t_j}{\Delta_m}\right| & \mbox{if } \left|\frac{x - t_j}{\Delta_m}\right| \leq 1,\\
0 & \mbox{otherwise}.
\end{cases}
\label{eq:hatbasisfun}
\end{align}
We illustrate the finite-dimensional representation of \Cref{eq:finApprox} in \Cref{fig:toyExample1a} 
for a deterministic function that satisfies two types of inequality constraints: boundedness and monotonicity (non-decreasing). % conditions.
\begin{figure}
	\centering
	\includegraphics[width=0.40\textwidth]{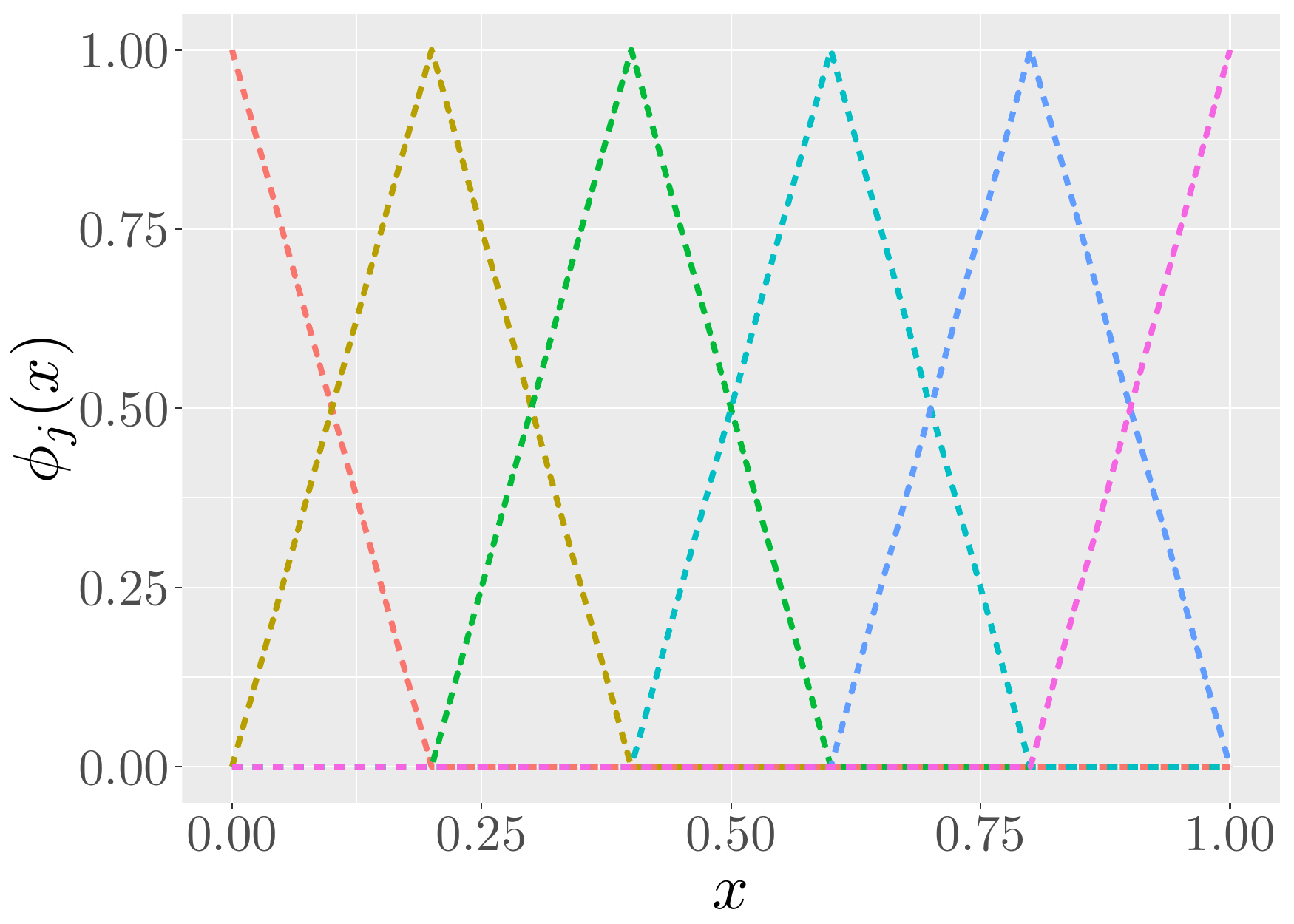} \hspace{20pt}
	\includegraphics[width=0.40\textwidth]{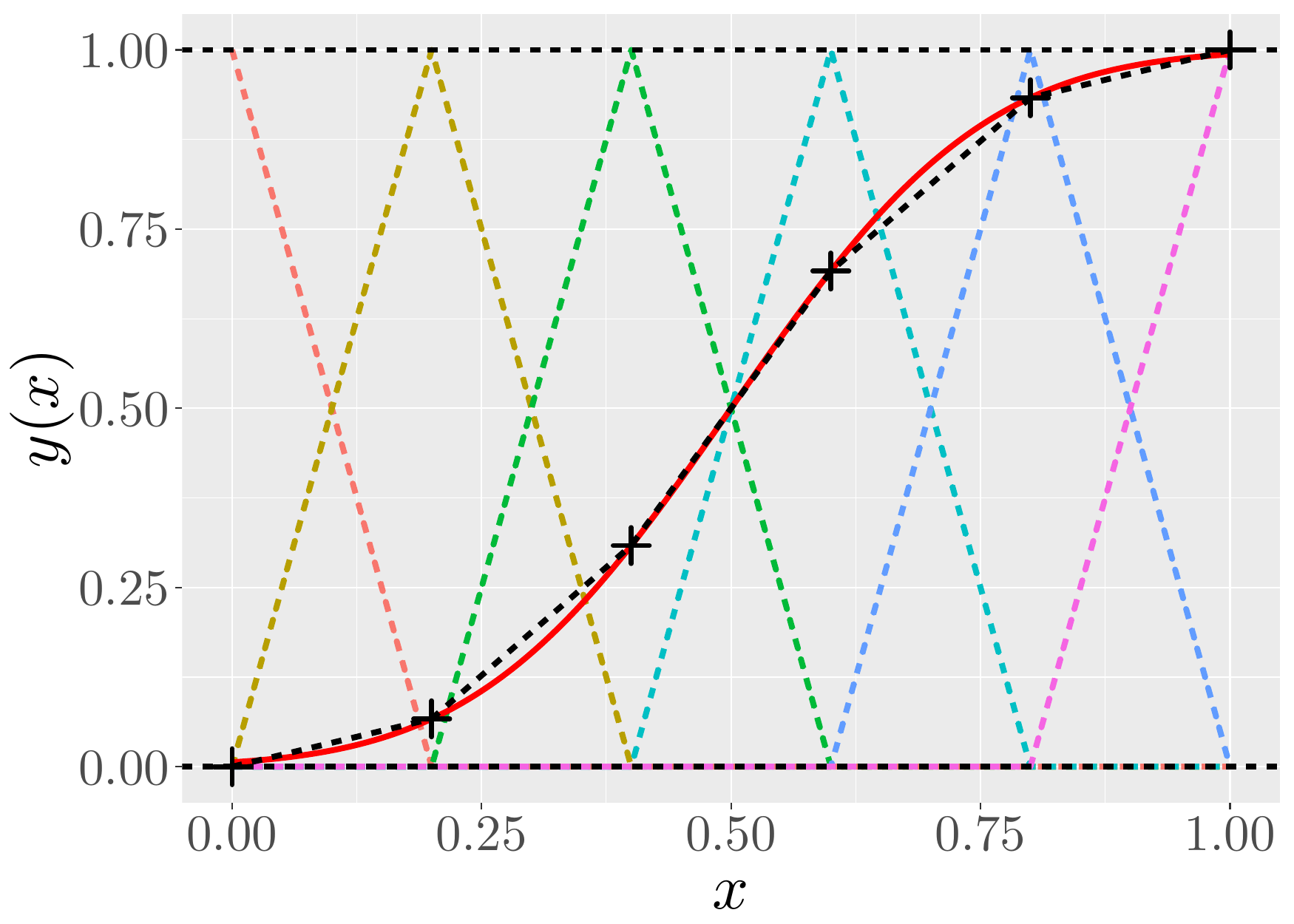}
	\caption{Illustration of the finite-dimensional approximation of \Cref{eq:finApprox}. (Left) Hat functions $\phi_j$ for $j = 1, \cdots, 6$. (Right) Approximation of the function $y(x) = \Phi(\frac{x-0.5}{0.2})$, where $\Phi$ is the standard cumulative normal distribution. Solid red and dashed black lines are the function $y$, and its finite approximation with six knots given by black crosses, respectively. Horizontal  black dashed lines denote the bounds.}
	\label{fig:toyExample1a}
\end{figure}
Now, let $\xi_j := Y(t_j)$ for $j =1, \cdots, m$. We aim at computing the distribution of $Y_m$ conditionally on $Y_m \in \mathcal{E}$ where $\mathcal{E}$ is a convex set of functions defined by some inequality constraints. For instance, we may have
\begin{equation}
\mathcal{E} = \mathcal{E}_\kappa := \begin{cases}
\left\{ f \in C(\mathcal{D}, \realset{}) \ \mbox{s.t.} \ \ell \leq f(x) \leq u, \ \forall x \in \mathcal{D} \right\} & \mbox{if } \kappa = 0, \\
\left\{ f \in C(\mathcal{D}, \realset{}) \ \mbox{s.t.} \ f \mbox{ is non-decreasing} \right\} & \mbox{if } \kappa = 1, \\
\left\{ f \in C(\mathcal{D}, \realset{}) \ \mbox{s.t.} \ f \mbox{ is convex} \right\} & \mbox{if } \kappa = 2,
\end{cases}
\label{eq:convexsetFun}
\end{equation}
which corresponds to boundedness, monotonicity, and convexity constraints. The benefit of using hat functions and the finite-dimensional approximation $Y_m$, is that satisfying the inequality conditions $Y_m(x) \in \mathcal{E}$, for all $x \in $ $\mathcal{D}$, is equivalent to satisfying only a finite number of inequality constraints \citep{Maatouk2016GPineqconst}. More precisely, for many natural choices of $\mathcal{E}$, we have
\begin{equation}
Y_m \in \mathcal{E} \quad \Leftrightarrow \quad \Bxi \in \mathcal{C},
\label{eq:convexEqui}
\end{equation}
where $\mathcal{C}$ is a convex set of $\realset{m}$ and %with 
$\Bxi = [\begin{smallmatrix} \xi_1, & \cdots, & \xi_m \end{smallmatrix}]^\top$. For instance, for the convex set $\mathcal{E}_\kappa$ of \Cref{eq:convexsetFun}, we have 
\begin{align}
\mathcal{C} = \mathcal{C}_\kappa := \begin{cases}
\left\{\Bc \in \realset{m}; \ \forall \ j = 1, \cdots, m \; : \; \ell \leq c_j \leq u \right\} & \mbox{if }  \kappa = 0, \\
\left\{\Bc \in \realset{m}; \ \forall \ j = 2, \cdots, m \; : \; c_j \geq c_{j-1} \right\} & \mbox{if } \kappa = 1,  \\
\left\{\Bc \in \realset{m}; \ \forall \ j = 3, \cdots, m \; : \; c_j - c_{j-1} \geq c_{j-1} - c_{j-2} \right\} & \mbox{if } \kappa = 2.
\end{cases}
\label{eq:convexsetKnots}
\end{align}

\subsection{Conditioning with interpolation and inequality constraints}
\label{subsec:GPwithICMaatouk}
Consider the finite dimensional representation of the GP as in \Cref{eq:finApprox}, given the interpolation and inequality constraints
\begin{equation}
Y_m(x) = \sum_{j=1}^{m} \xi_j \phi_{j}(x), \quad \mbox{s.t.} \quad \begin{cases} Y_m(x_i) = y_i & \mbox{(interpolation conditions)}, \\ Y_m \in \mathcal{E} & \mbox{(inequality conditions)}, \end{cases}
\label{eq:constrFinApprox}
\end{equation}
where $x_i \in \mathcal{D}$ and $y_i \in \realset{}$ for $i = 1, \cdots, n$. Given a design of experiment (DoE) $\Bx = [\begin{smallmatrix} x_1, & \cdots, & x_n \end{smallmatrix}]^\top$, we have matricially:
\[
\BY_m = \begin{bmatrix} Y_m(x_1), & \cdots, & Y_m(x_n) \end{bmatrix}^\top  = \BPhi \Bxi,
\]
where $\BPhi$ is the $n \times m$ matrix defined by $\BPhi_{i,j} = \phi_j (x_i)$. Let $\By = [\begin{smallmatrix} y_1, & \cdots, & y_n \end{smallmatrix}]^\top$ be a realization of $\BY_m$ as in \Cref{eq:constrFinApprox}. From \Cref{eq:convexEqui}, the conditional distribution of $Y_m$, under the inequality constraints $Y_m \in \mathcal{E}$ and interpolation conditions $Y_m(x_i) = y_i$ for $i = 1, \cdots, n$, can be obtained from the conditional distribution of $\Bxi$ given $\Bxi \in \mathcal{C}$ and $\BPhi \Bxi = \By$ (see \Cref{eq:convexEqui}). 

Observe that the vector $\Bxi$ of the values at the knots is a zero-mean Gaussian vector with covariance matrix $\BGamma= (k(t_i,t_j))_{1 \leq i,j \leq m}$. Then, the distribution of $\Bxi$ given both interpolation and inequality conditions is truncated multinormal:
\begin{equation}
\Bxi \sim \normrnd{\Bzero}{\BGamma} \quad \mbox{s.t.} \quad  \begin{cases} \BPhi \Bxi = \By & \mbox{(interpolation conditions)}, \\ \Bxi \in \mathcal{C} & \mbox{(inequality conditions)}, \end{cases}
\label{eq:trGP1}
\end{equation}
with $\mathcal{C}$ as in \Cref{eq:convexEqui}. For sampling purposes (see \cref{alg:tmKriging}), we need to compute the posterior mode which is given by the maximum of the probability density function of the posterior, i.e. $\Bmu_\Bxi^\ast = \min \{ \Bxi^\top \BGamma^{-1} \Bxi| \ \BPhi \Bxi = \By, \Bxi \in \mathcal{C}\}$ (maximum a posteriori, MAP). Notice that $\Bmu_\Bxi^\ast$ converges uniformly to the solution provided by thin plate splines when $m \to \infty$ \citep{Bay2016KimeldorfWahba}. More details and theoretical properties are provided in \citep{Maatouk2016GPineqconst,Bay2016KimeldorfWahba}. 

\Cref{fig:toyExample1b} shows different Gaussian models for the example of \cref{fig:toyExample1a}. We used a squared exponential (SE) covariance function with parameters $(\sigma^2 = 1, \; \theta = 0.2)$,\footnote{SE covariance function: $k_\Btheta(x-x') = \sigma^2 \exp\left\{-\frac{(x-x')^2}{2 \theta^2} \right\}$ with $\Btheta = (\sigma^2,\theta)$.} and we fixed $m=100$. The posterior distribution was approximated via Hamiltonian Monte Carlo (HMC) \citep{Pakman2014Hamiltonian}. From  \cref{subfig:toyExample1b2,subfig:toyExample1b3}, we observe that including the inequality constraints in the conditional distribution provides smaller confidence intervals compared to the ones given by the unconstrained GP. However, they do not satisfy both the boundedness and monotonicity conditions exhibited by the function $y$. On the other hand, from \cref{subfig:toyExample1b4}, %we note that 
imposing both conditions leads to a more accurate prediction and more realistic confidence intervals. Later in \Cref{sec:GPwithICLinear}, we will detail how to obtain the results of \cref{subfig:toyExample1b4}.
\begin{figure}[t!]
	\centering
	\subfigure[\scriptsize \label{subfig:toyExample1b1} Unconstrained.]{\includegraphics[width=0.40\textwidth]{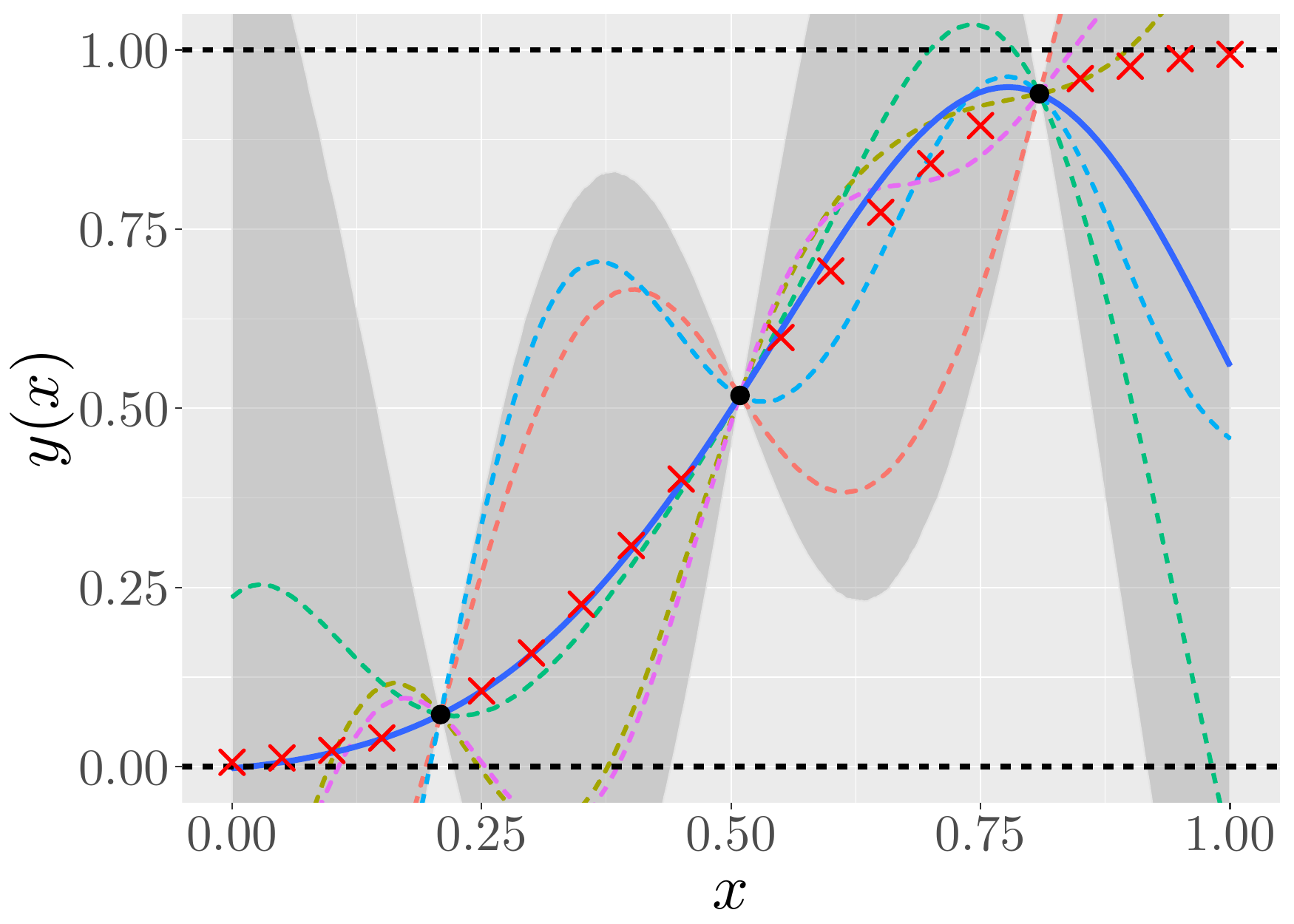}} \hspace{20pt}
	\subfigure[\scriptsize \label{subfig:toyExample1b2} Boundedness.]{\includegraphics[width=0.40\textwidth]{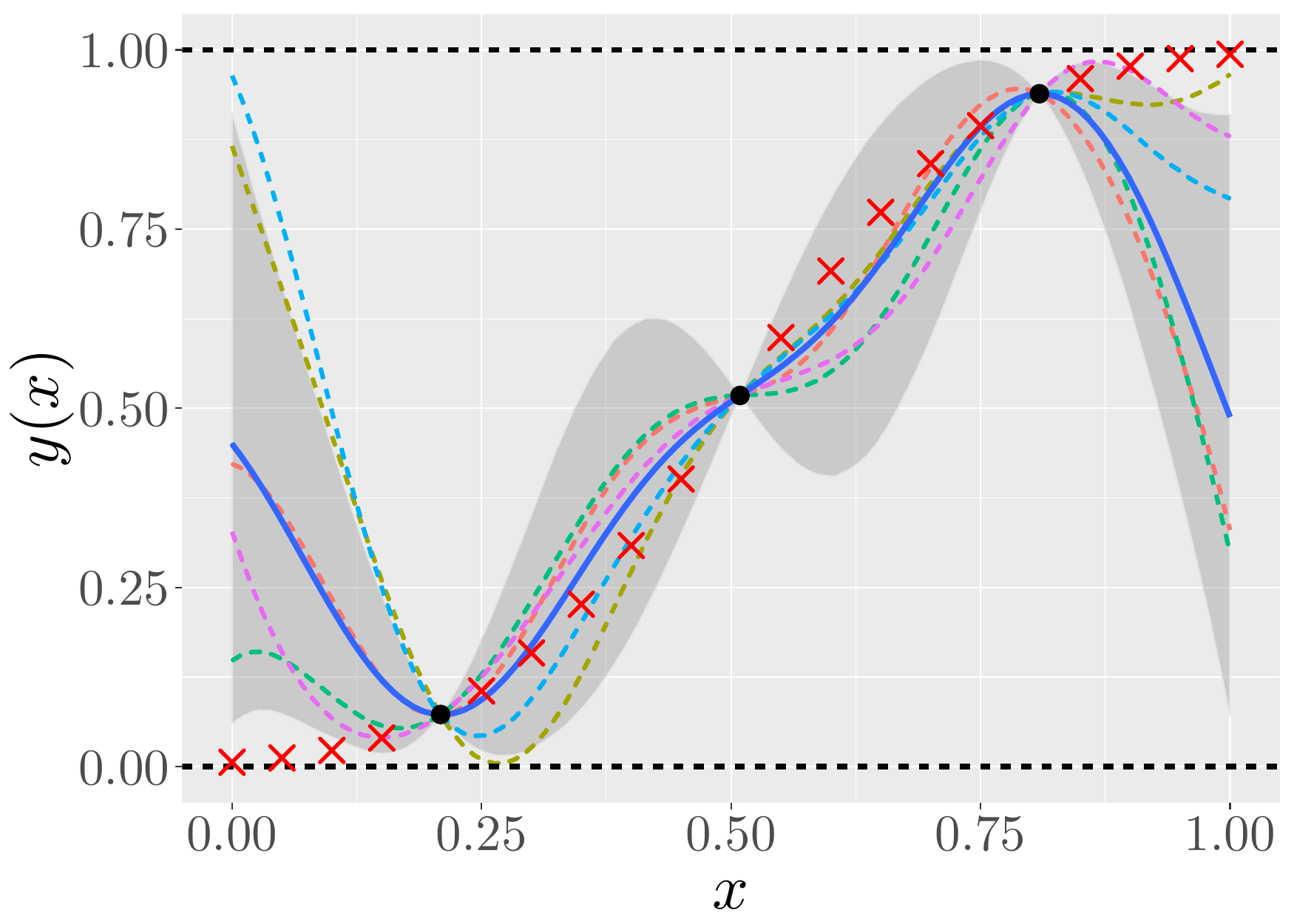}}
	
	\subfigure[\scriptsize \label{subfig:toyExample1b3} Monotonicity.]{\includegraphics[width=0.40\textwidth]{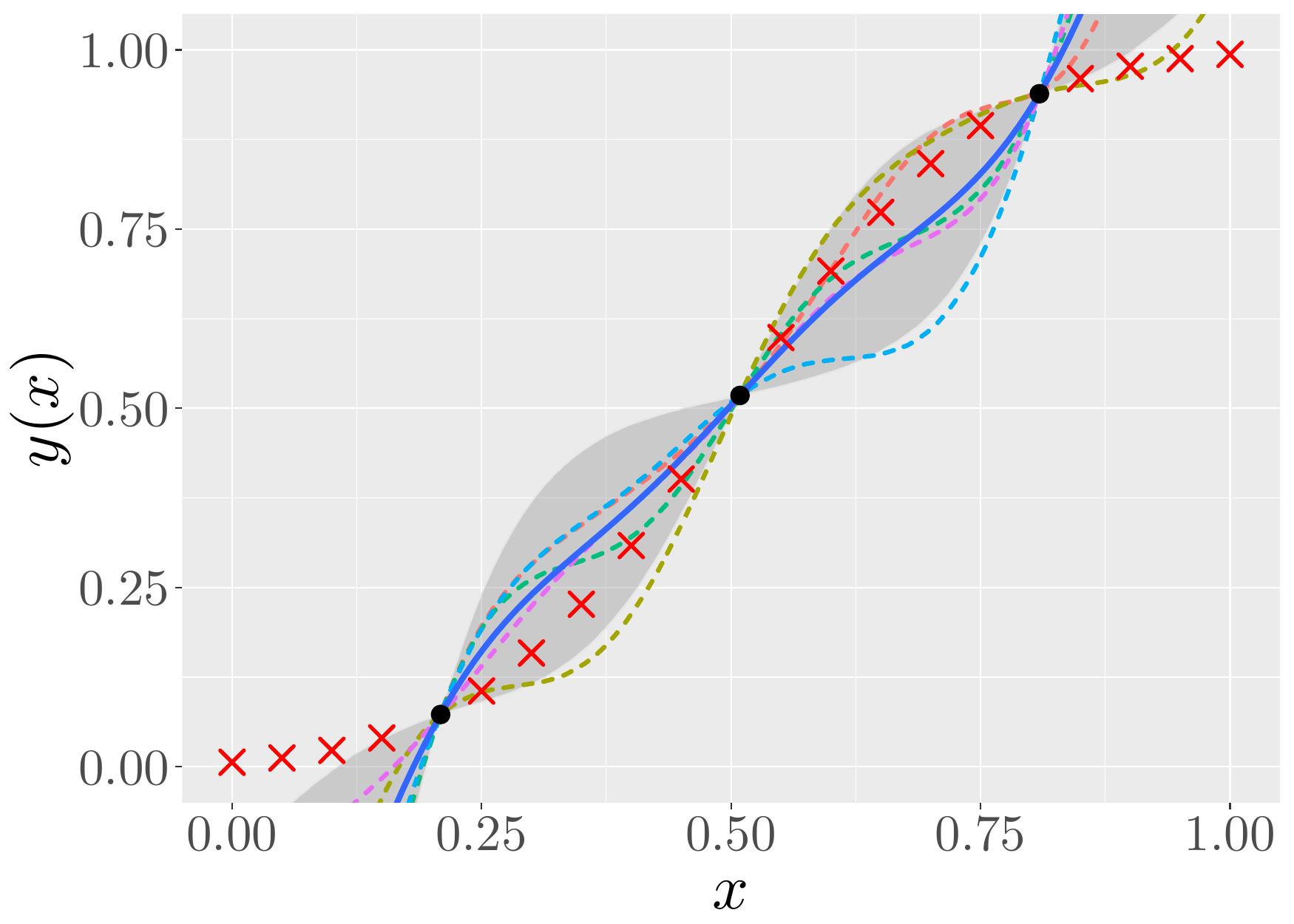}} \hspace{20pt}
	\subfigure[\scriptsize \label{subfig:toyExample1b4} Boundedness and monotonicity.]{\includegraphics[width=0.40\textwidth]{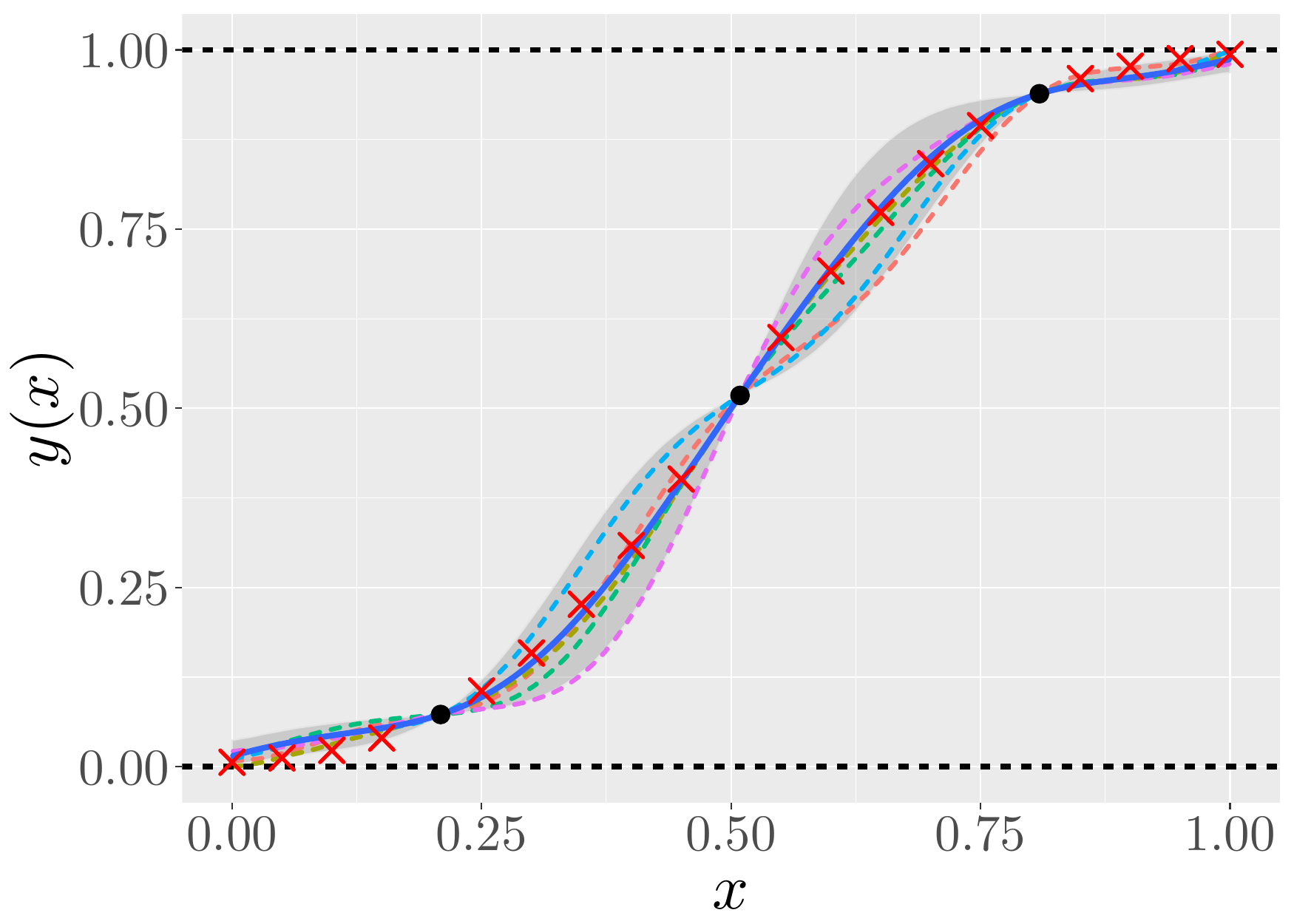}}
	\caption{Example of Gaussian models satisfying different types of inequality constraints for interpolating the function $x \mapsto \Phi(\frac{x-0.5}{0.2})$. Each panel shows: training and test points (black dots and red crosses, respectively), the conditional mean function (blue solid line), the $90\%$ confidence interval (grey region), and conditional realizations (dashed multi colour lines). For boundedness constraints, bounds are showed in black dashed lines.}
	\label{fig:toyExample1b}
\end{figure}

\section{Finite-dimensional Gaussian approximation with linear inequality constraints}
\label{sec:GPwithICLinear}
Now, we consider the case where $\mathcal{C}$ is composed by a set of $q$ linear inequalities of the form
\[
\mathcal{C} = \bigg\{\Bc \in \realset{m}; \ \forall \; j = 1, \dots, m , \forall \; k = 1, \dots, q, %q \geq m, 
\lambda_{k,j} \in \realset{} \; : \; \ell_k \leq \sum_{j = 1}^{m} \lambda_{k,j} c_j \leq u_k \bigg\},
\]
where the $\lambda_{k,j}$'s encode the linear operations, the $l_k$'s and $u_k$'s represent the lower and upper bounds. Notice that the convex sets $\mathcal{C}_\kappa$ of \Cref{eq:convexsetKnots} are particular cases of $\mathcal{C}$. Denote $\BLambda = (\lambda_{k,j})_{1 \leq k \leq q, 1 \leq j \leq m}$, $\Bl = (\ell_{k})_{1 \leq k \leq q}$, and $\Bu = (u_{k})_{1 \leq k \leq q}$. Hence, \Cref{eq:trGP1} is written
\begin{equation}
\Bxi \sim \normrnd{\Bzero}{\BGamma} \quad \mbox{s.t.} \quad  \begin{cases} \BPhi \Bxi = \By & \mbox{(interpolation conditions)}, \\ \Bl \leq \BLambda \Bxi \leq \Bu & \mbox{(inequality conditions)}. \end{cases}
\label{eq:trGP2}	
\end{equation}
We further assume that $q \geq m$ and that $\BLambda$ has rank $m$. By the rank-nullity theorem (see e.g. \citep{Meyer2000Algebra}), it implies that $\BLambda$ is injective. In particular a linear system of the form $\BLambda \Bxi = \Beta$ admits a unique solution $\Bxi$ when $\Beta$ is in the image space of $\BLambda$.  This assumption is verified in many practical situations, up to adding inactive constraints. For instance, the monotonicity condition $\Bxi_1 \leq \cdots \leq \Bxi_m$, which involves only $q=m-1$ (linearly independent) conditions, can be made compatible by adding the condition $-\infty \leq \Bxi_1$ (and/or $\Bxi_m \leq \infty$).

We now explain how to sample $\Bxi$ from \Cref{eq:trGP2}. First, we compute the conditional distribution given the interpolation constraints $\Bxi|\{\BPhi \Bxi = \By\}$. Since $\Bxi \sim \normrnd{\Bzero}{\BGamma}$, then $\BPhi \Bxi \sim \normrnd{\Bzero}{\BPhi \BGamma \BPhi^\top}$ and the conditional distribution $\Bxi |\{\BPhi \Bxi = \By\}$ is also Gaussian $\normrnd{\Bmu}{\BSigma}$ \citep{Rasmussen2005GP}, with
\begin{equation}
\Bmu = \BGamma \BPhi^\top [\BPhi \BGamma \BPhi^\top]^{-1} \By, \quad \mbox{and} \quad 
%	\BSigma = (\BGamma -  \BGamma \BPhi^\top [\BPhi \BGamma \BPhi^\top]^{-1} \BPhi \BGamma).
\BSigma = \BGamma -  \BGamma \BPhi^\top [\BPhi \BGamma \BPhi^\top]^{-1} \BPhi \BGamma.
\label{eq:posterior}
\end{equation}
Therefore, we have $\BLambda \Bxi |\{\BPhi \Bxi = \By\} \sim \normrnd{\BLambda \Bmu}{\BLambda \BSigma \BLambda^\top}$. 
Let $\tnormrnd{\textbf{m}}{\textbf{C}}{\textbf{a}}{\textbf{b}}$ be the truncated multinormal distribution with mean vector $\textbf{m}$, covariance matrix $\textbf{C}$, and bound vectors $(\textbf{a},\textbf{b})$ such that $\textbf{a} \leq \textbf{b}$. Thus, the posterior distribution of \Cref{eq:trGP2} is obtained from
\begin{equation}
\BLambda \Bxi|\{\BPhi \Bxi = \By, \Bl \leq \BLambda\Bxi \leq \Bu\} \sim \tnormrnd{\BLambda \Bmu}{\; \BLambda \BSigma \BLambda^\top}{\; \Bl}{\; \Bu}.
\label{eq:trposterior}
\end{equation}
Notice that the inequality conditions are encoded in the posterior mean $\BLambda \Bmu$, the posterior covariance $\BLambda \BSigma \BLambda^\top$, and bounds $(\Bl,\Bu)$. Finally, the posterior mode is given by $\Bnu_\Bxi^\ast = \BLambda \Bmu_\Bxi^\ast$ where $\Bmu_\Bxi^\ast$ is the solution provided in \Cref{subsec:GPwithICMaatouk}. The truncated multinormal of \Cref{eq:trposterior} can be approximated using Markov Chain Monte Carlo (MCMC) algorithms. Denoting $\Beta = \BLambda \Bxi$, notice that the samples for $\Bxi$ can be obtained by using the ones obtained for $\Beta$ if the linear system is solved. Indeed, as mentioned above, we assumed that $\BLambda$ has rank $m$, which implies that the solution of $\BLambda \Bxi= \Beta$ exists and is unique. 

The whole sampling scheme is summarized in \cref{alg:tmKriging}. Now, we illustrate some examples where the proposed framework satisfies different types of inequality conditions. The posterior is approximated via HMC (more details about HMC are given in \Cref{sec:simPosterior}).
\begin{algorithm}[t!]
	\small	
	\caption{Sampling from the finite-dimensional GP with linear inequality constraints.}\label{alg:tmKriging}
	\begin{algorithmic}[1]
		\Procedure{Sampling from $\Bxi |\{ \BPhi \Bxi = \By, \Bl \leq \BLambda \Bxi \leq \Bu \}$, where $\Bxi \sim \normrnd{\Bzero}{\BGamma}$}{}
		\BState \textbf{Input:} $\By$, $\BGamma \in \realset{m\times m}$, $\BPhi \in \realset{n\times m}$, $\mathcal{C}$.
		\BState Compute the conditional mean and covariance of $\Bxi |\{ \BPhi \Bxi = \By \}$
		\State $\Bmu = \BGamma \BPhi^\top (\BPhi \BGamma \BPhi^\top)^{-1} \By$, and
		\State $\BSigma = \BGamma - \BGamma \BPhi^\top (\BPhi \BGamma \BPhi^\top)^{-1} \BPhi \BGamma$.
		\BState Solve the quadratic problem in $\realset{m}$: $\Bmu_\Bxi^\ast = \min{_{\Bxi \in \realset{m}}} \{ \Bxi^\top \BGamma^{-1} \Bxi| \ \BPhi \Bxi = \By, \Bl \leq \BLambda \Bxi \leq \Bu \}$.
		\BState Sample from the truncated multinormal distribution
		\State $\BLambda \Bxi |\{ \BPhi \Bxi = \By, \textbf{\textit{l}} \leq \BLambda \Bxi \leq \textbf{\textit{u}} \} \sim \tnormrnd{\BLambda \Bmu}{\; \BLambda \BSigma\BLambda^\top}{\; \Bl}{\; \Bu}.$
		\BState Define $\Beta = \BLambda \Bxi$, and solve the linear system to obtain the sample $\Bxi$.
		\BState \textbf{Remark:} use the posterior mode $\Bnu_\Bxi^\ast = \BLambda \Bmu_\Bxi^\ast$ as a starting state for an MCMC sampler (see \Cref{sec:simPosterior}).
		\EndProcedure
	\end{algorithmic}
\end{algorithm}
%\medskip
\vspace{-10pt}
\paragraph{Example 1} We continue with the example of \cref{fig:toyExample1a}. As we can fix the structure of the linear inequalities $(\BLambda,\Bl,\Bu)$, we can impose both boundedness and monotonicity conditions in the constrained GP. One way to do this is to encode them individually. Let $\Bl_1 \leq \BLambda_1 \Bxi \leq \Bu_1$ and $\Bl_2 \leq \BLambda_2  \Bxi \leq \Bu_2$ be the sets of conditions to satisfy boundedness and monotonicity constraints, respectively. Then, we can build an extended set of inequalities $\Bl \leq \BLambda  \Bxi \leq \Bu$ by stacking the constraints (i.e. $\BLambda = [\begin{smallmatrix} \BLambda_1, & \BLambda_2 \end{smallmatrix}]^\top, \Bl = [\begin{smallmatrix} \Bl_1, & \Bl_2 \end{smallmatrix}]^\top, \Bu = [\begin{smallmatrix}\Bu_1, & \Bu_2 \end{smallmatrix}]^\top$), so that \cref{alg:tmKriging} can be used. Notice that one can encode the same information in a reduced set of linear inequalities. Instead of encoding independently the boundedness and monotonicity constraints, which requires $q = 2m - 1$ inequalities, one can impose boundedness conditions only for the first and last knot, and monotonicity conditions for all the knots except the first one. Due to monotonicity, the intermediate knots will also satisfy the boundaries. In this way, we only need $q = m + 1$ conditions. In many other cases, the size of specific sets of linear constraints can be reduced. However for general discussions, we will use the full extended set and we will apply efficient samplers to approximate the posterior.
%\medskip
\vspace{-10pt}
\paragraph{Example 2} Notice from the previous example that the extension for more than two sets of inequalities is straightforward. Consider for instance $Q$ different sets of conditions. We can build the posterior from \Cref{eq:trposterior} with $\BLambda = [\begin{smallmatrix} \BLambda_1, & \cdots, & \BLambda_Q \end{smallmatrix}]^\top$, $\Bl = [\begin{smallmatrix} \Bl_1, & \cdots, & \Bl_Q \end{smallmatrix}]^\top$, and $\Bu = [\begin{smallmatrix}\Bu_1, & \cdots, & \Bu_Q \end{smallmatrix}]^\top$, and apply \cref{alg:tmKriging}. \Cref{fig:toyExample2} shows an example with the target function $y(x) = x^2$, satisfying three types of inequality constraints: boundedness, monotonicity and convexity. We proposed different models satisfying one or more inequality constraints. We used a SE covariance function with parameters $(\sigma^2 = 1.0,\; \theta =0.2)$. By imposing the three conditions, we obtain samples that also satisfy the three types of constraints.
%\medskip
\begin{figure}[t!]
	\centering
	\subfigure[\scriptsize \label{subfig:toyExample2a1} Adding boundedness constraint.]{\includegraphics[width=0.325\textwidth]{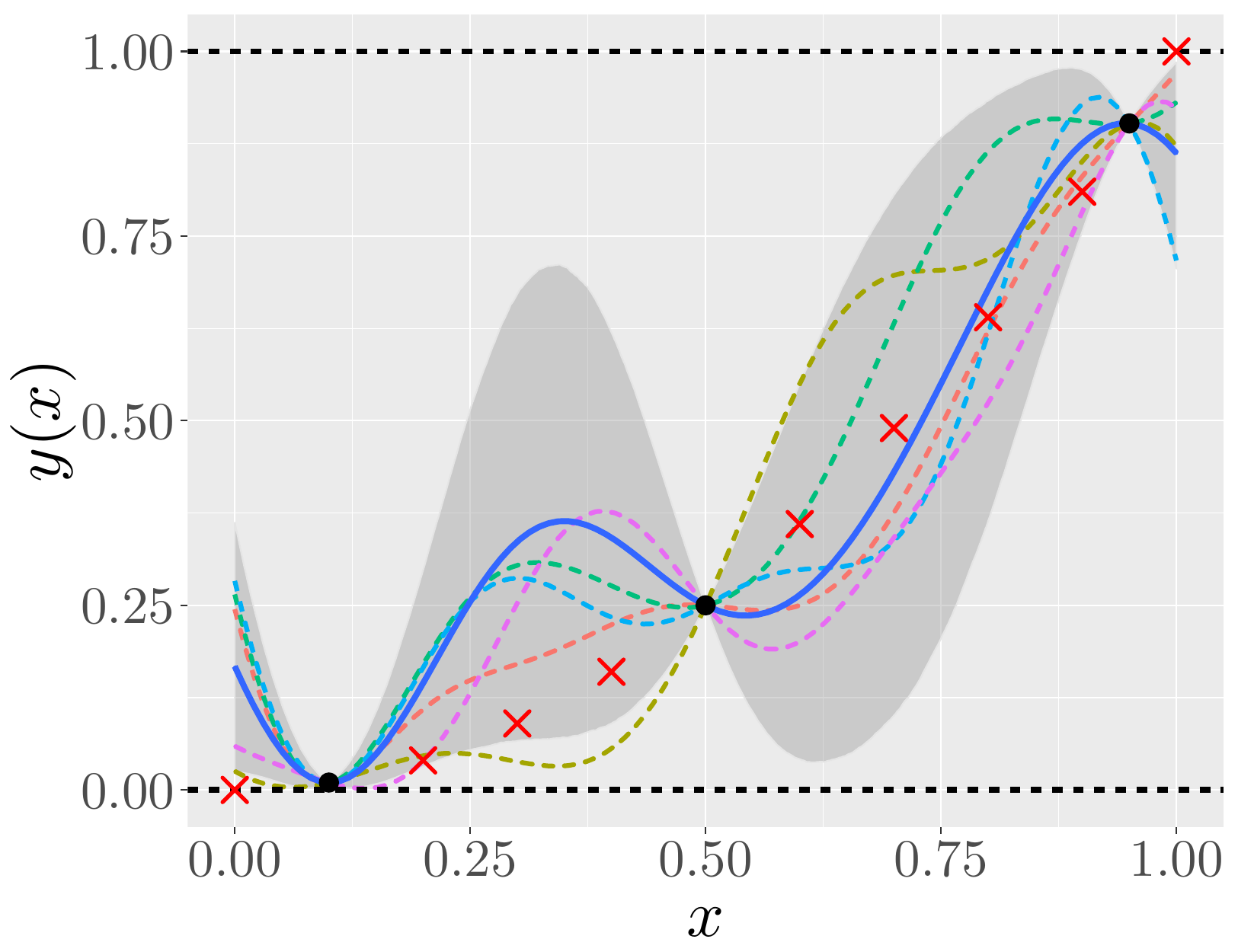}}	
	\subfigure[\scriptsize \label{subfig:toyExample2a4} Adding monotonicity constraint.]{\includegraphics[width=0.325\textwidth]{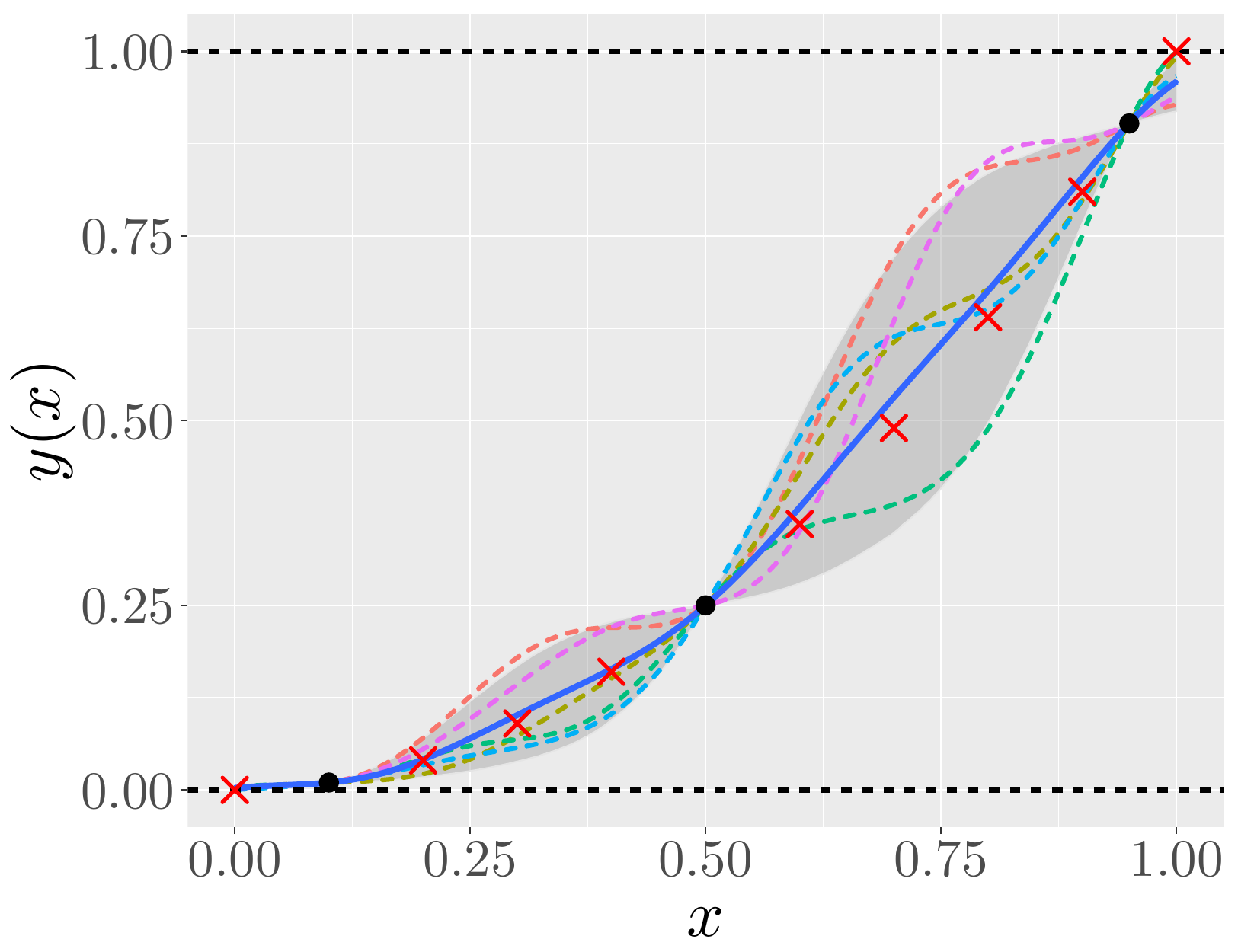}}	
	\subfigure[\scriptsize \label{subfig:toyExample2a7} Adding  convexity constraint.]{\includegraphics[width=0.325\textwidth]{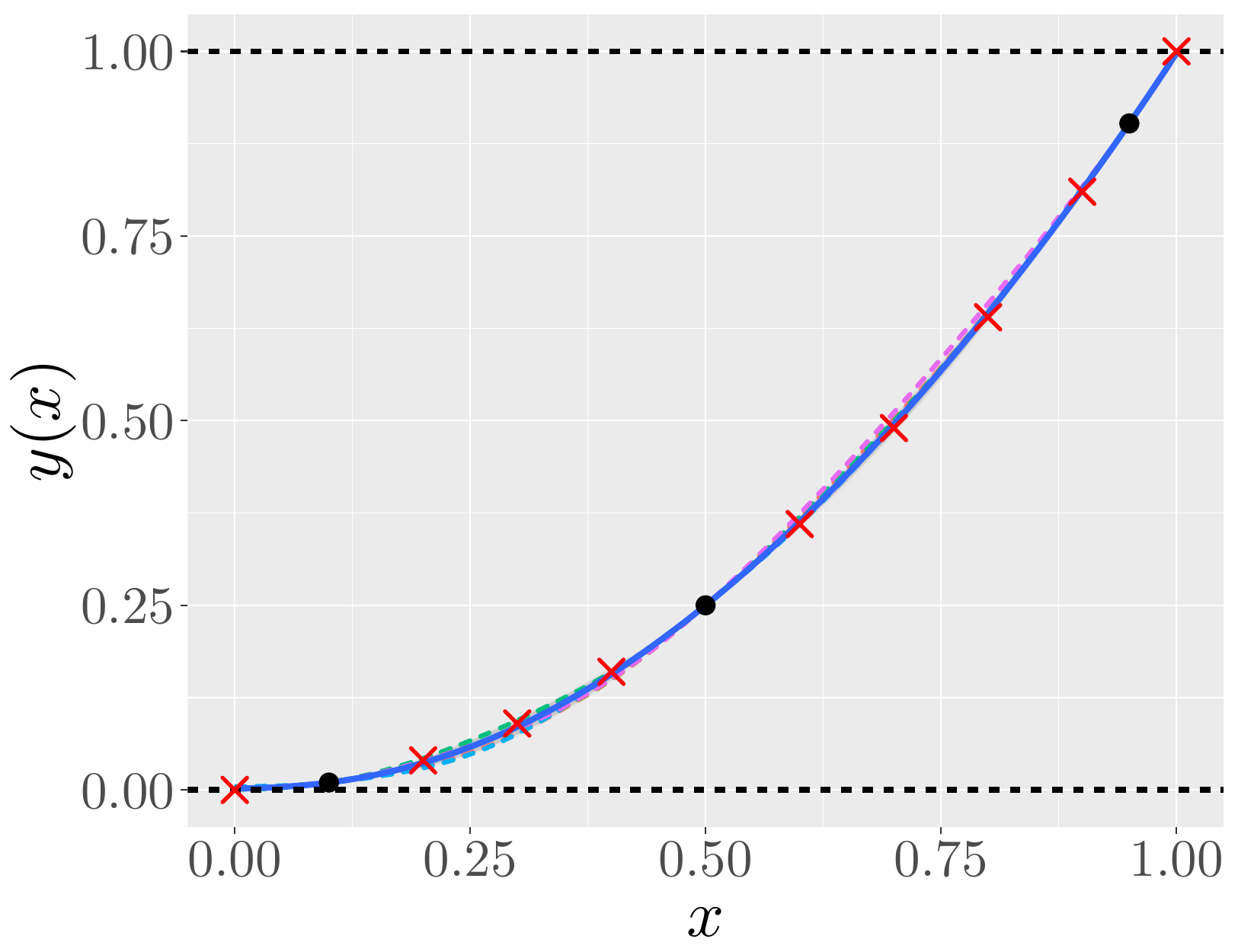}}
	\caption{Examples of Gaussian models satisfying one or several types of inequality constraints for interpolating the square function $x \mapsto x^2$. Panel description is the same as in \Cref{fig:toyExample1b}.}
	\label{fig:toyExample2}
\end{figure}
\begin{figure}[t!]
	\subfigure[\scriptsize \label{subfig:toyExample3a1} Unconstrained.]{\includegraphics[width=0.325\textwidth]{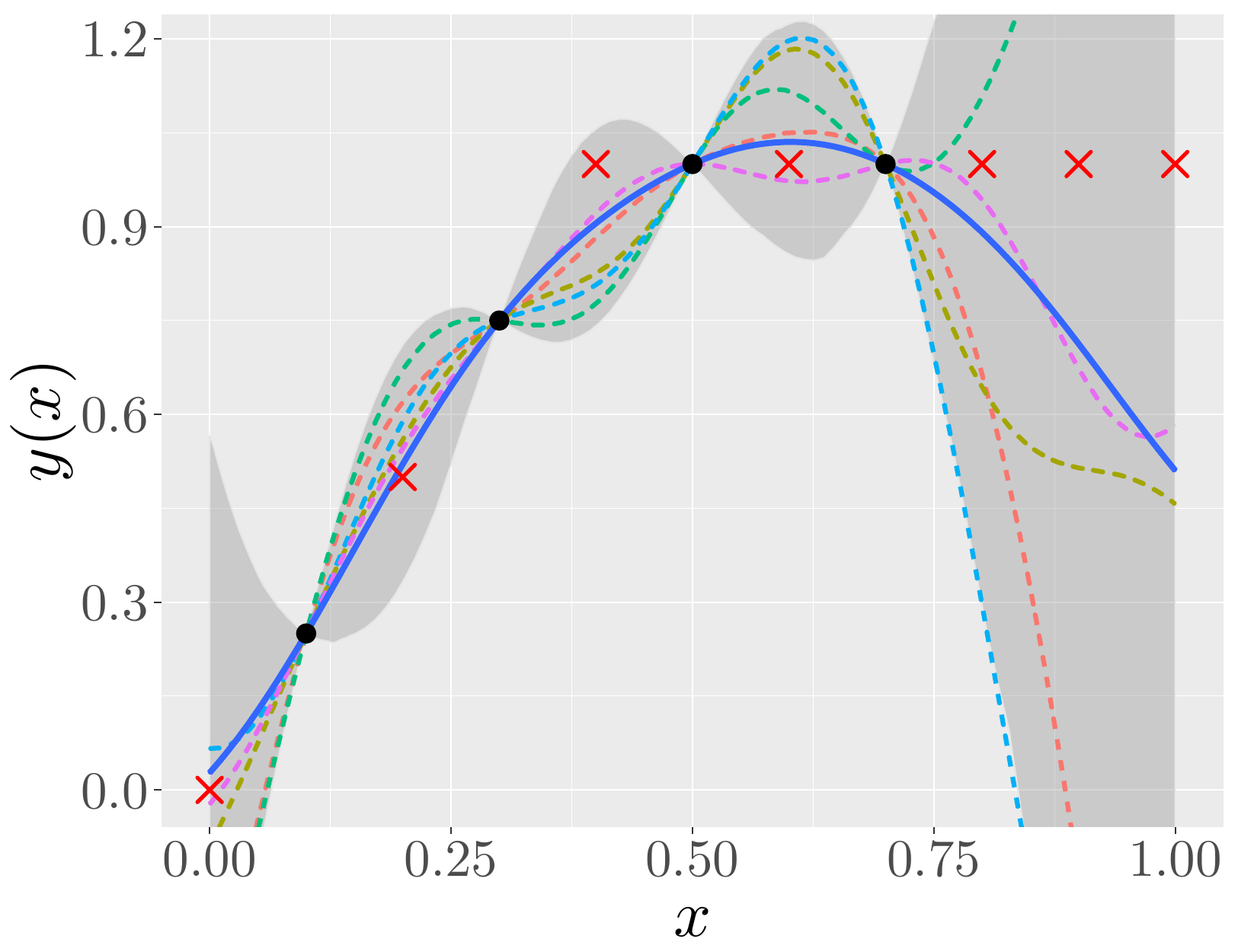}}	
	\subfigure[\scriptsize \label{subfig:toyExample3a4} Multiple contraints.]{\includegraphics[width=0.325\textwidth]{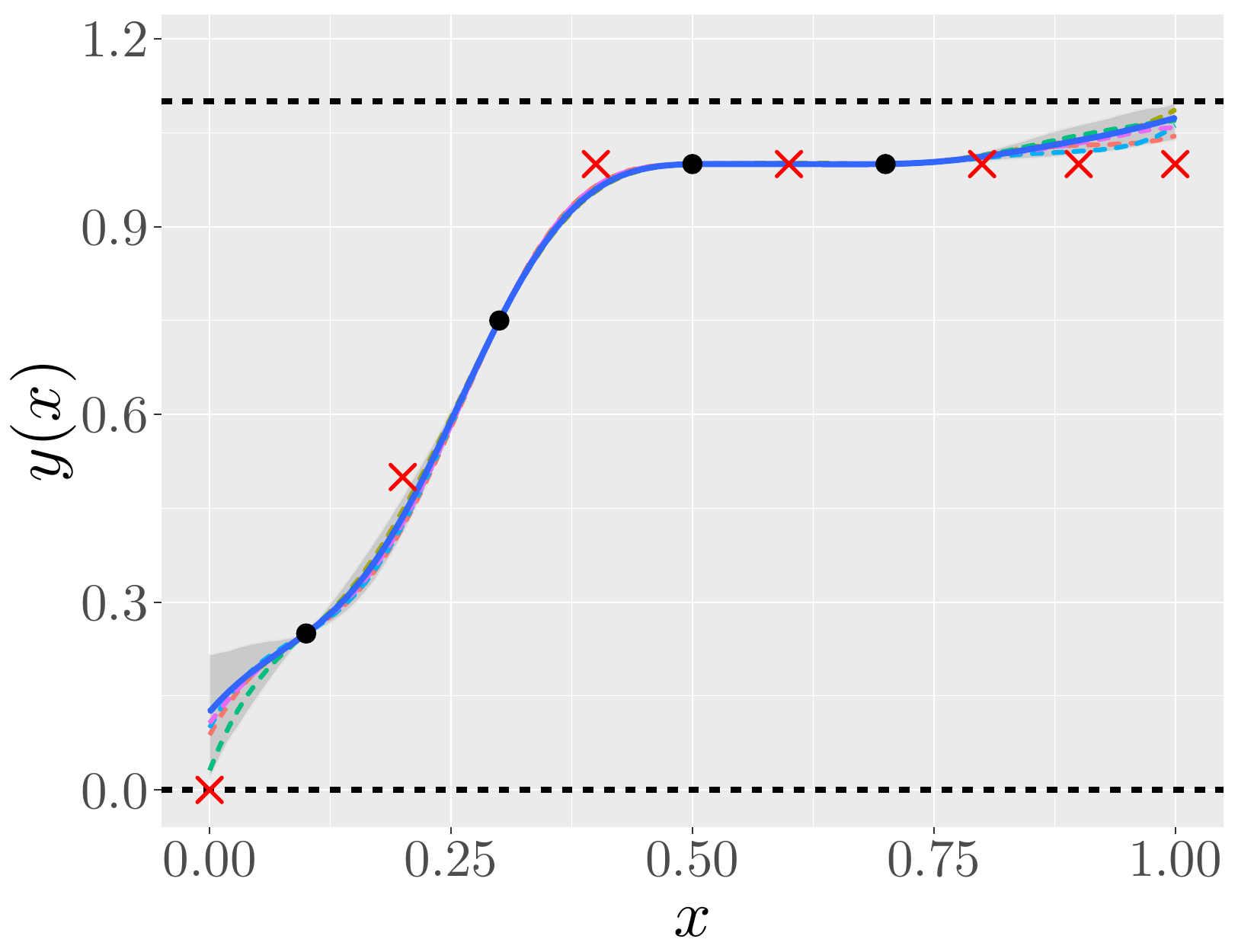}}	
	\subfigure[\scriptsize \label{subfig:toyExample3a5} Sequential constraints.]{\includegraphics[width=0.325\textwidth]{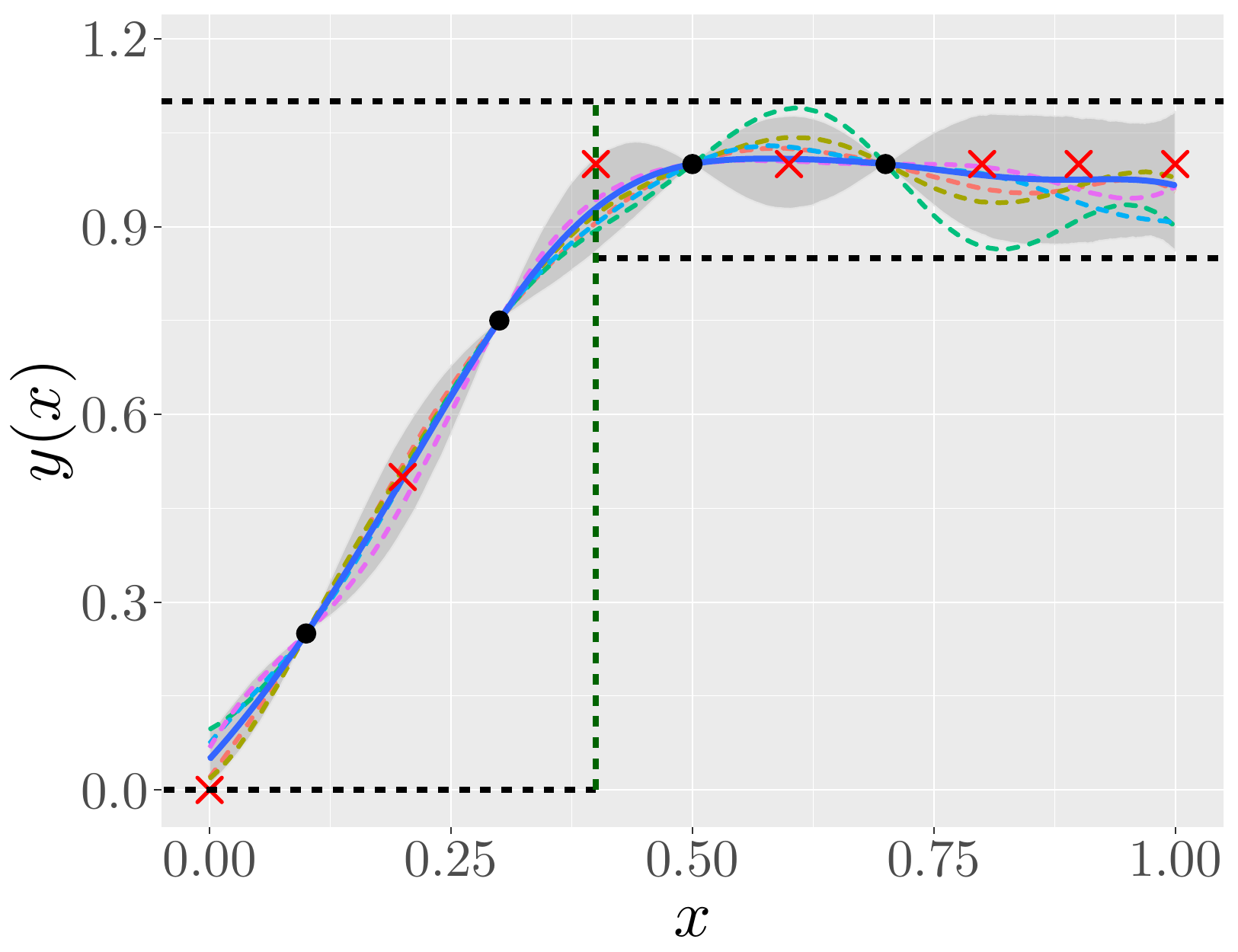}}
	\caption{Examples of Gaussian models with different types of constraints for the example 3 from \Cref{sec:GPwithICLinear}. \subref{subfig:toyExample3a4} Boundedness and monotonicity constraints are imposed. \subref{subfig:toyExample3a5} The two non-overlapping intervals are divided by a vertical dashed line at 0.4. In the first interval, boundedness and monotonicity constraints are taken into account. In the second interval, only boundedness is imposed. Panel description is the same as in \Cref{fig:toyExample1b}.}
	\label{fig:toyExample3}
\end{figure}
\vspace{-10pt}
\paragraph{Example 3} 
Since the bounds $(\Bl,\Bu)$ are not forced to be the same everywhere, it is possible to fix specific constraints over non-overlapping intervals. For instance if the interval is partitioned in $G$ subintervals, we consider the corresponding partition $\Bxi = [\begin{smallmatrix} \Bxi_1, & \cdots, & \Bxi_G \end{smallmatrix}]^\top$. Then, we can impose different types of inequality conditions in each group by considering the same structure used in Example 2. \cref{fig:toyExample3} shows an example where the function $y$ satisfies different behaviours in two non-overlapping intervals. The output increases monotonically and peaks at $y(0.4) = 1.0$. This kind of profile is met in different applications (e.g. step responses in control theory, protein profiles in molecular biology) \citep{Murphy2012ML,Kocijan2016GPControl}. We trained three models satisfying different conditions. For the case of multiple constraints, we imposed boundedness and monotonicity. For the case of sequential conditions, we divided the profile in two non-overlapping intervals satisfying different types of constraints. We used a SE covariance function with parameters $(\sigma^2 = 1.0, \; \theta =0.2)$. By imposing sequentially the constraints, we obtain less restricted uncertainties and more accurate models for data fitting.

\section{Simulating from the posterior distribution}
\label{sec:simPosterior}

As shown in \Cref{eq:trposterior}, the posterior distribution $\BLambda \Bxi|\{\BPhi \Bxi = \By, \Bl \leq \BLambda \Bxi \leq \Bu\}$ is truncated multinormal. It is supported on $\mathbb{R}^m$, where $m$ is the number of basis functions. Notice that $m$ should be chosen large enough for better approximations. A Monte Carlo (MC) algorithm based on rejection sampling was proposed in \citep{Maatouk2016RSM} using the posterior mode. This method, called rejection sampling from the mode (RSM), is an exact sampler that provides independent and identically distributed (iid) sample paths. However, the acceptance rate from RSM decreases when $m$ gets larger, providing a poor performance for high dimensional spaces. Another MC-based exact sampler was introduced by \citep{Botev2017MinimaxTilting} to deal with truncated multinormals in higher dimensions. It can both simulate multinormals under linear constraints, and estimate the probabilities that these constraints are satisfied, via minimax exponential tilting (ET). As RSM and ET are exact methods, we will use them as gold standards to evaluate the performance of the MCMC techniques that we describe now.

\subsection{MCMC for truncated multinormal distributions}
\label{subsec:MCMC}
%Several MCMC algorithms have been proposed for sampling from complex distributions.
MCMC approaches assume that samples follow a Markov chain, providing correlated samples but with a higher acceptance rate. Recently, efficient algorithms have been proposed for truncated multinormal distributions such as Gibbs sampling \citep{Benjamini2017fastGibbs}, Metropolis-Hastings \citep{Murphy2012ML}, and Hamiltonian Monte Carlo \citep{Pakman2014Hamiltonian}. In this section, we apply them to simulate from the posterior distribution of \Cref{eq:trposterior}.
%\medskip
\vspace{-10pt}
\paragraph{Gibbs sampling}
Algorithms based on Gibbs sampling are widely used to sample from truncated multinormals due to their easy implementation, and their reliable performances \citep{Murphy2012ML,Brooks2011handbook}. They sample each variable in turn conditionally on the values of the other ones \citep{Murphy2012ML}. Therefore, sampling from a truncated multinormal is reduced to sampling sequentially from conditional truncated (univariate) normals. Unlike RSM, there is no rejection step. However, the ``single site updating'' property may produce strong correlations, requiring to discard intermediate samples (thinning effect). Several studies have proposed efficient algorithms to obtain less correlated sample paths (e.g. collapsed Gibbs sampling, blocked Gibbs sampling) \citep{Murphy2012ML}. In this paper, we will use the fast Gibbs sampler proposed in \citep{Benjamini2017fastGibbs}.
%\medskip
\vspace{-10pt}
\paragraph{Metropolis Hastings (MH)}
MH-based algorithms propose to move all the coordinates at a time in each step to obtain less correlated simulations. Given a proposed state $\Bx'$, we either accept or reject the new state according to a given acceptance rule \citep{Murphy2012ML}. If the proposal is accepted, the new state is $\Bx'$, otherwise the new state remains at the previous state $\Bx$. For multinormal distributions, a symmetric Gaussian proposal is commonly used, i.e. $q(\Bx'|\Bx) = \normrnd{\Bx}{\eta \BSigma_{\Bx'|\Bx}}$ where $\eta$ is a scale factor. This approach is known as random walk Metropolis algorithm \citep{Murphy2012ML}. One can increase the acceptance rate by tuning properly the value of $\eta$.
%\medskip
\vspace{-10pt}
\paragraph{Hamiltonian Monte Carlo (HMC)}
Nowadays, hybrid methods have been subject to great attention from the statistical community due to the inclusion of physical interpretation that may provide useful intuition \citep{Brooks2011handbook,Neal1996BayesianLearning}. In \citep{Duane1987hybridHM}, an efficient hybrid approach was introduced using the properties of Hamiltonian dynamics. Later in \citep{Neal1996BayesianLearning}, the hybrid approach from \citep{Duane1987hybridHM} was extended to statistical applications, and was introduced formally as Hamiltonian Monte Carlo (HMC). The Hamiltonian dynamics provides distant proposal distributions producing less correlated sample paths without diminishing the acceptance rate. In this paper, we use the HMC-based approach for truncated multinormals introduced in \citep{Pakman2014Hamiltonian}. 

\subsection{Results}
In \Cref{tab:ess}, we evaluate the efficiency of the MC and MCMC approaches described in \Cref{subsec:MCMC} on the examples from \cref{fig:toyExample1b}. In order to reduce the simulation cost, we used $m = 30$ hat basis functions. Hence, the problem is to sample a vector of length 30 from a truncated multinormal distribution. We set the tuning hyperparameters such that the effective sample size (ESS) is within the ranges produced by both RSM and ET (grey columns). The ESS is a heuristic used commonly to evaluate the quality of correlated sample paths, and it gives an intuition on how many samples from the path can be considered independent \citep{Gong2016ESS}. A standard ESS is given by $\operatorname{ESS} = n_s/(1+2\sum_{k=1}^{n_s} \rho_k)$ where $n_s$ is the size of the sample path and $\rho_k$ is the sample autocorrelation with lag $k$. However, the drawback of this indicator is that it accepts negative correlations to evaluate the quality of mean estimators (e.g. for variance reduction). Thus, we suggest an alternative ESS which penalizes both positive and negative correlations: $\operatorname{ESS} = n_s/(1+2|\sum_{k=1}^{n_s} \rho_k|)$. Notice that it is equal to $n_s$ for iid sample paths, or smaller otherwise. We compute the ESS indicator for each coordinate of $\Bxi$, i.e. $\operatorname{ESS}_j = \operatorname{ESS}(\xi_j^1, \cdots, \xi_j^{n_s})$ for the $j$-th component of $\Bxi$ with $j = 1, \cdots, m$. We then compute quantiles $(q_{10\%},q_{50\%},q_{90\%})$ over the 30 resulting $\operatorname{ESS}$ values. To take into account cross-correlations from multivariate MCMC, we also compute the multivariate ESS (mvESS) proposed in \citep{Vats2015multiMCMC}. For mvESS, values higher than $n_s$ indicate the presence of negative correlations. In our case, we are interested in being around $n_s$. The size $n_s = 10^{4}$ is chosen to be larger than the minimum ESS required to obtain a proper estimation of the vector $\Bxi \in \realset{30}$: $\operatorname{minESS}(30) = 8563$ \citep{Gong2016ESS,Vats2015multiMCMC}. Finally, using the procedure proposed in \citep{Lan2016Sampling}, we test the efficiency of each method by computing the time normalised ESS (TN-ESS) at $q_{10\%}$ (worst case) using the CPU time in seconds, i.e. $\operatorname{TN-ESS} = q_{10\%}(\operatorname{ESS})/\operatorname{(CPU \ Time)}$. %The $\operatorname{TN-ESS}$ indicates how many effective iid simulations can produced per second. 

\Cref{tab:ess} shows the efficiency of MC/MCMC algorithms in terms of ESS indicators. Notice that for the two examples of \cref{subfig:toyExample1b2,subfig:toyExample1b3}, the MC/MCMC techniques tend to produce similar ESS intervals, but RSM and MH are the most expensive procedures due to their high rejection rates. Although the Gibbs sampler requires to discard a large amount of simulations in order to be within reasonable ESS ranges, it also presents accurate results in both efficiency and CPU time. In general, both ET and HMC methods present more efficient results than the other samplers in the first two examples. For more complex constraints as in the example of \cref{subfig:toyExample1b4}, the efficiency is reduced dramatically for all the methods. For example, the acceptance rates of both RSM and MH are so small that sampling was not feasible in a reasonable time. For the other methods, the TN-ESS rates are smaller but HMC still gives a reasonable value (almost three times larger than for ET), concluding that HMC is an efficient sampler for the proposed framework.
\begin{table}
	\centering
	\scriptsize
	\caption{Efficiency of MC/MCMC samplers (by rows) in term of ESS-based indicators (by columns). \textbf{Samplers:} Rejection Sampling from the Mode (RSM) \citep{Maatouk2016RSM}, Exponential Tilting (ET) \citep{Botev2017MinimaxTilting}, Gibbs Sampling (Gibbs) \citep{Benjamini2017fastGibbs}, Metropolis-Hasting (MH) \citep{Murphy2012ML}, Hamiltonian Monte Carlo (HMC) \citep{Pakman2014Hamiltonian}. \textbf{Indicators:} effective sample size (ESS): $\operatorname{ESS} = n/(1+2|\sum_{\forall k} \rho_k|)$, multivariate ESS (mvESS) \citep{Vats2015multiMCMC}, time normalised ESS (TN-ESS) \citep{Lan2016Sampling}.}
	\label{tab:ess}	
	\begin{tabular}{c|c|cggc|c}
		\toprule
		\multirow{2}{*}{Toy Example} & \multirow{2}{*}{Method} & CPU Time & ESS $[\times10^4]$& mvESS & TN-ESS & \multirow{2}{*}{Hyperparameter} \\
		& & $[s]$ & $(q_{10\%},q_{50\%},q_{90\%})$ & $[\times10^4]$ &  $[\times10^4 s^{-1}]$ & \\
		\midrule
		& RSM & $99.06$ & $(0.81,0.89,0.96)$ & $1.22$ & $0.01$ & - \\
		Toy & ET & $\textbf{0.44}$ & $(0.83,0.90,0.99)$ & $1.17$ & $\textbf{1.88}$ & - \\		
		\Cref{subfig:toyExample1b2} & Gibbs & $3.54$ & $(0.81,0.91,0.93)$ & $1.16$ & $0.23$ & $\operatorname{thinning} = 200$ \\
		(Boundedness) & MH & $52.21$ & $(0.77,0.87,0.95)$ & $1.21$ & $0.01$ & $\eta = 1$ \\
		& HMC & $\textbf{0.44}$ & $(0.72,0.76,0.91)$ & $1.26$ & $1.64$ & - \\
		\midrule
		& RSM & $190.62$ & $(0.86,0.90,0.93)$ & $1.21$ & $0.00$ &- \\
		Toy & ET & $0.77$ & $(0.84,0.91,0.97)$ & $1.18$ & $1.09$ & - \\	
		\Cref{subfig:toyExample1b3} & Gibbs & $3.04$ & $(0.80,0.93,0.99)$ & $1.15$ & $0.26$ & $\operatorname{thinning} = 200$ \\
		(Monotonicity) & MH & $96.64$ & $(0.81,0.91,0.99)$ & $1.23$ & $0.01$ & $\eta  = 1$ \\
		& HMC & $\textbf{0.33}$ & $(0.73,0.79,0.88)$ & $1.28$ & $\textbf{2.22}$ & - \\
		\midrule
		& RSM & - & - & - & - & - \\
		Toy & ET & $41.16$ & $(0.80,0.88,0.95)$ & $1.23$ & $0.02$ & - \\
		\Cref{subfig:toyExample1b4} & Gibbs & $40.28$ & $(0.28,0.44,0.77)$ & $1.09$ & $0.01$ & $\operatorname{thinning} = 1000$ \\
		(Bounded Monotonicity) & MH & - & - & - & - & - \\		
		& HMC & $\textbf{12.92}$ & $(0.72,0.85,0.99)$ & $1.26$ & $\textbf{0.06}$ & - \\
		\bottomrule
	\end{tabular}
\end{table}

\section{Covariance parameter estimation with inequality constraints}
\label{sec:ml}

\subsection{Conditional maximum likelihood}
\label{subsec:cml}
Let $\{k_\Btheta; \Btheta \in \BTheta\}$, with $\BTheta \subset \realset{p}$, be a parametric family of covariance functions. We assume in this section that the zero-mean GP $Y$ has covariance function $k_{\Btheta^\ast}$ for an unknown $\Btheta^\ast \in \BTheta$. We consider the problem of estimating $\Btheta^\ast$. Commonly, $\Btheta^\ast$ is estimated by maximising the unconstrained Gaussian likelihood $p_\Btheta(\BY_m)$ with respect to $\Btheta \in \BTheta$ (maximum likelihood, ML), with $\BY_m = [\begin{smallmatrix} Y_m(x_1), & \cdots, & Y_m(x_n) \end{smallmatrix}]^\top$. Let $\mathcal{L}_{m}(\Btheta)$ be the log likelihood of $\Btheta$
\begin{equation}
\mathcal{L}_{m}(\Btheta) =  \log p_\Btheta(\BY_m) = - \frac{1}{2} \log (\det (\BK_\Btheta)) - \frac{1}{2} \BY_m^\top \BK_\Btheta^{-1} \BY_m - \frac{n}{2} \log 2 \pi,
\label{eq:ML} 
\end{equation}	
with $\BK_\Btheta = \BPhi \BGamma_\Btheta \BPhi^\top$ and $\BGamma_\Btheta = (k_\Btheta(t_i,t_j))_{1 \leq i,j \leq m}$. Then, the ML estimation (MLE) is
\begin{equation}
\widehat{\Btheta}_{\text{MLE}} = \operatorname{\arg \max}\limits_{\Btheta \in \BTheta} \ \mathcal{L}_{m}(\Btheta).
\label{eq:MLE}
\end{equation}	

When we maximise the likelihood of \Cref{eq:MLE}, we are looking for a parameter $\Btheta$ that improves the ability of our model to explain the data \citep{Rasmussen2005GP}. However, because the unconstrained ML itself does not take into account the constraints $\Bxi \in \mathcal{C}$, the estimated $\widehat{\Btheta}_{\text{MLE}}$ may produce less realistic models. Here, we suggest to use the constrained likelihood. Let $p_\Btheta(\BY_m| \Bxi \in \mathcal{C})$ be the conditional probability density function of $\BY_m$ given $\Bxi \in \mathcal{C}$, when $Y$ has covariance function $k_\Btheta$. By using Bayes' theorem, the constrained log likelihood $\mathcal{L}_{\mathcal{C},m}(\Btheta) = \log p_\Btheta(\BY_m|\Bxi \in \mathcal{C})$ is
\begin{equation}
\label{eq:CML} \mathcal{L}_{\mathcal{C},m}(\Btheta)
%	&= \log p_\Btheta(\BY_m|\Bxi \in \mathcal{C}) \\
%	= \log \frac{p_\Btheta(\BY_m) P_\Btheta(\Bxi \in \mathcal{C}|\BPhi \Bxi = \BY_m)}{P_\Btheta(\Bxi \in \mathcal{C})}
= \log p_\Btheta(\BY_m) + \log P_\Btheta(\Bxi \in \mathcal{C}|\BPhi \Bxi = \BY_m) - \log P_\Btheta(\Bxi \in \mathcal{C}),
\end{equation}
where the first term is the unconstrained log-likelihood, and the last two terms depend on the inequality constraints. Then, the constrained ML (CML) estimator is given by
\begin{equation}
\widehat{\Btheta}_{\text{CMLE}} = \operatorname{\arg \max}\limits_{\Btheta \in \BTheta}  \mathcal{L}_{\mathcal{C},m}(\Btheta).
\label{eq:CMLE}
\end{equation}
Notice that $P_\Btheta(\Bxi \in \mathcal{C} |\BPhi \Bxi = \BY_m)$ and $P_\Btheta(\Bxi \in \mathcal{C})$ are Gaussian orthant probabilities. As they have no explicit expressions, numerical procedures have been investigated \citep{Botev2017MinimaxTilting,Genz1992numericalcomputation}. Hence, the likelihood evaluation and optimisation of \Cref{eq:CML,eq:CMLE} have to be done numerically.

\subsection{Simulation study}
\label{subsec:cmlresults}
To assess the performance of the estimator of \Cref{eq:CMLE}, we simulated sample paths from a zero-mean constrained GP $Y$ using a Mat\'ern 5/2 covariance function with $\Btheta^\ast = (1, 0.2)$.\footnote{Mat\'ern 5/2 kernel function: $k_\Btheta(x-x') = \sigma^2 \left(1 + \frac{\sqrt{5} |x-x'|}{\theta} + \frac{5}{3} \frac{(x-x')^2}{\theta^2}\right) \exp\left\{-\frac{\sqrt{5}|x-x'|}{\theta} \right\}$ with $\Btheta = (\sigma^2,\theta)$.} We sampled 100 realizations of $Y$ on $\mathcal{D} = [0, 1]$ such that $Y \in [-1, 1]$. Then, for each realization, we trained a constrained model assuming boundedness conditions with bounds $[-1, 1]$. We used 10 training points regularly spaced in $\mathcal{D}$ and $m = 50$ hat basis functions. For ML and CML optimisations, we used multistart with ten initial vectors of covariance parameters located on a maximin Latin hypercube DoE with $\sigma^2 \in [0, 2]$ and $\theta \in [0.04, 0.40]$. As the parameters of the Mat\'ern 5/2 covariance function are non-microergodic for one-dimensional input spaces, they cannot be estimated consistently \citep{Zhang2004InconsistentEst}. Therefore, we evaluated the quality of the likelihood estimators using the consistently estimable ratio $\rho = \sigma^2/\theta^5$. In \cref{subfig:cmleToyExample1fig2}, we show the boxplots of the estimated ratios obtained with the 100 simulations drawn from the GP. Notice that the estimated logged ratios $\log \widehat{\rho}_{\text{MLE}}$ and $\log \widehat{\rho}_{\text{CMLE}}$ are reasonably close to the true value $\log \rho^\ast = \log  (1^2/0.2^5)$, but the one using CMLE is slightly better in terms of variance and bias. 

\begin{figure}[t!]
	\centering
	\subfigure[\scriptsize \label{subfig:cmleToyExample1fig2}]{\includegraphics[width=0.32\textwidth]{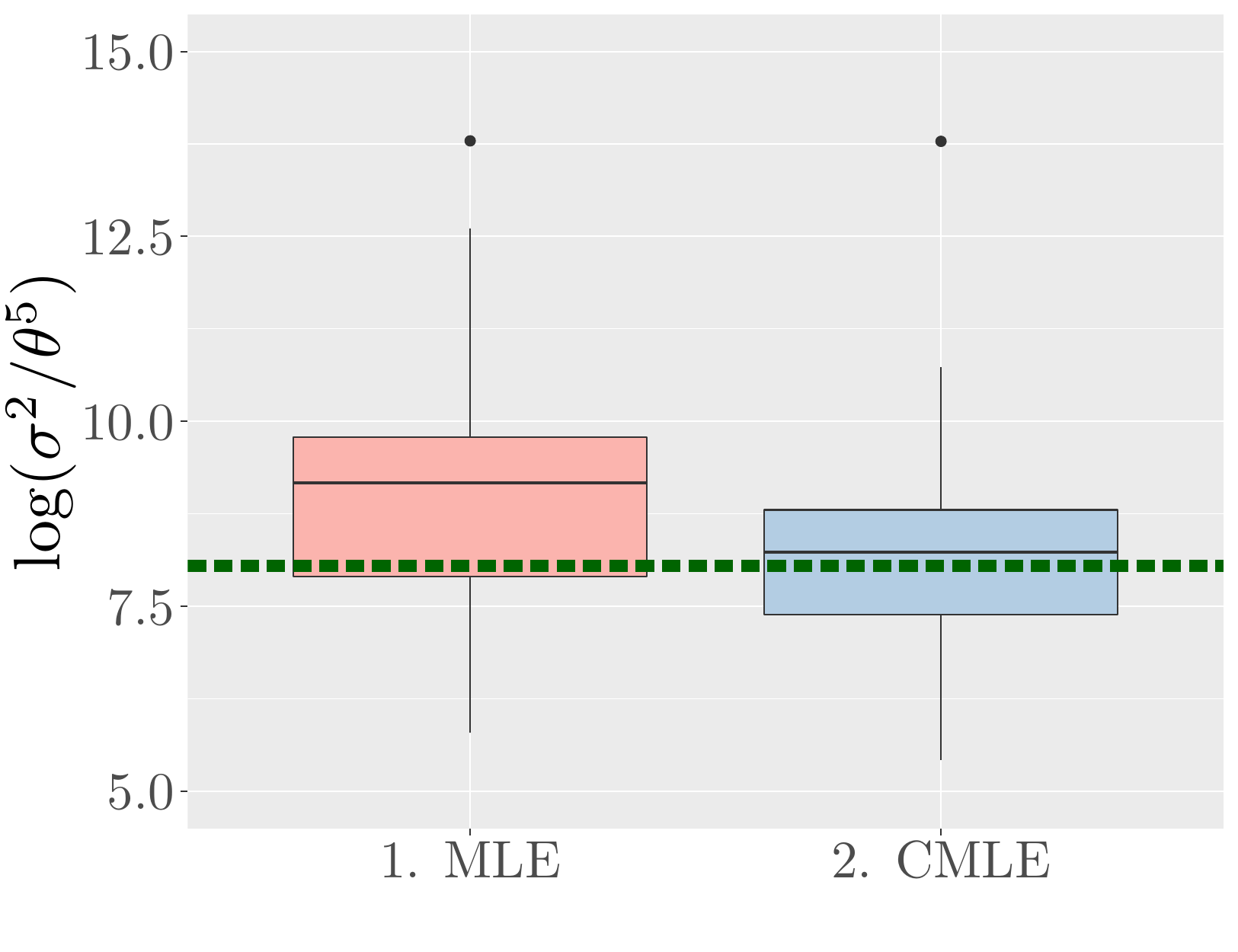}}
	\subfigure[\scriptsize \label{subfig:cmleToyExample1fig4}]{\includegraphics[width=0.32\textwidth]{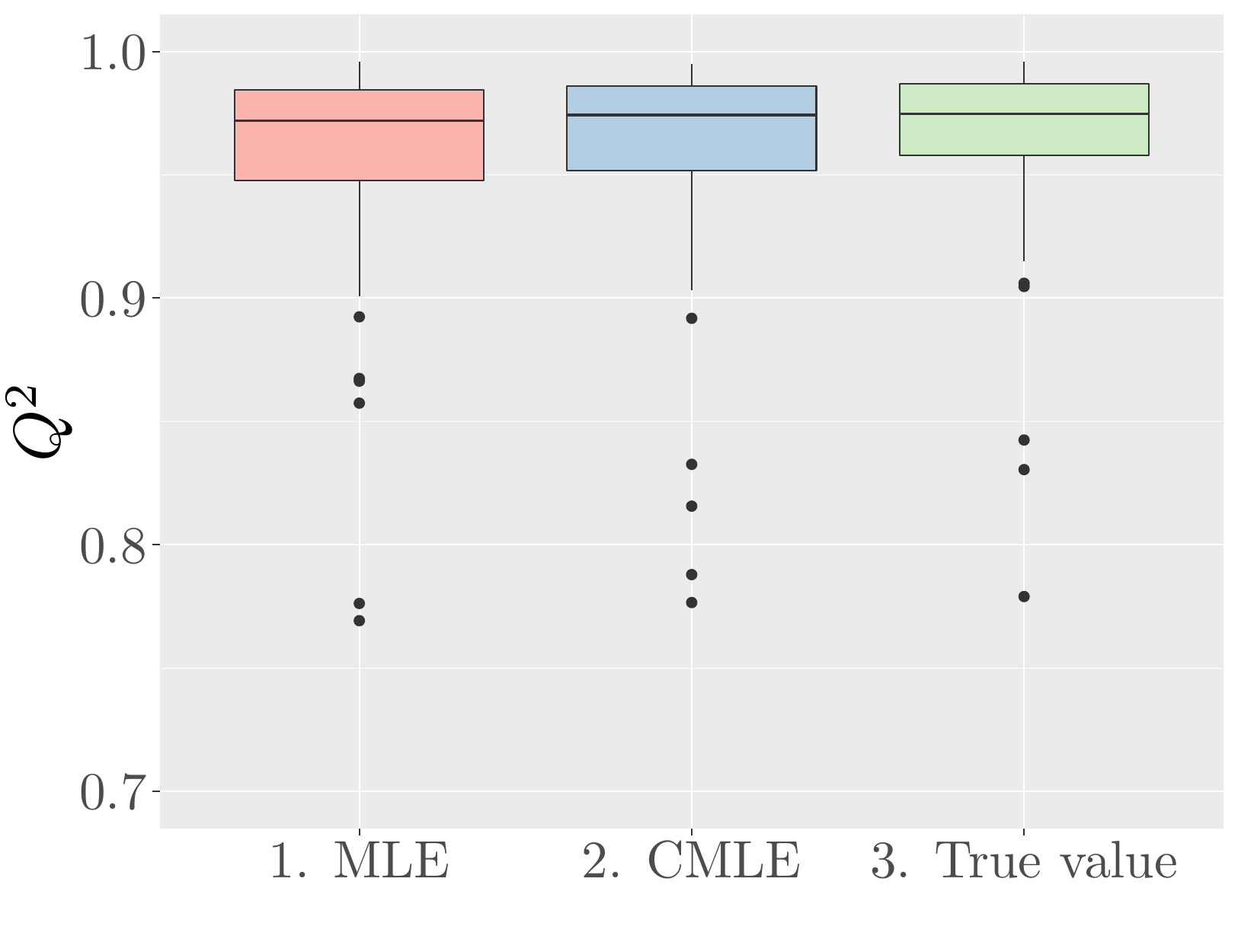}}
	\subfigure[\scriptsize \label{subfig:cmleToyExample1fig5}]{\includegraphics[width=0.32\textwidth]{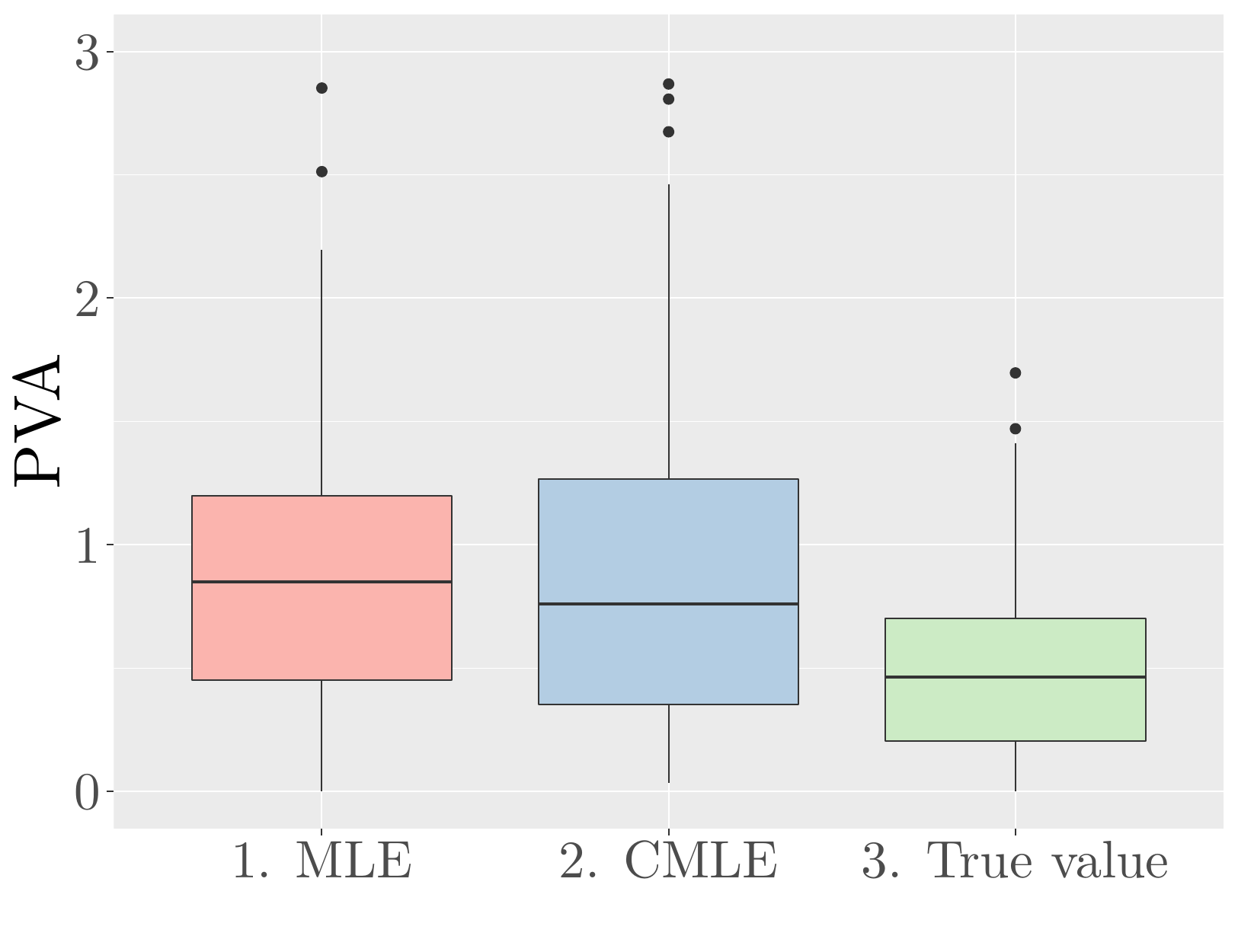}}
	
	\subfigure[\scriptsize \label{subfig:cmleToyExample1fig3d} $\widehat{\Btheta}_{\text{MLE}}$]{\includegraphics[width=0.32\textwidth]{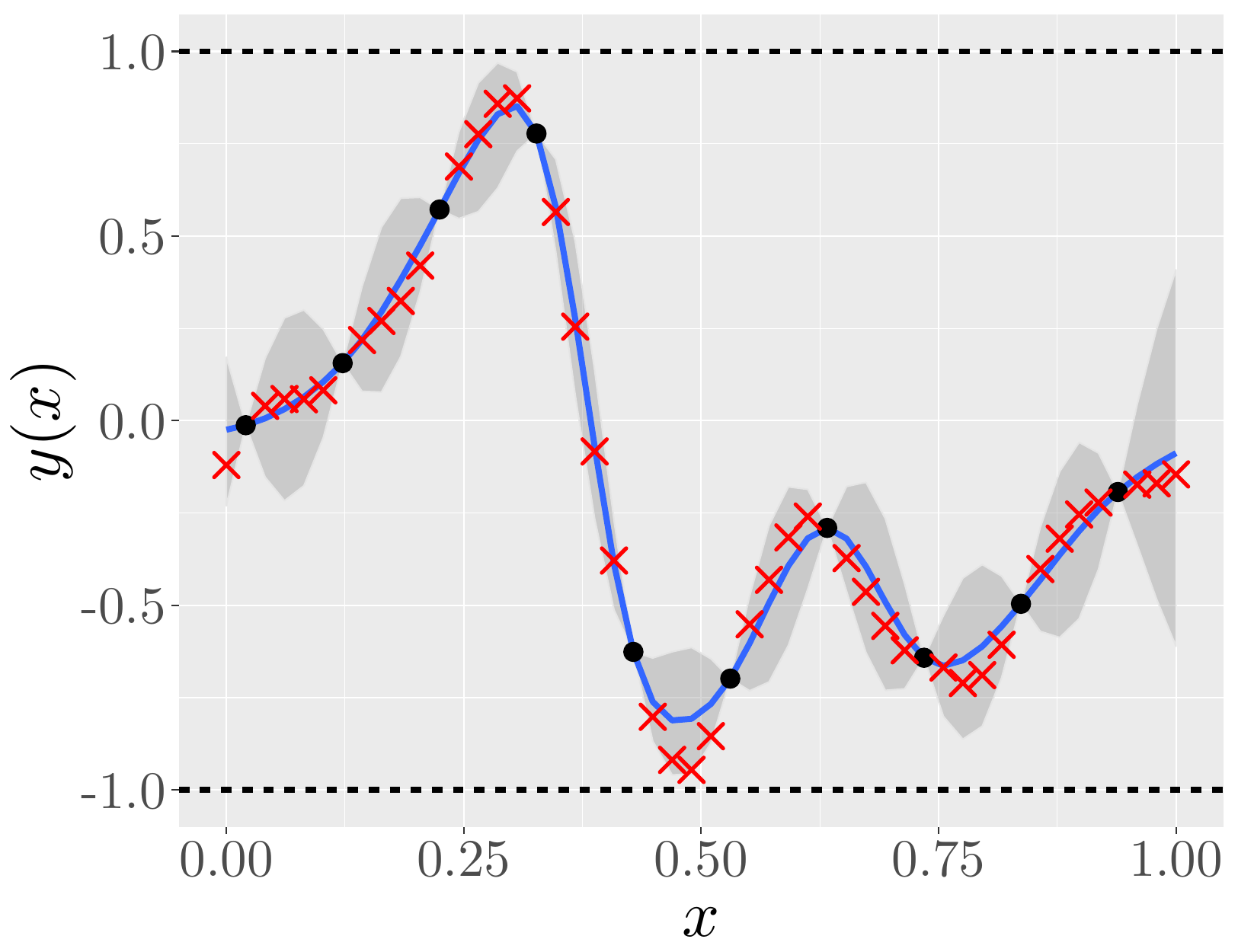}}
	\subfigure[\scriptsize \label{subfig:cmleToyExample1fig3e} $\widehat{\Btheta}_{\text{CMLE}}$]{\includegraphics[width=0.32\textwidth]{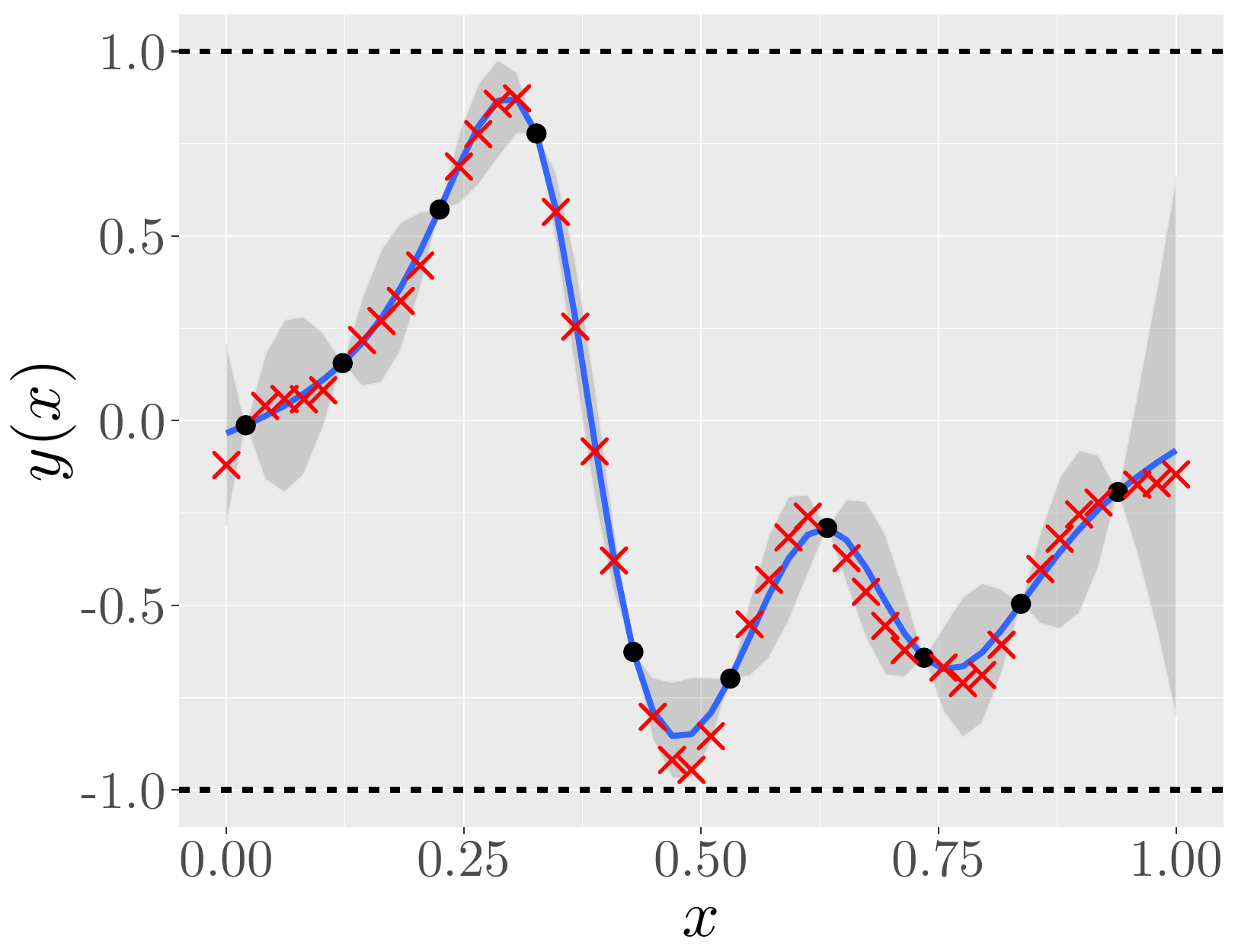}}	
	\subfigure[\scriptsize \label{subfig:cmleToyExample1fig3f} $\Btheta^\ast$]{\includegraphics[width=0.32\textwidth]{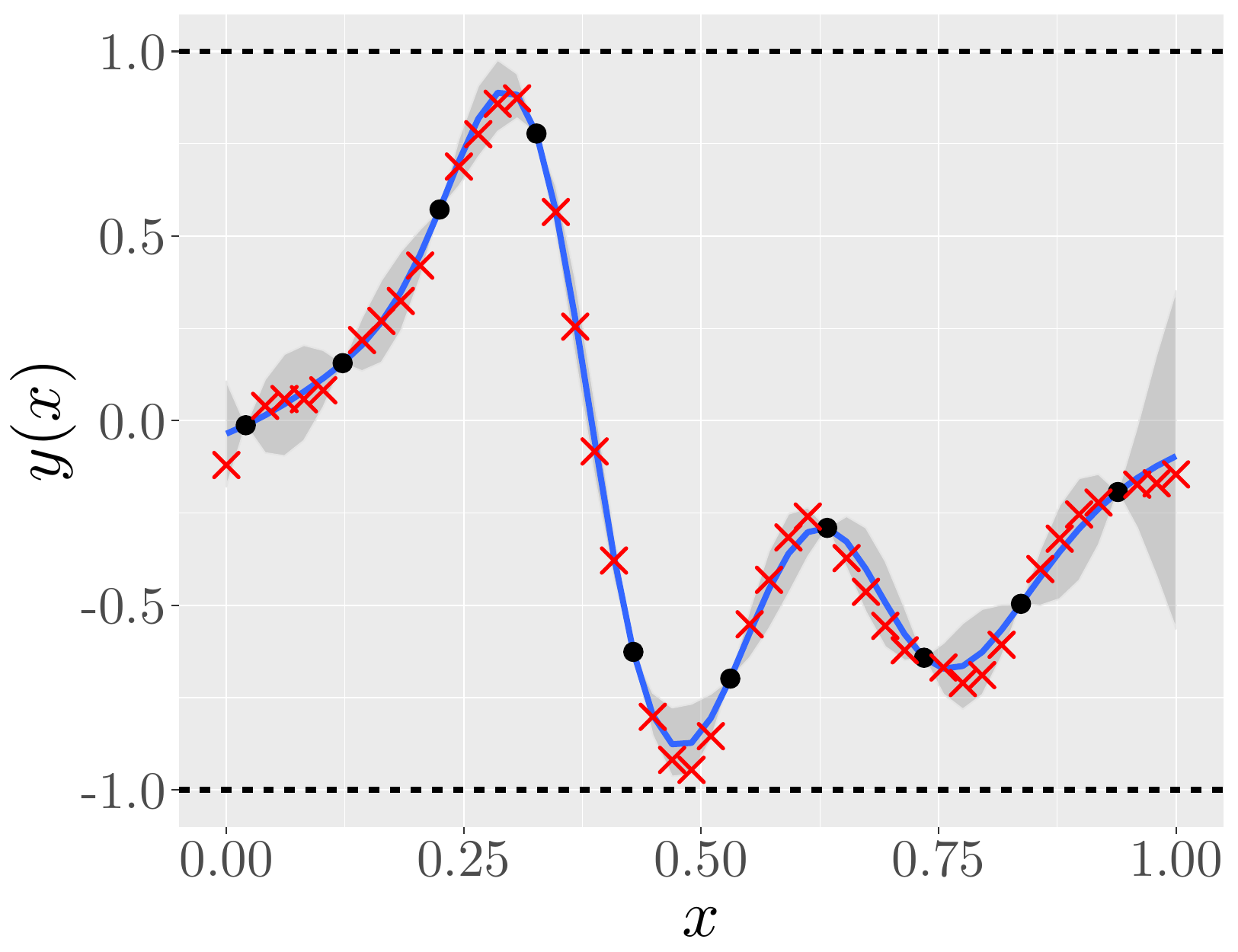}}
	
	\caption{Assessment of the likelihood (ML) and conditional likelihood (CML) estimators for 100 samples drawn from a GP with true parameters $\Btheta^\ast = (1, 0.2)$, and satisfying the bounds $[-1,1]$. \subref{subfig:cmleToyExample1fig2} Estimated values of the log-ratio $\log \rho^\ast = \log (1^2/0.2^5)$ (dashed green line) using MLE and CMLE. Predictive accuracies are evaluated using the \subref{subfig:cmleToyExample1fig4} $Q^2$ and \subref{subfig:cmleToyExample1fig5} PVA criteria. Predictions are showed for one sample using \subref{subfig:cmleToyExample1fig3d} $\widehat{\Btheta}_{\text{MLE}}$, \subref{subfig:cmleToyExample1fig3e} $\widehat{\Btheta}_{\text{CMLE}}$, and \subref{subfig:cmleToyExample1fig3f} $\Btheta^\ast$. For the predictions, panel description is the same as \Cref{fig:toyExample1b}.}
	\label{fig:cmleToyExample1fig2}
\end{figure}

We also evaluate the efficiency of the two estimators in terms of prediction accuracy. For each realization, we estimated the covariance parameters $\Btheta^\ast$ by MLE and CMLE. We then simulated the posterior at 50 new regularly spaced locations using the estimated covariance parameter $\widehat{\Btheta}$. The conditional sample paths were simulated via HMC. We used the $Q^2$ and predictive variance adequation (PVA) criteria to assess the quality of predictions over the 50 new values. Denoting by $n_t$ the number of test points, $z_1, \cdots, z_{n_t}$ and $\widehat{z}_1, \cdots, \widehat{z}_{n_t}$ the sets of test and predicted observations (respectively), then $Q^2 = 1 - \sum_{i = 1}^{n_t} (\widehat{z}_i - z_i)^2/\sum_{i = 1}^{n_t} (\overline{z} - z_i)^2$, where $\overline{z}$ is the mean of the test data. Notice that for noise-free observations, the $Q^2$ indicator is equal to one if the predictors $\widehat{z}_1, \cdots, \widehat{z}_{n_t}$ are exactly equal to the test data (ideal case), zero if they are equal to the constant prediction $\overline{z}$, and negative if they perform worse than $\overline{z}$. On the other hand, PVA assesses the quality of predictive variances $\widehat{\sigma}_i^2$, for all $i = 1, \cdots, n_t$, using the prediction errors: $\operatorname{PVA} = \big|\log \big( \frac{1}{{n_t}} \sum_{i=1}^{n_t} (z_i - \widehat{z}_i)^2/ \widehat{\sigma}_i^2 \big) \big|$ \citep{Bachoc2013CrossValidation}. Notice that, if the standardized residuals $(z_i - \widehat{z_i})/\widehat{\sigma_i}$ were an iid sample, then the PVA value would be close to $0$ by the law of large numbers. Departure from $0$ may indicate either a lack of independence, or a biased estimation of the prediction uncertainty $\sigma_i$. In that sense, smaller PVA values may correspond to more reliable confidence intervals.

\Cref{fig:cmleToyExample1fig2} shows the inferred sample paths for one realization using \ref{subfig:cmleToyExample1fig3d} $\widehat{\Btheta}_{\text{MLE}}$, \ref{subfig:cmleToyExample1fig3e} $\widehat{\Btheta}_{\text{CMLE}}$, and \ref{subfig:cmleToyExample1fig3f} $\Btheta^\ast$. We observe that, in the three cases, the models tend to fit properly the test data with accurate confidence intervals. According to \cref{subfig:cmleToyExample1fig4,subfig:cmleToyExample1fig5}, we see they provide $Q^2$ and PVA median values close to the ones obtained when the true $\Btheta^\ast$ is used. Although the predictive accuracies obtained using CMLE are slightly better than for MLE in terms of bias, we observe larger variances in the PVA criterion. Since the orthant Gaussian terms from the conditional likelihood of \Cref{eq:CML} have to be approximated, we believe that this affects the effectiveness of CMLE. Furthermore, existing estimators of orthant Gaussian probabilities present some numerical instabilities limiting the CML optimisation routine and providing suboptimal results. Finally, notice that MLE also provides reliable predictions. This suggests that, if we properly take into account the inequality constraints in the posterior distribution, the unconstrained ML optimisation can be used for practical implementation.

\subsection{Asymptotic properties}
\label{subsec:asymptotic}
Now, we study the asymptotic properties of likelihood-based estimators for constrained GPs. We consider the fixed-domain asymptotic setting \citep{stein1999interpolation}, with a dense sequence of observation points in a bounded domain. It should be noted that, when the GP is not constrained, significant contributions have been provided to study the consistency or asymptotic normality of the ML estimator \citep{Zhang2004InconsistentEst,Du2009AsympMLE,Ying1993ExpCov,Loh2005Matern,Loh2000Matern32}. In this paper, we show that, loosely speaking, any consistency result for ML with unconstrained GPs, is preserved when adding either boundedness, monotonicity or convexity constraint. Furthermore, this consistency occurs for both the unconditional and conditional likelihood functions.

For $\kappa \in \{0,1,2\}$, let $Y$ be a GP with $C^\kappa$ trajectories on a bounded set $\mathbb{X} \subset \realset{d}$. Let $\mathcal{E}_\kappa$ be one of the following convex set of functions
\begin{equation}
	\mathcal{E}_\kappa = 
	\begin{cases}
	f \ : \ \mathbb{X} \to \realset{}, f \mbox{ is } C^0 \mbox{ and } \forall \Bx \in \mathbb{X}, \ \ell \leq f(\Bx) \leq u & \mbox{if } \kappa = 0, \\
	f \ : \ \mathbb{X} \to \realset{}, f \mbox{ is } C^1 \mbox{ and } \forall \Bx \in \mathbb{X}, \ \forall i = 1, \cdots, d, \ \frac{\partial}{\partial x_i} f(\Bx) \geq 0 & \mbox{if } \kappa = 1, \\
	f \ : \ \mathbb{X} \to \realset{}, f \mbox{ is } C^2 \mbox{ and } \forall \Bx \in \mathbb{X}, \ \frac{\partial^2}{\partial \Bx^2} f(\Bx) \mbox{ is a non-negative definite matrix} & \mbox{if } \kappa = 2.
	\end{cases}
	\label{eq:convexsetFun2}
\end{equation}
For the purpose of asymptotic analysis, we do not consider the hat basis functions anymore, and we focus on the GP $Y$ and the observation vector $\BY_n = [\begin{smallmatrix} Y(x_1), & \cdots, & Y(x_n) \end{smallmatrix}]^\top$. We study the (unconstrained) likelihood function based on $p_\Btheta (\BY_n)$ and the constrained likelihood function based on $p_\Btheta (\BY_n| Y \in \mathcal{E}_\kappa)$. Notice that these quantities are more challenging to evaluate in practice than for \Cref{subsec:cml,subsec:cmlresults}, but the purpose is a theoretical analysis.
%In \cref{prop:ml,prop:cml}, we investigate some asymptotic properties of the ML and CML estimators.

In \cref{prop:ml}, we prove that if ML is consistent, when considering the (unconditional) distribution of $Y$, then it remains consistent when conditioning to $Y \in \mathcal{E}_\kappa$. In \cref{prop:cml}, we prove that, under mild conditions, implying the consistency of the ML estimator with the (unconditional) distribution of $Y$, the CML remains consistent when adding the constraint $Y \in \mathcal{E}_\kappa$. The proofs of \cref{prop:ml,prop:cml} require supplementary conditions and lemmas, which are given in \cref{app:asymtotic}.
\begin{proposition}
	\label{prop:ml}
	Let $Y$ be a zero-mean GP on a bounded set $\mathbb{X} \subset \realset{d}$ with covariance function $k$ satisfying Condition \ref{cond:Balls}. 
	Let $\BTheta$ be a compact set on $(0,\infty)^{d+1}$.
	Let $k_\Btheta$ be the covariance function of $x \to \sigma Y(\theta_1 x_1, \cdots, \theta_d x_d)$ for $\Btheta = (\sigma^2, \theta_1, \cdots, \theta_d) \in \BTheta$. 
	Let $\Btheta^\ast = (1, \cdots, 1)$. Remark that $k = k_{\Btheta^\ast}$ and assume that $\Btheta^\ast \in \BTheta$.
	Let $(\Bx_i)_{i \in \mathds{N}}$ be a dense sequence in $\mathbb{X}$.
	Let $\BY_n = [\begin{smallmatrix} Y(x_1), & \cdots, & Y(x_n) \end{smallmatrix}]^\top$.  
	Let 
	\begin{displaymath}
	\mathcal{L}_{n}(\Btheta) = - \frac{1}{2} \log (\det (\BR_\Btheta)) - \frac{1}{2} \BY_n^\top \BR_\Btheta^{-1} \BY_n - \frac{n}{2} \log 2 \pi,
	\end{displaymath}
	with $\BR_\Btheta = (k_\Btheta(x_i,x_j))_{1\leq i,j\leq n}$.
	Let $\widehat{\Btheta} \in \operatorname{\arg \max}\limits_{\Btheta \in \BTheta} \ \mathcal{L}_{n}(\Btheta)$.
	Assume that $\forall \varepsilon > 0$, 
	\begin{displaymath}
	P(\|\widehat{\Btheta}- \Btheta^\ast \|\geq \varepsilon) \xrightarrow[n \to \infty]{} 0.
	\end{displaymath}
	Let $\kappa \in \{0,1,2\}$. Let $\mathcal{E}_\kappa$ be as in \Cref{eq:convexsetFun2}. Then, we have $P(Y \in \mathcal{E}_\kappa) > 0$ from \Cref{lem:nonzero,lem:nonzeroC1,lem:nonzeroC2}, and thus
	\begin{displaymath}
	P(\|\widehat{\Btheta}- \Btheta^\ast \|\geq \varepsilon \ |\ Y \in \mathcal{E}_\kappa) \xrightarrow[n \to \infty]{} 0.
	\end{displaymath}
\end{proposition}
\begin{proof}
	We have
	\begin{displaymath}
	P(\|\widehat{\Btheta}- \Btheta^\ast \| \geq \ \varepsilon | \ Y \in \mathcal{E}_\kappa) = \frac{P(\|\widehat{\Btheta}- \Btheta^\ast\| \geq \varepsilon, \ Y \in \mathcal{E}_\kappa)}{P(Y \in \mathcal{E}_\kappa)} \leq \frac{P(\|\widehat{\Btheta}- \Btheta^\ast\| \geq \varepsilon)}{P(Y \in \mathcal{E}_\kappa)}.
	\end{displaymath}
	Since $P(Y \in \mathcal{E}_\kappa) > 0 $ is fixed, and $P(\|\widehat{\Btheta}- \Btheta^\ast\| \geq \varepsilon) \xrightarrow[n \to \infty]{} 0$, the result follows.
\end{proof}

\begin{proposition}
	\label{prop:cml}	
	We use the same notations and assumptions as in \cref{prop:ml}.
	Let $\kappa \in \{0,1,2\}$ be fixed. Let $P_\Btheta$ be the distribution of $Y$ with covariance function $k_\Btheta$.
	Let 
	\begin{displaymath}
	\mathcal{L}_{\mathcal{C},n}(\Btheta) = \mathcal{L}_{n}(\Btheta) + \log P_\Btheta (Y \in \mathcal{E}_\kappa| \BY_n) - \log P_\Btheta (Y \in \mathcal{E}_\kappa).
	\end{displaymath}
	Assume that $\forall \varepsilon > 0$ and $\forall M < \infty$,
	\begin{displaymath}
	P \bigg( \sup\limits_{\|\Btheta-\Btheta^\ast\|\geq\varepsilon} (\mathcal{L}_{n}(\Btheta) - \mathcal{L}_{n}(\Btheta^\ast)) \geq -M  \bigg) \xrightarrow[n\to \infty]{} 0.
	\end{displaymath}
	Then,
	\begin{displaymath}
	P \bigg( \sup\limits_{\|\Btheta-\Btheta^\ast\|\geq\varepsilon} (\mathcal{L}_{\mathcal{C},n}(\Btheta) - \mathcal{L}_{\mathcal{C},n}(\Btheta^\ast)) \geq -M  \ \bigg| \ Y \in \mathcal{E}_\kappa \bigg) \xrightarrow[n\to \infty]{} 0.
	\end{displaymath}	
	Consequently
	\begin{displaymath}
	\operatorname{argmax}\limits_{\Btheta \in \BTheta} \ \mathcal{L}_{n}(\Btheta) \xrightarrow[n\to \infty]{P} \Btheta^\ast, \qquad \mbox{and} \qquad \operatorname{argmax}\limits_{\Btheta \in \BTheta} \ \mathcal{L}_{\mathcal{C},n}(\Btheta) \xrightarrow[n\to \infty]{P|Y \in \mathcal{E}_\kappa} \Btheta^\ast,
	\end{displaymath}	
	where $\xrightarrow[n\to \infty]{P}$ denotes the convergence in probability under the distribution of $Y$, and $\xrightarrow[n\to \infty]{P|Y \in \mathcal{E}_\kappa}$ denotes the convergence in probability under the distribution of $Y$ given $Y \in \mathcal{E}_\kappa$.
\end{proposition}
\begin{proof}
	We have from \cref{lem:lb,lem:lbd} that $\forall \varepsilon > 0$
	\begin{displaymath}
	P\{\log(P_{\Btheta^\ast}(Y \in \mathcal{E}_\kappa \ | \ \BY_n) ) \geq \log(1-\varepsilon) \ | \ Y \in \mathcal{E}_\kappa \} \xrightarrow[n \to \infty]{} 1.
	\end{displaymath}
	Hence $\forall \delta > 0$
	\begin{displaymath}
	P\bigg\{\sup\limits_{\|\Btheta-\Btheta^\ast\|\geq\varepsilon} \ \log(P_{\Btheta}(Y \in \mathcal{E}_\kappa \ | \ \BY_n)) - \log(P_{\Btheta^\ast}(Y \in \mathcal{E}_\kappa \ | \ \BY_n) ) \geq \delta \ \bigg| \ Y \in \mathcal{E}_\kappa \bigg\} \xrightarrow[n \to \infty]{} 0.
	\end{displaymath}
	Also, from \cref{lem:continuity}, there exists $\Delta > 0$  so that we have 
	\begin{displaymath}
	\inf\limits_{\|\Btheta-\Btheta^\ast\|\geq\varepsilon} P_{\Btheta}(Y \in \mathcal{E}_\kappa) \geq \Delta > 0,
	\end{displaymath}
	so that 
	\begin{displaymath}
	\sup\limits_{\|\Btheta-\Btheta^\ast\|\geq\varepsilon} -\log(P_{\Btheta}(Y \in \mathcal{E}_\kappa)) + \log(P_{\Btheta^\ast}(Y \in \mathcal{E}_\kappa) ) \leq - \log(\Delta) < \infty.
	\end{displaymath}
	Hence, the proposition follows.
\end{proof}

\section{Extension to multidimensional input spaces}
\label{sec:2Dapp}

\subsection{2D dimensional case}
\label{subsec:2Dext}
The finite-dimensional Gaussian representation of \Cref{sec:GPwithICLinear} can be extended to $d$ dimensional input spaces by tensorisation. For readability, we focus on the case $d=2$ with $\mathcal{D} = [0, 1]^2$ and $m_1 \times m_2$ knots located on a regular grid. Then, the finite approximation is given by
\begin{equation}
	Y_{m_1,m_2}(x_1,x_2) := \sum_{j_1 = 1}^{m_1} \sum_{j_2 = 1}^{m_2} \xi_{j_2,j_1} \phi_{j_1}^{1}(x_1) \phi_{j_2}^{2}(x_2), \ \mbox{s.t.} \ \begin{cases} Y_{m_1,m_2}\left(x_1^{i},x_2^{i}\right) = y_i, \; {(i = 1, \cdots, n)}\\ \xi_{j_2,j_1} \in \mathcal{C}, \end{cases}
	\label{eq:finApprox2D}
\end{equation}
where $\xi_{j_2,j_1} = Y(t_{j_1},t_{j_2})$ and $(x_1^1, x_1^2), \cdots, (x_1^n, x_1^n)$ constitute a DoE. If we follow a similar procedure as in \Cref{sec:GPwithICLinear}, we observe that $\Bxi = [\begin{smallmatrix} \xi_{1,1}, & \cdots, & \xi_{1,m1}, \cdots, \xi_{m2,1}, \cdots \xi_{m2,m1} \end{smallmatrix}]^\top$ is a zero-mean Gaussian vector with covariance matrix $\BGamma$ as in \Cref{eq:trGP2}. Notice that each row of the new matrix $\BPhi$ is given by
\begin{equation*}
\BPhi_{i,:} = \begin{bmatrix} \phi_{1}^1(x_1^i) \phi_{1}^2(x_2^i) & \cdots & \phi_{m_1}^1(x_1^i) \phi_{1}^2(x_2^i) & \cdots & \phi_{1}^1(x_1^i) \phi_{m_2}^2(x_2^i) & \cdots & \phi_{m_1}^1(x_1^i) \phi_{m_2}^2(x_2^i) \end{bmatrix},
\end{equation*}
% $\BPhi_{i,:} = \BPhi_{i,:}^1 \otimes \BPhi_{i,:}^2$, for all $i = 1, \cdots, n$, where $\otimes$ is the Kronecker product and $\BPhi_{i,:}^d$ are the hat basis functions associated to the observation $i$ and the input dimension $d$.
for all $i = 1, \cdots, n$. Finally, the posterior distribution of \Cref{eq:trposterior} can be computed, and the routine follows \cref{alg:tmKriging}. The module from \Cref{eq:finApprox2D} was also used in \citep{Maatouk2016GPineqconst} for monotonicity constraints in multidimensional cases.

\begin{figure}
	\centering
	\subfigure[\scriptsize \label{subfig:toyExampleB2Dfig1} Boundedness in 2D.]{\includegraphics[width=0.37\textwidth]{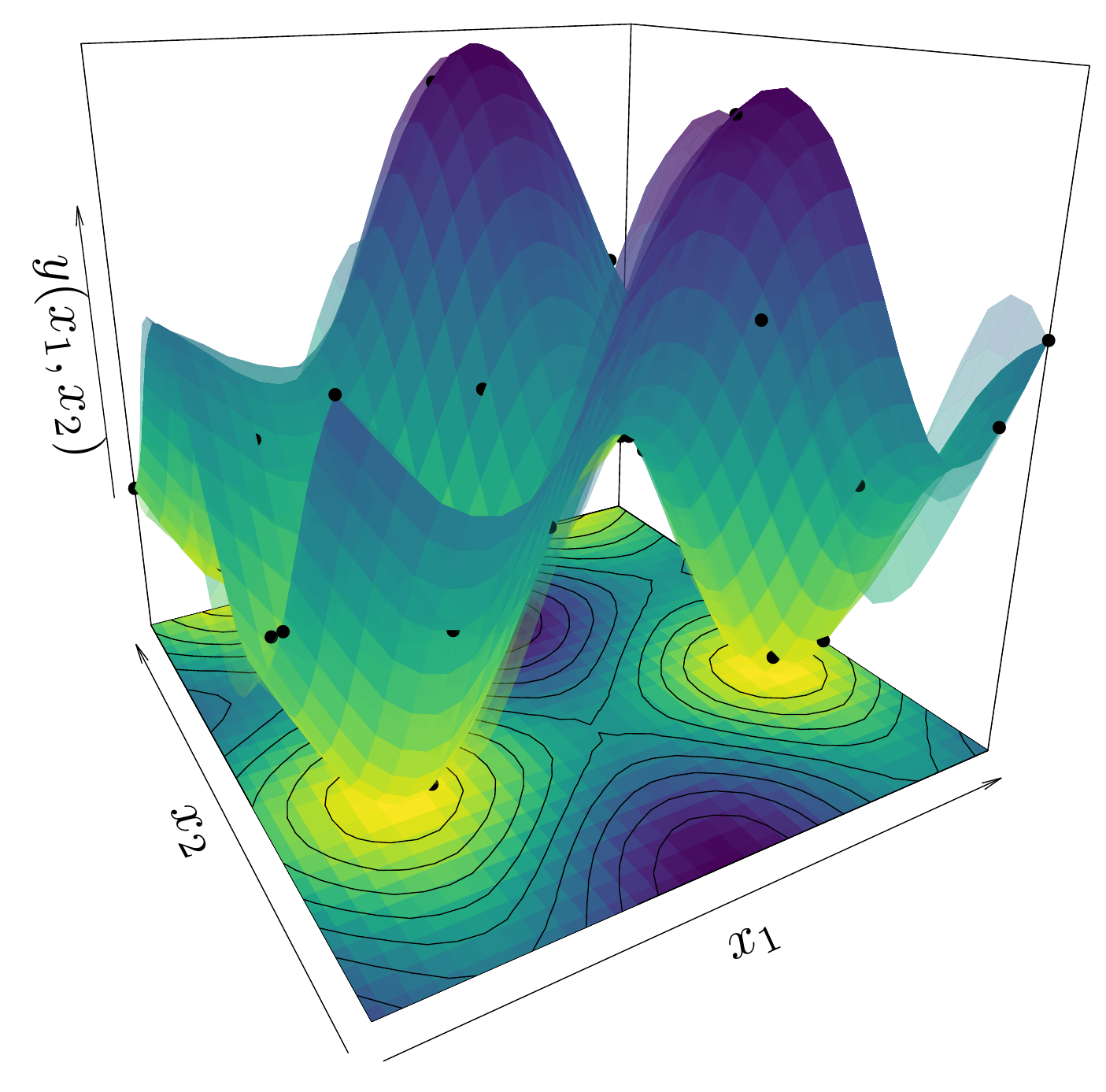}} \hspace{20pt}
	\subfigure[\scriptsize \label{subfig:toyExampleM2Dfig1} Monotonicity in 2D.]{\includegraphics[width=0.37\textwidth]{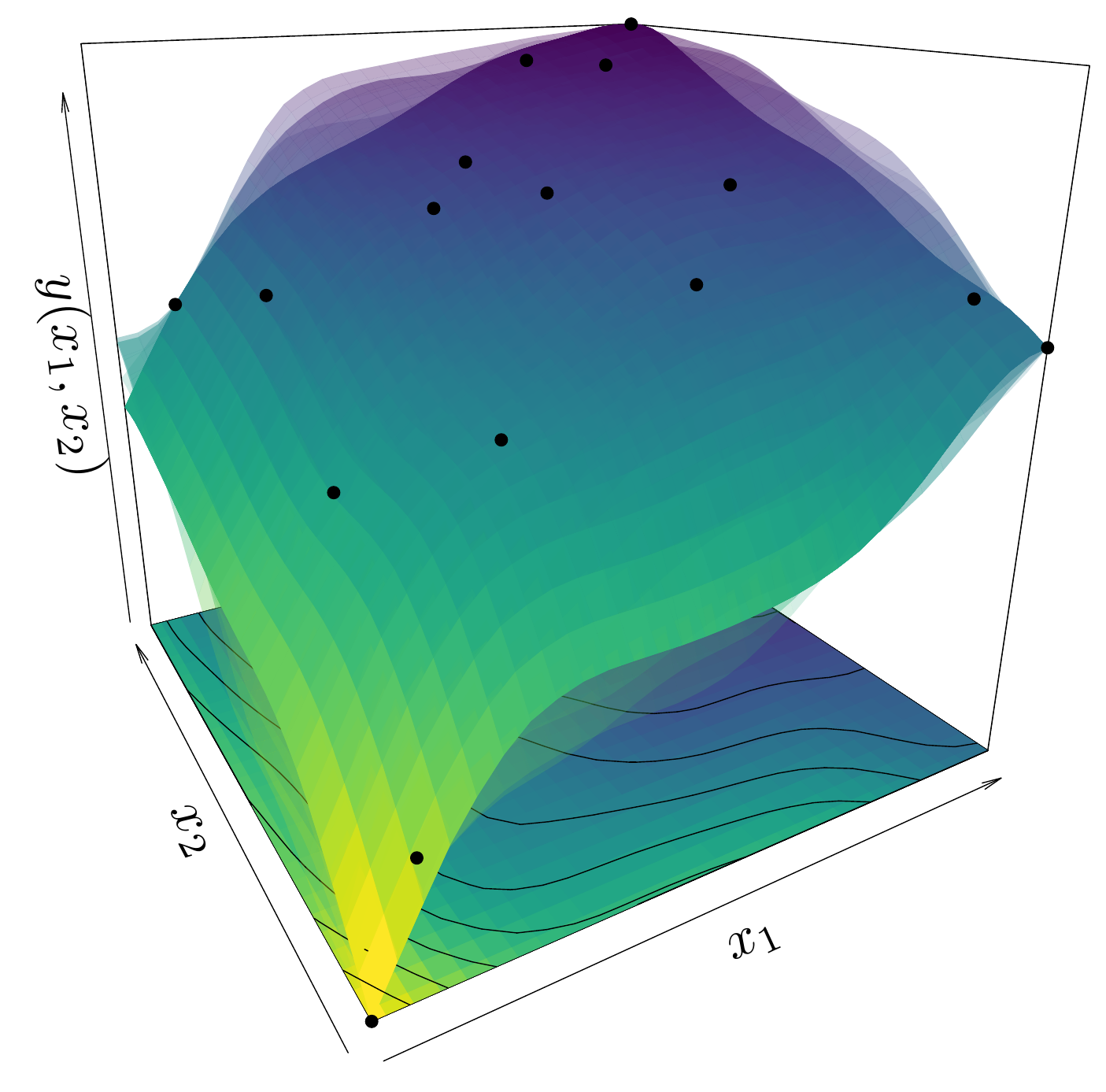}}
	\caption{Examples of 2D Gaussian models with different types of constraints for interpolating the toy examples from \Cref{subsec:2Dext}. Each panel shows: training points (black dots), the conditional  function (solid surface), and some conditional realizations (light surfaces).}
	\label{fig:2DToyExample1fig1}
\end{figure}

\Cref{fig:2DToyExample1fig1} shows two examples where boundedness or monotonicity are exhibited. We used a 2D SE covariance function with parameters $(\sigma^2 = 1.0, \; \theta_1 = 0.2, \; \theta_2 = 0.2)$.\footnote{2D SE covariance function: $k_\Btheta(\Bx - \Bx') = \sigma^2 \exp\left\{-\frac{(x_1-x_1')^2}{2 \theta_1^2} -\frac{(x_2-x_2')^2}{2 \theta_2^2} \right\}$ with $\Btheta = (\sigma^2,\theta_1,\theta_2)$.} The training points were generated with a maximin Latin hypercube DoE over $[0, 1]^2$. The functions are: \ref{subfig:toyExampleB2Dfig1} $y(x_1,x_2) = - \frac{1}{2} [\sin(9 x_1) - \cos(9 x_2)]$, and \ref{subfig:toyExampleM2Dfig1} $y(x_1,x_2) = \arctan(5 x_1) + \arctan(x_2)$.

\subsection{2D application: nuclear criticality safety}
\label{subsec:realdata}

For assessing the stability of neutron production in nuclear reactors, safety criteria based on the effective neutron multiplication factor $k_{\text{eff}}$ are commonly used \citep{IAEA2014,Fernex2005moret}. This factor is defined as the ratio of the total number of neutrons produced by a fission chain reaction to the total number of neutrons lost by absorption and leakage. Besides the geometry and composition of fissile materials (e.g. mass, density), $k_{\text{eff}}$ is sensitive to other types of parameters like the structure materials characteristics (e.g. concrete), and the presence of specific materials (e.g. moderators). Since the optimal control of an individual parameter or a combination of them can lead to safe conditions, the understanding of their influence in criticality safety assessment is crucial.

In this section, we applied the proposed framework to a dataset provided by the ``Institut de Radioprotection et de S\^uret\'e Nucl\'eaire'' (IRSN), France. The $k_{\text{eff}}$ factor was obtained from a nuclear reactor called ``Lady Godiva device'' originally situated at Los Alamos National Laboratory (LANL), New Mexico, U.S., where uranium materials were managed. Two input parameters of the uranium sphere are considered: its radius $r$ and its density $d$. The dataset contains 121 observations in a regular grid of $11 \times 11$ locations (see \Cref{fig:godivaData}). Notice that, on the domain considered for the input variables, $k_{\text{eff}}$ increases as the geometry and density of the uranium sphere increase.

We trained different Gaussian models whether the inequality constraints are considered or not. For all the models, we normalised the input space to be in $[0, 1]^2$. We used the same 2D SE covariance functions as for the example from \Cref{fig:2DToyExample1fig1}. For the unconstrained model, we used multistart for the ML optimisation with six initial vectors of covariance parameters $\BTheta = (\sigma^2,  \theta_1, \theta_2)$ with $\sigma^2 \in [0.2, 1]$ and $\theta_1,\theta_2 \in [0.1, 0.9]$. For the constrained models, since the $k_{\text{eff}}$ factor indicates the production rate of neutron population, the output of the constrained processes has to be positive. Taking also into account the non-decreasing behaviour, we also consider the monotonicity constraints. We estimated the covariance parameters by MLE or CMLE. We trained both unconstrained and constrained models with a fixed maximin Latin hypercube DoE at eight locations extracted from the normalised grid. We used the remaining data to asses the quality of prediction tasks. 
\begin{figure}
	\centering
	\includegraphics[width=0.40\textwidth]{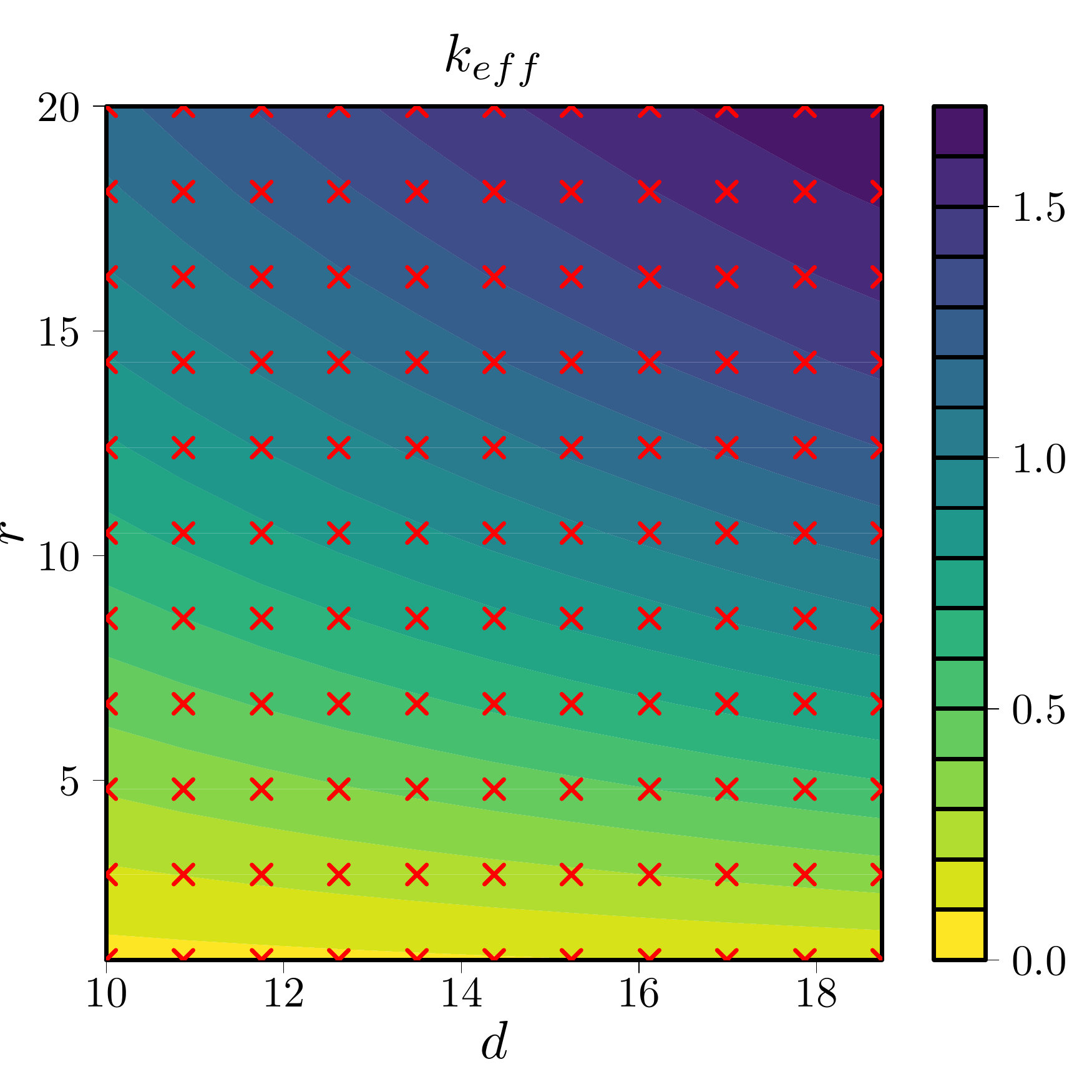} \hspace{20pt}
	\includegraphics[width=0.38\textwidth]{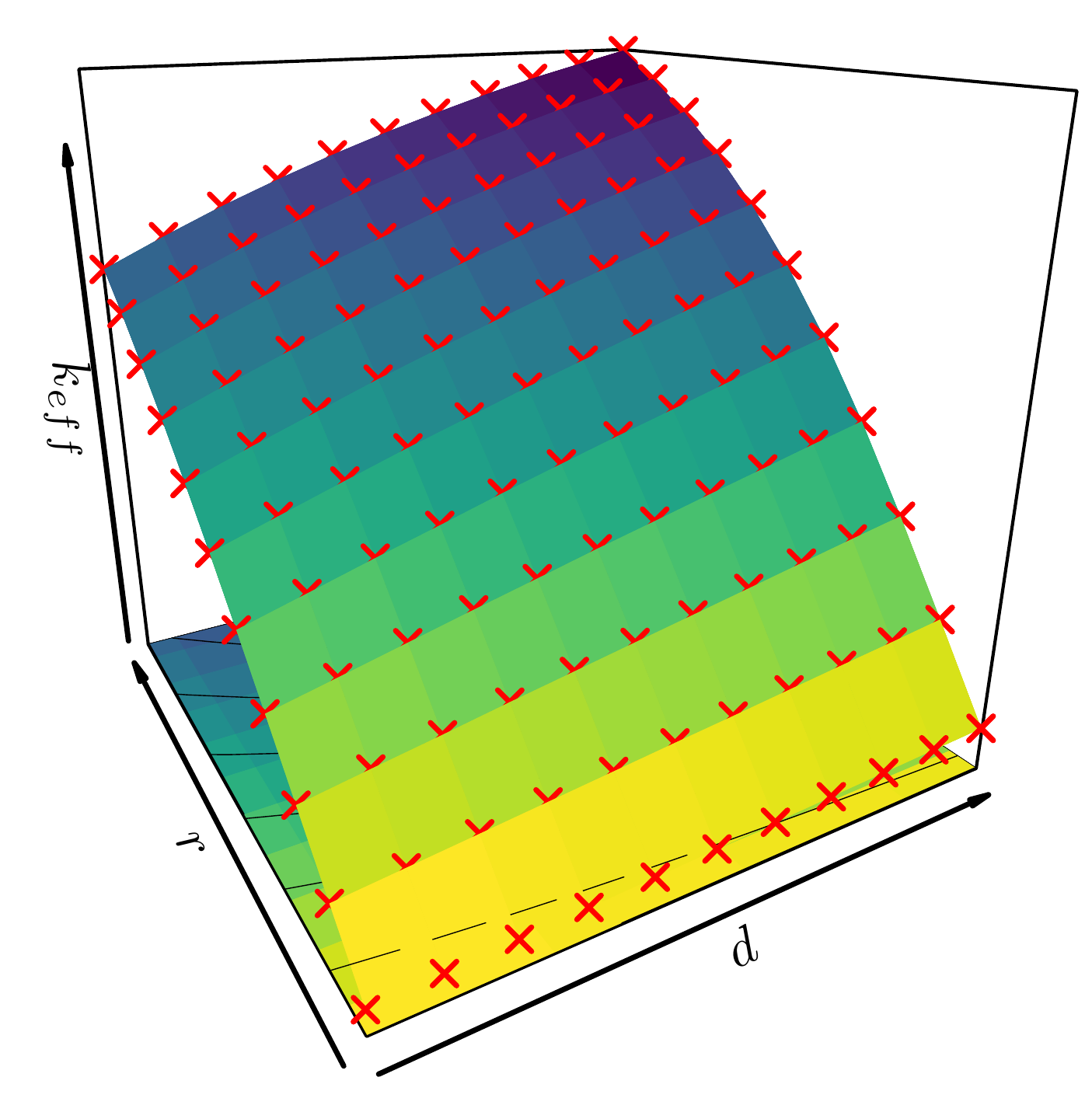}
	
	\caption{Nuclear criticality safety assessments: Godiva's dataset. (Left) 2D visualization of the $k_{\text{eff}}$ values measured over a regular grid. (Right) 3D visualization of the $k_{\text{eff}}$ data.}
	\label{fig:godivaData}
\end{figure}

\Cref{fig:godiva} shows the performance of the proposed models using four or eight points from the proposed fixed DoE. For the unconstrained models, we observe that the quality of the predictions depends strongly on both the amount of training data and their distribution in the input space. Notice from \Cref{subfig:godivaExampleSKMLEfig2} that if only few training points are available, predictions are poor and they do not satisfy positive and non-decreasing behaviours. In \Cref{subfig:godivaExampleSKMLEfig4}, we observe that if the training data are large enough and cover the input space, the unconstrained model behaves well and provides reliable predictions. On the other hand, we observe that the constrained models produce accurate prediction results also when the training set is small. 
\begin{figure}
	\centering
	\subfigure[\label{subfig:godivaExampleSKMLEfig2} $Q^2 = 0.07696$]{\includegraphics[width=0.32\textwidth]{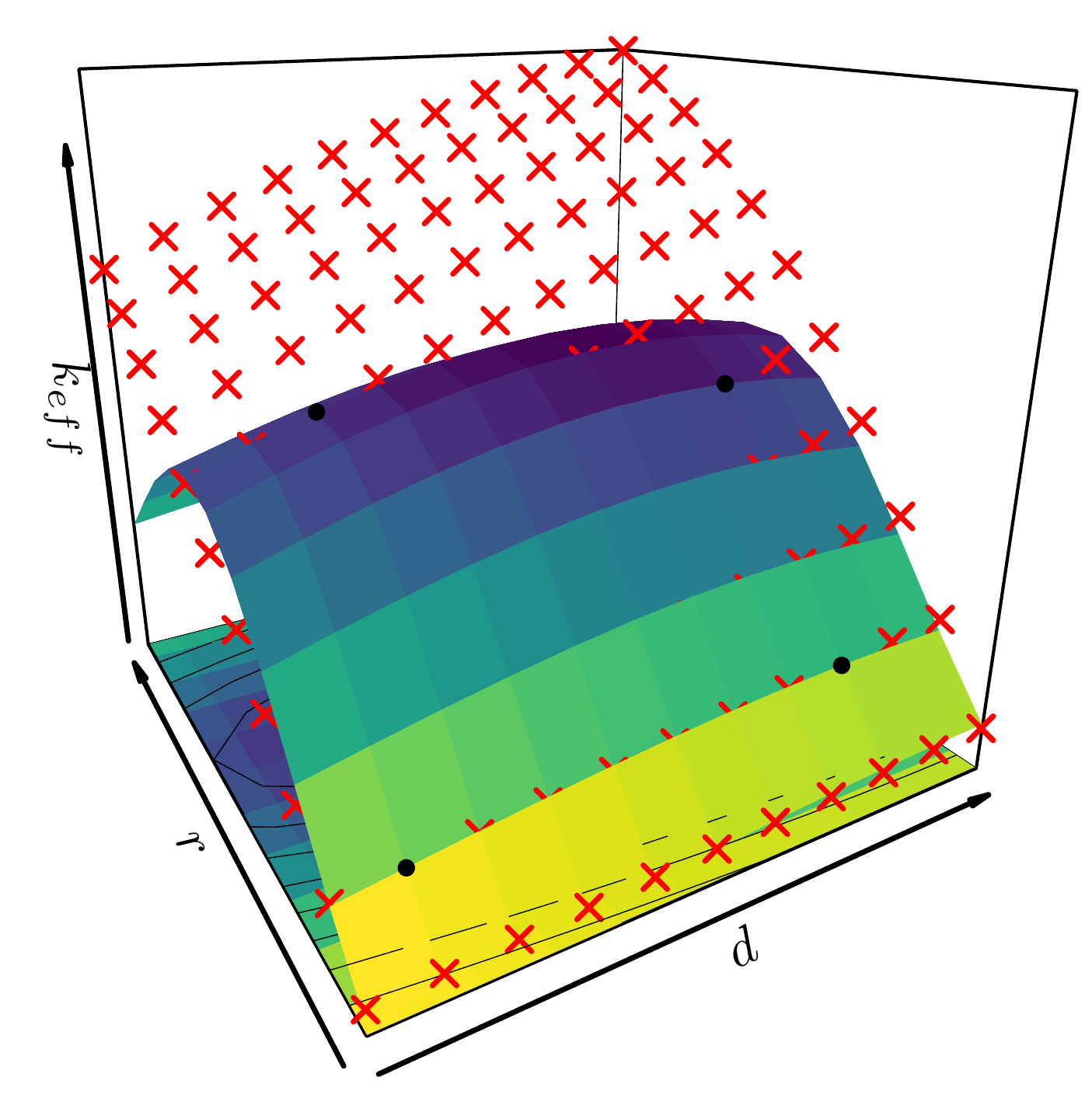}} 
	\subfigure[\label{subfig:godivaExampleCKMLEfig2} $Q^2 = 0.98225$]{\includegraphics[width=0.32\textwidth]{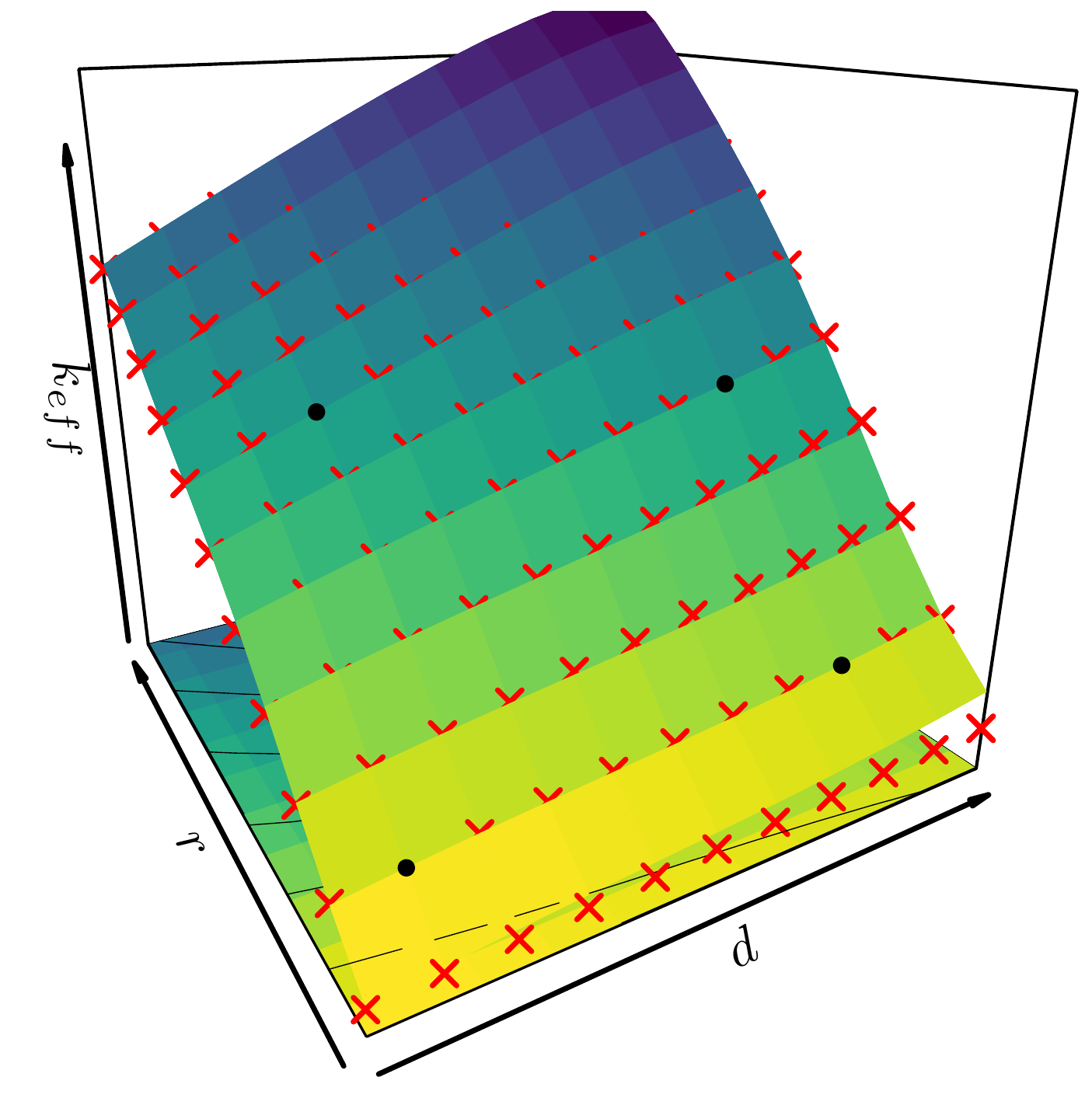}} 
	\subfigure[\label{subfig:godivaExampleCKCMLEfig2} $Q^2 = 0.90887$]{\includegraphics[width=0.32\textwidth]{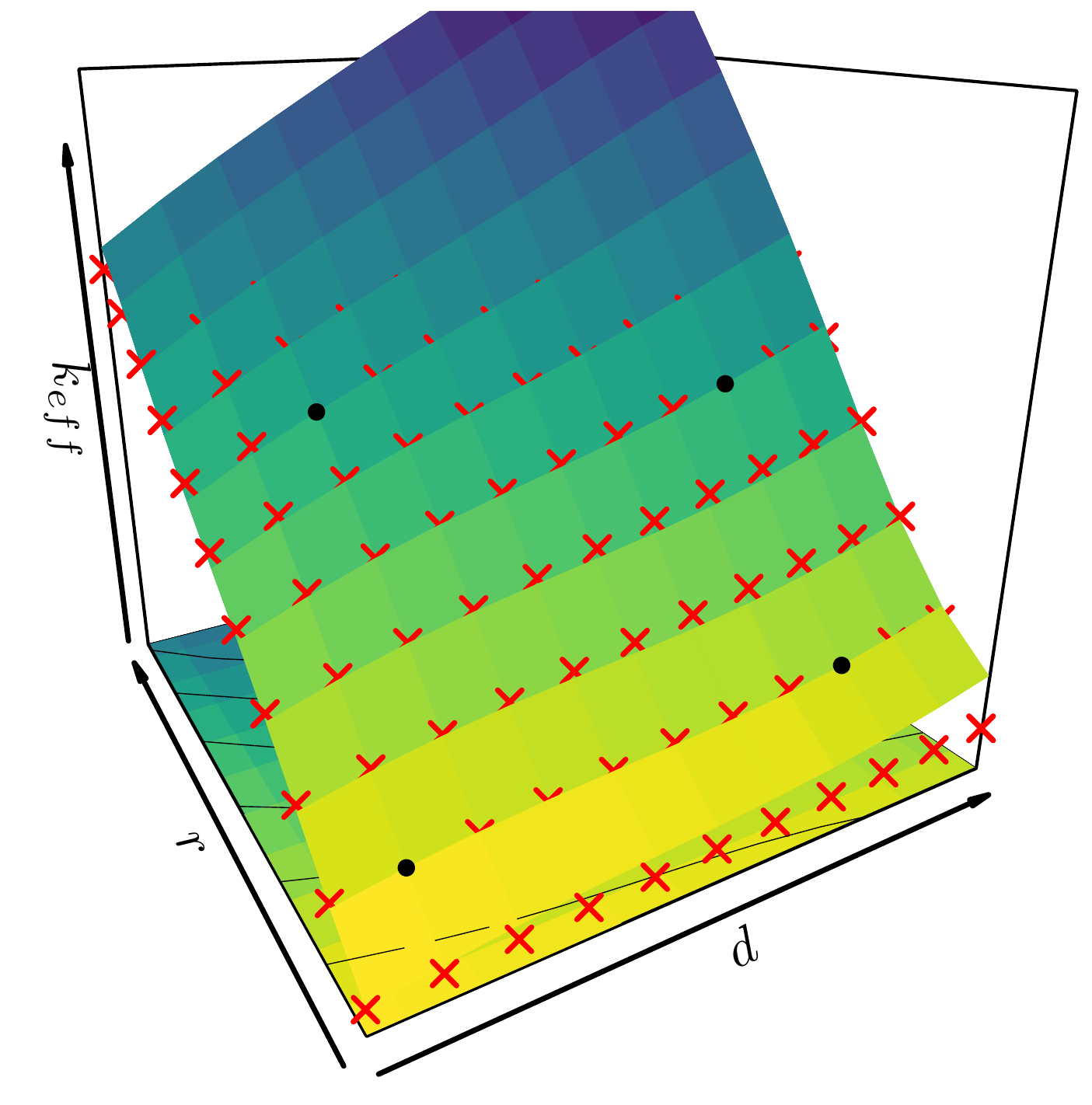}}
	
	\subfigure[\label{subfig:godivaExampleSKMLEfig4} $Q^2 = 0.99767$]{\includegraphics[width=0.32\textwidth]{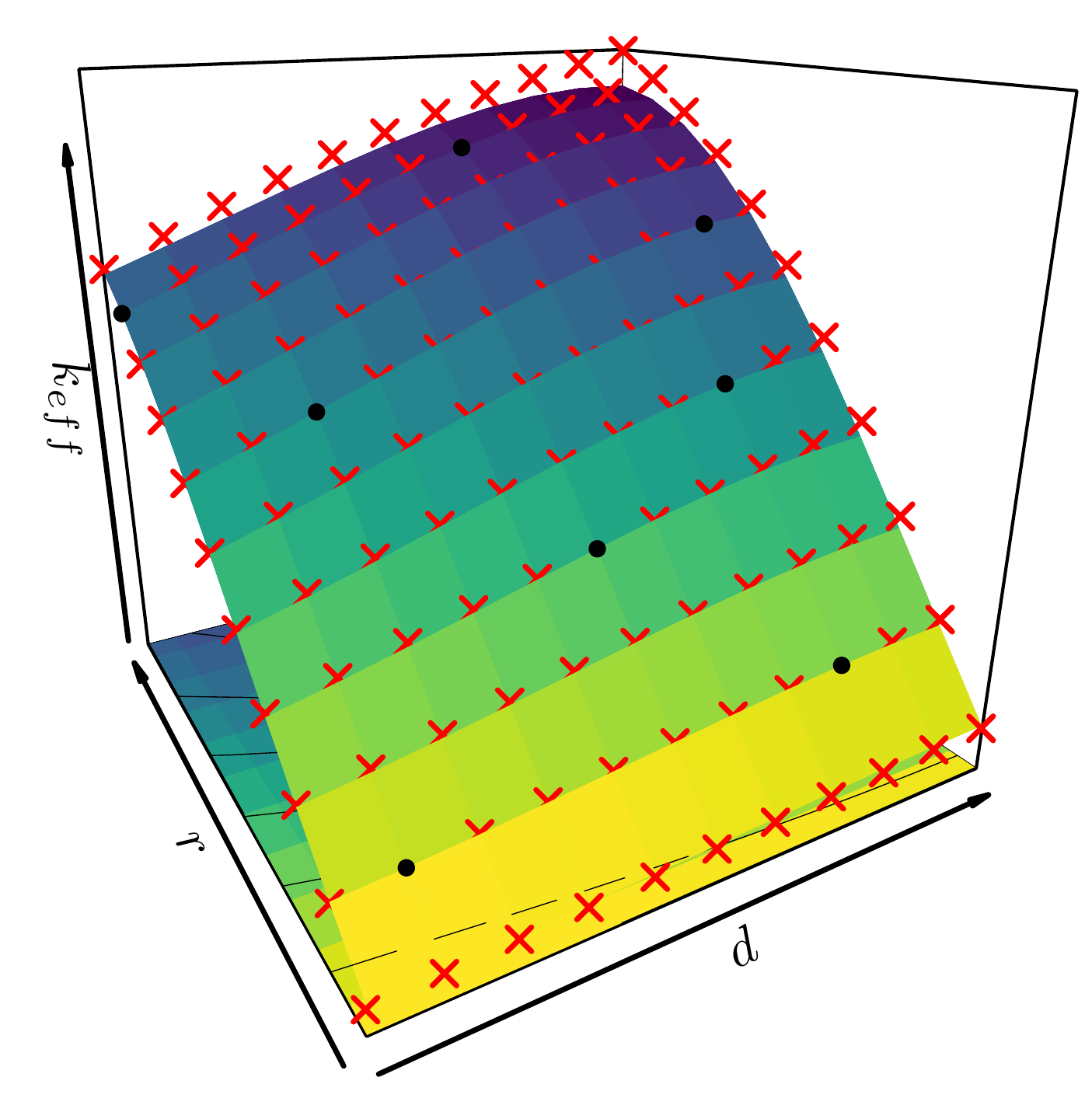}}
	\subfigure[\label{subfig:godivaExampleCKMLEfig4} $Q^2 = 0.99930$]{\includegraphics[width=0.32\textwidth]{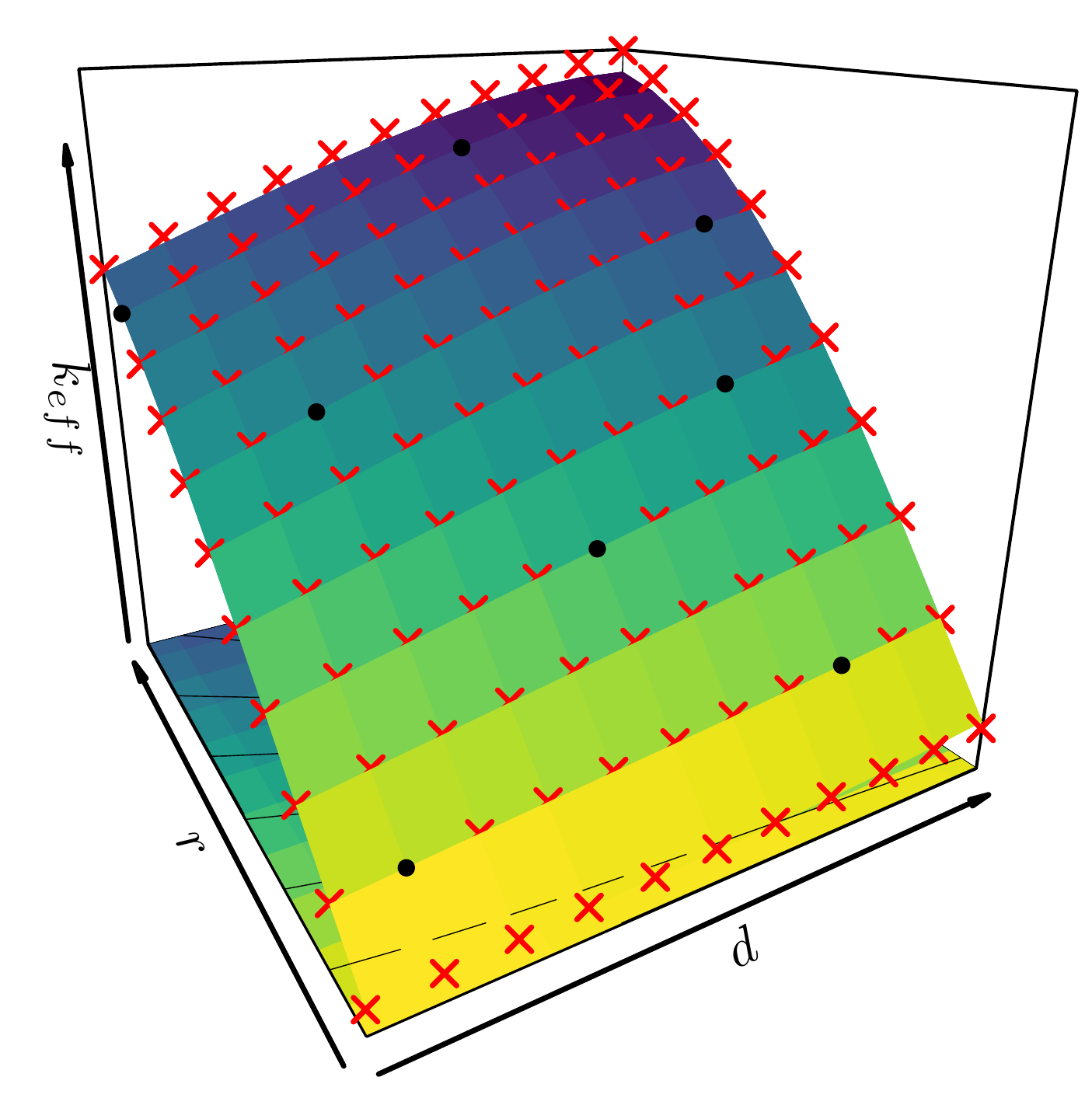}}
	\subfigure[\label{subfig:godivaExampleCKCMLEfig4} $Q^2 = 0.99949$]{\includegraphics[width=0.32\textwidth]{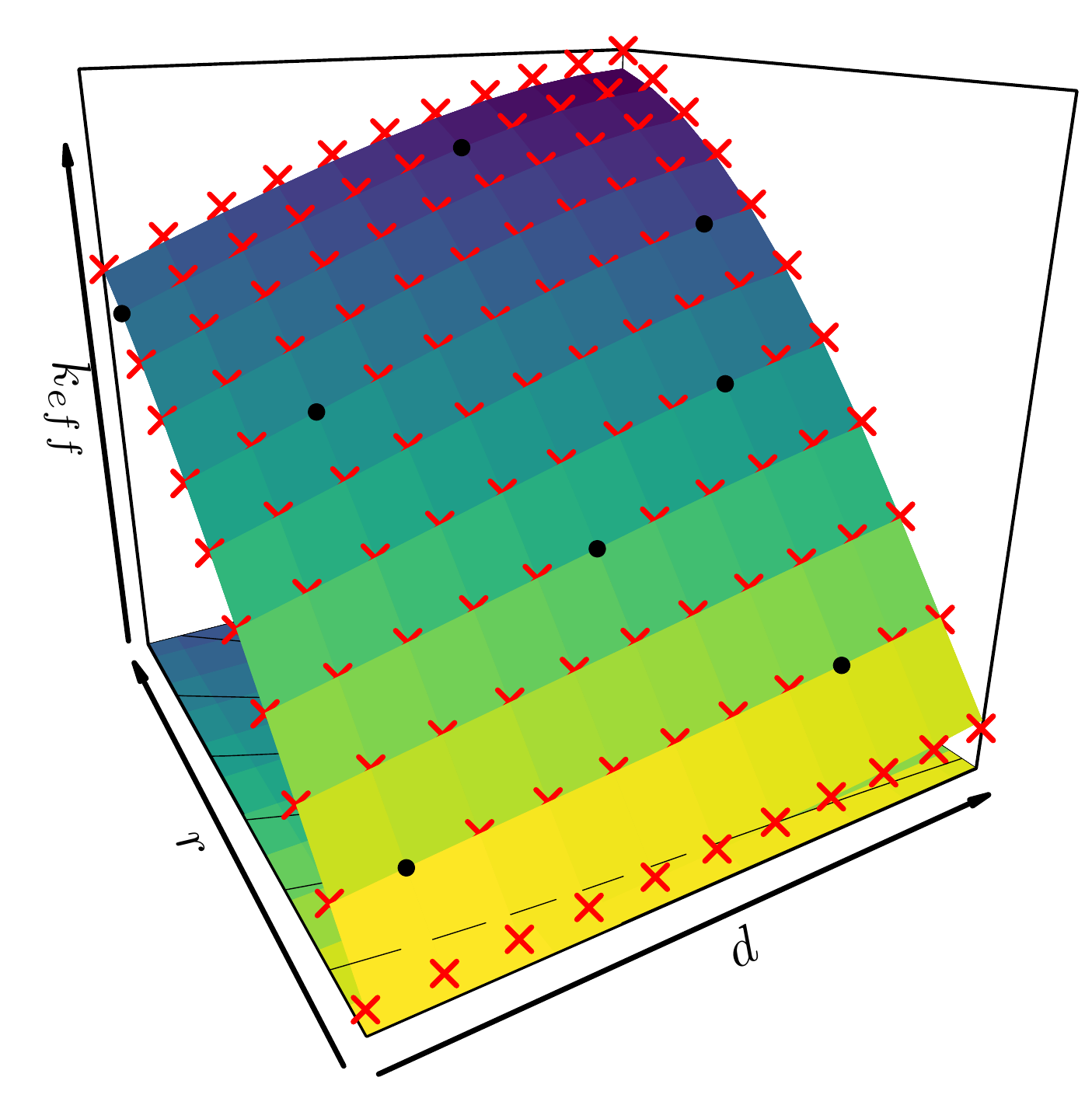}}
	
	\caption{2D Gaussian models for interpolating the Godiva's dataset. Unconstrained GP models are trained using MLE (left column) and using \subref{subfig:godivaExampleSKMLEfig2} four, or \subref{subfig:godivaExampleSKMLEfig4} eight training points from the proposed maximin Latin hypercube DoE. Constrained GP models are trained either using MLE (middle column) and CMLE (right column). Each panel shows: training and test points (black dots and red crosses), the conditional mean function (solid surface), and the $Q^2$ criterion (subcaptions).}
	\label{fig:godiva}
\end{figure}
\begin{figure}
	\centering
	\includegraphics[width=0.49\textwidth]{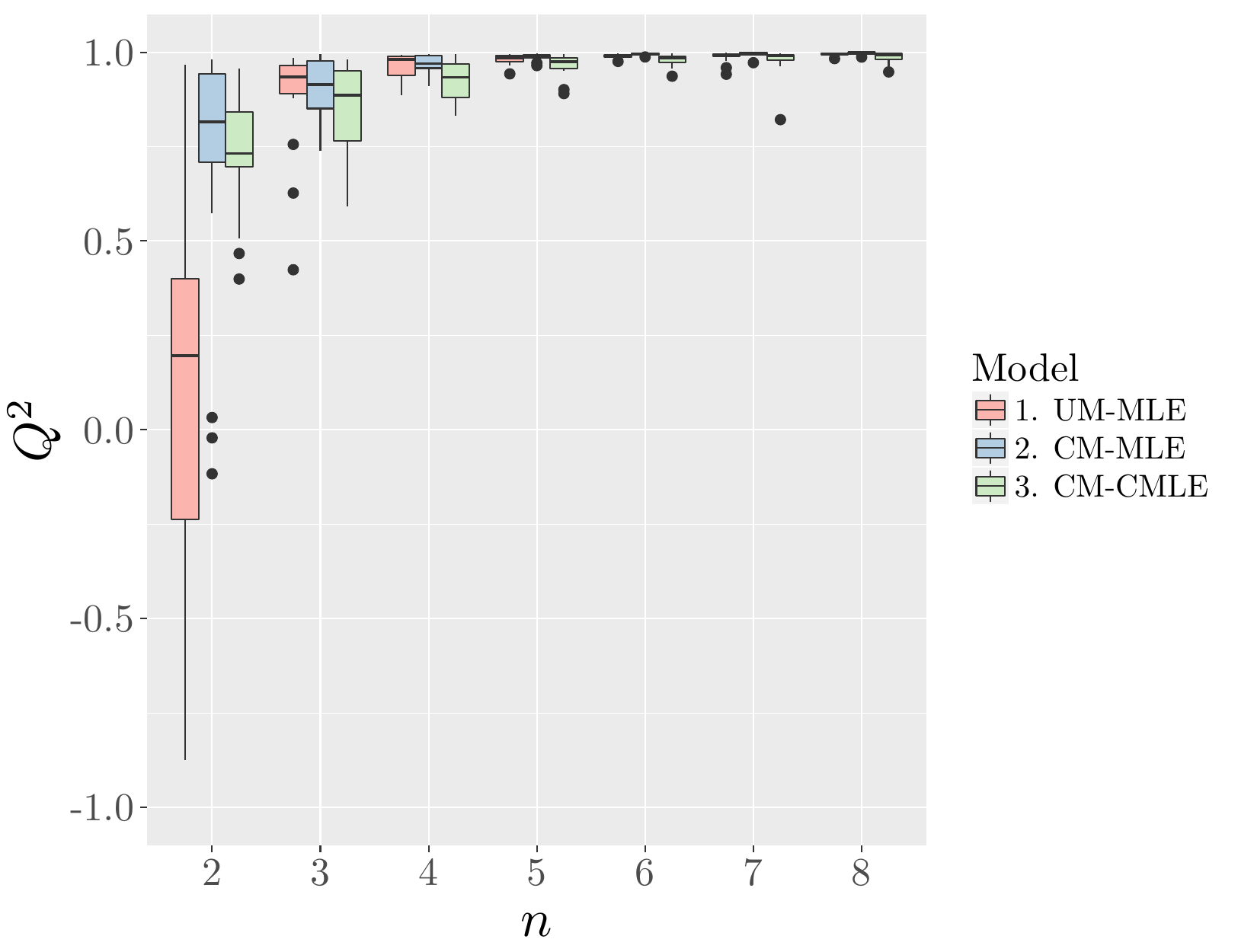}
	\includegraphics[width=0.49\textwidth]{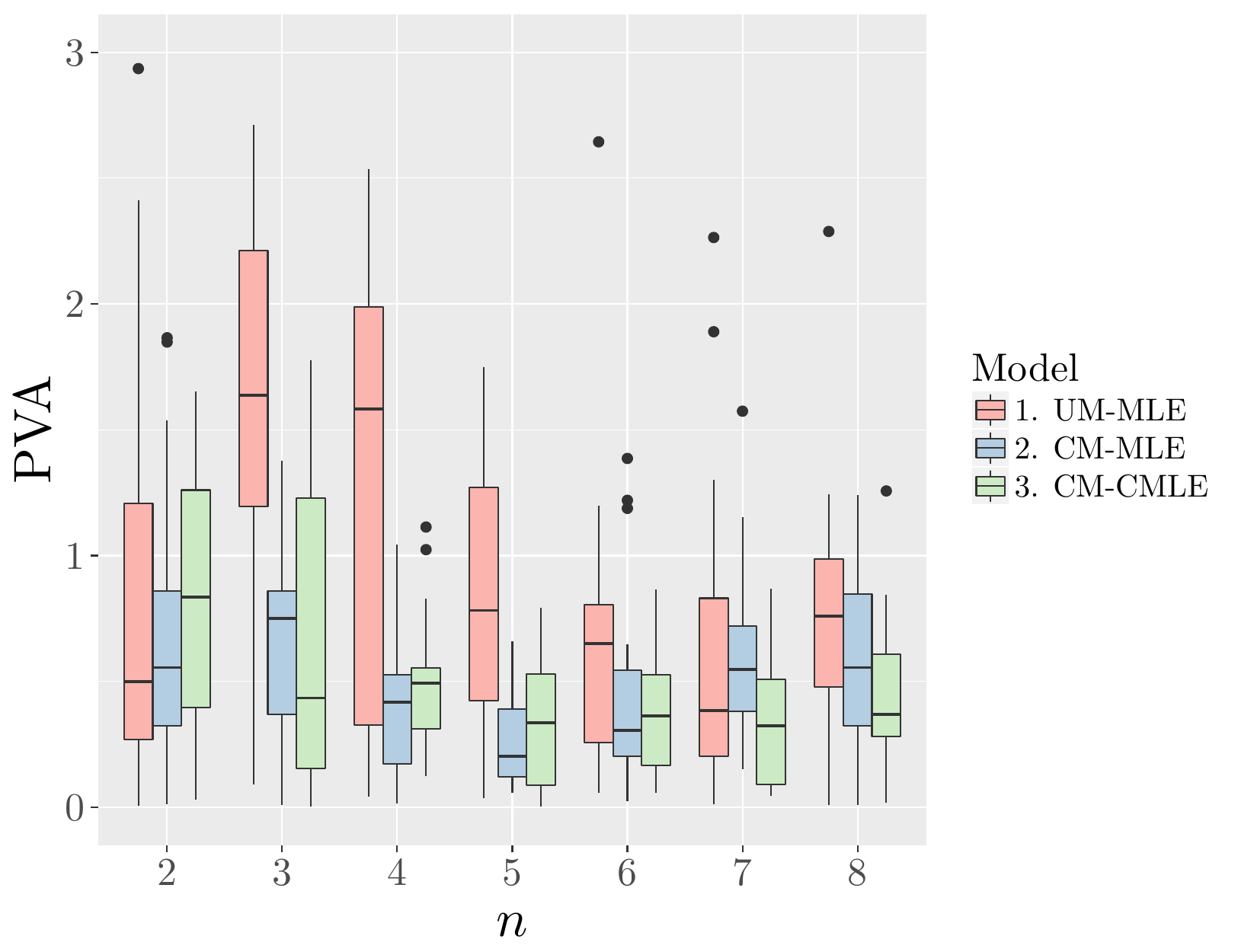}
	
	\caption{Assessment of the Gaussian models for interpolating the dataset from \Cref{fig:godivaData} using different number of training points $n$ and using twenty different random Latin hypercube designs. Predictive accuracy is evaluated using the (left) $Q^2$ and (right) PVA criteria. Results are showed for the unconstrained model (UM) using MLE (red), and the constrained models (CM) using either MLE (blue) or CMLE (green).}
	\label{fig:godivaInd}
\end{figure}

Because the prediction accuracy depends on the training set, we repeated the procedure with twenty different random Latin hypercube DoEs using several values of $n$. We used the $Q^2$ and PVA criteria to evaluate the quality of the predictions (see \Cref{subsec:cmlresults}). \Cref{fig:godivaInd} shows that the constrained models often outperform the unconstrained ones. Notice that although the $Q^2$ results obtained by the unconstrained model are comparable with the constrained ones when the number of training points is large enough, we observe, according to the PVA criterion, that the constrained models provide more reliable confidence intervals. This means that, if we consider both positivity and monotonicity conditions to take into account the physics of the $k_{\text{eff}}$ factor, we can obtain more informative and robust models. Furthermore, we observe that the unconstrained MLE achieves a good tradeoff between prediction accuracy/reliability and computational cost.

\section{Conclusions}
\label{sec:conclusions}

Continuing the approach proposed in \citep{Maatouk2016GPineqconst}, we have introduced a full Gaussian-based framework to satisfy linear sets of inequality constraints. The proposed finite-dimensional approach takes into account the inequalities for both data interpolation and covariance parameter estimation. Because the posterior distribution is expressed as a truncated multinormal distribution, we compared different MCMC methods as well as exact sampling methods. According to experiments, we concluded that the Hamiltonian Monte Carlo-based sampler adapts to our needs. For parameter estimation, we suggested the constrained likelihood which takes into account the inequalities constraints. We showed that, loosely speaking, any consistency result for ML with unconstrained GPs, is preserved when adding boundedness, monotonicity and convexity constraints. Furthermore, this consistency occurs for both the unconditional and conditional likelihood functions. 

We tested our model in both synthetic and real-world data in 1D or 2D. According to the experimental results under different types of inequalities, the proposed framework fits properly the observations and provides realistic confidence intervals. Our approach is also flexible enough to satisfy multiple inequality conditions and to deal with specific types of constraints which are sequentially activated. Finally, as we showed in the 2D nuclear criticality safety assessment, the proposed framework provides reliable predictions on both data prediction and uncertainty quantification satisfying the inequality constraints exhibited by the neutron population (positivity and monotonicity conditions). We also observe that the unconstrained MLE achieves a good tradeoff between prediction accuracy/reliability and computational cost.

The framework presented in this paper can be improved in different ways. First, the precision of the results depends on the number of knots $m$ used in the finite approximation. For higher values of $m$, the interpolation is better but more expensive. In this sense, we could consider the optimal location of the knots over the input space instead of using regular grids. This potentially allows to reduce the computational cost of the full framework. Second, as we discussed for 2D input spaces, the model can be generalized to $d$ dimensional problems. However, due to its tensor structure, its practical application could be time-consuming. Hence, there is a need to find an extension of the model to higher dimensions that can be applied in real-word problems. Finally, the estimation of the orthant Gaussian probabilities can be improved in order to exploit the advantages of the constrained likelihood.

\appendix
\section{Conditional maximum likelihood: asymptotic properties}
\label{app:asymtotic}
We detail in this appendix section the conditions and lemmas we used in \Cref{prop:ml,prop:cml} from \Cref{subsec:asymptotic}. We use the same notations and assumptions as in this subsection.
\begin{Condition}\label{cond:Balls}
	Let $\Bx, \Bx' \in \mathbb{X}$. For a fixed $\kappa \in \{0,1,2\}$, assume one of the following conditions:	
	
	- If $\kappa = 0$. 	Assume that $Y$ has continuous trajectories. Let $k$ be the covariance function of $Y$.  Let
	\begin{displaymath}
	d_{k} (\Bx,\Bx') = \sqrt{k(\Bx,\Bx) + k(\Bx',\Bx') - 2k(\Bx,\Bx')}.
	\end{displaymath}
	Let $N(\mathbb{X}, d_k,\rho)$ be the minimum number of balls with radius $\rho$ (w.r.t. the distance $d_k$), required to cover $\mathbb{X}$.
	Assume that
	\begin{displaymath}
	\int_{0}^{\infty} \sqrt{\log(N(\mathbb{X}, d_{k},\rho))} \ d\rho< \infty.
	\label{eq:intBall0}
	\end{displaymath}
	Assume also that the Fourier transform $\widehat{k}$ of $k$ satisfies
	\begin{displaymath}
	\exists \ P < \infty \quad \mbox{so that as} \quad \|w\| \to \infty, \quad \widehat{k}(w) \|w\|^P \to \infty.
	\label{eq:Fourier0}
	\end{displaymath}
	
	- If $\kappa = 1$. Assume that $Y$ has $C^1$ trajectories. Let $k_i^{[1]}$ be the covariance function of $\frac{\partial}{\partial x_i}Y$. Let $d_{k_i^{[1]}}$ and $N(\mathbb{X}, d_{k_i^{[1]}},\rho)$ be defined as $d_{k}$ and $N(\mathbb{X}, d_{k},\rho)$ for $\kappa = 0$.
	Assume that
	\begin{equation*}
	\int_{0}^{\infty} \sqrt{\log(N(\mathbb{X}, d_{k_i^{[1]}},\rho))} \ d\rho< \infty, \quad \forall i = 1,\cdots,d.
	\label{eq:intBall1}
	\end{equation*}
	Assume also that the Fourier transform $\widehat{k}_i^{[1]}$ of $k_i^{[1]}$ satisfies the same conditions as for $\kappa = 0$, for $i = 1, \cdots, d$. 
	
	- If $\kappa = 2$. Assume that $Y$ has $C^2$ trajectories. Let $k_{i,j}^{[2]}$ be the covariance function of $\frac{\partial^2}{\partial x_i \partial x_j}Y$. Let $d_{k_{i,j}^{[2]}}$ and $N(\mathbb{X}, d_{k_{i,j}^{[2]}},\rho)$ be defined as $d_{k}$ and $N(\mathbb{X}, d_{k},\rho)$ for $\kappa = 0$.
	Assume that
	\begin{equation*}
	\int_{0}^{\infty} \sqrt{\log(N(\mathbb{X}, d_{k_{i,j}^{[2]}},\rho))} \ d\rho< \infty, \quad \forall i,j = 1,\cdots,d.
	\label{eq:intBall2}
	\end{equation*}
	Assume also that the Fourier transform $\widehat{k}_{i,j}^{[2]}$ of $k_{i,j}^{[2]}$ satisfies the same conditions as for $\kappa = 0$, for $i,j = 1, \cdots, d$.
\end{Condition}

\begin{lemma}
	\label{lem:lb} 
	Let $0 \leq \ell < u \leq \infty$. Let
	\begin{equation*}
	P_{n,\ell,u}(\BY_n) = P_{\Btheta^\ast}(Y \in \mathcal{E}_0 | \ \BY_n).
	\end{equation*}
	Then, $\forall \varepsilon \geq 0$, we have
	\begin{equation*}
	P(P_{n,\ell,u}(\BY_n) \leq 1-\varepsilon \ | \ Y \in \mathcal{E}_0) \xrightarrow[n\to\infty]{} 0.
	\end{equation*}
\end{lemma}
\begin{proof}[Proof]
	From \Cref{lem:nonzero} we have $P(Y \in \mathcal{E}_0) > 0$. Hence, it is sufficient to show
	\begin{equation*}
	P(P_{n,\ell,u}(\BY_n) \leq 1-\varepsilon, \ Y \in \mathcal{E}_0) \xrightarrow[n\to\infty]{} 0.
	\end{equation*}
	The term $P_{n,\ell,u}(\BY_n)$, being a conditional expectation, is a martingale with respect to the $\sigma$-algebra generated by $Y(\Bx_1), \cdots, Y(\Bx_n)$. Furthermore, $0 \leq P_{n,\ell,u}(\BY_n) \leq 1$.
	Hence
	\begin{equation*}
	P_{n,\ell,u}(\BY_n) \xrightarrow[n\to\infty]{a.s.} P(Y \in \mathcal{E}_0\ | \ \mathcal{F}_\infty),
	\end{equation*}
	where $\mathcal{F}_\infty$ is the $\sigma$-algebra generated by $[Y(\Bx_i)]_{i \in \mathds{N}}$ using Theorem 6.2.3 from \citep{Kallenberg2002foundations}. Let $\mu_n$ and $k_n$ be the mean and the covariance function (respectively) of $Y$ given $\BY_n$. From proposition 2.8 in \citep{Bect2016MartingaleArXiv}, the conditional distribution of $Y$ given $\mathcal{F}_\infty$ is the distribution of a GP with mean function $\mu_\infty$ and covariance function $k_\infty$. Furthermore, a.s., $\mu_n$ and $k_n$ converge uniformly to $\mu_\infty$ and $k_\infty$, respectively. Hence we can show simply that, because $(x_i)_{i \in \mathds{N}}$ is dense in $\mathbb{X}$, we have a.s. $\mu_\infty = Y$ and $k_\infty$ is the zero function. Hence a.s. if $Y \in \mathcal{E}_0$ holds, then 
	\begin{displaymath}
	P(Y \in \mathcal{E}_0 \ | \ \mathcal{F}_\infty) = 1, \quad \mbox{so that} \quad P_{n,\ell,u}(\BY_n) \xrightarrow[n \to \infty]{} 1.
	\end{displaymath}
	Hence by the dominated convergence theorem
	\begin{displaymath}
	P(P_{n,\ell,u}(\BY_n) \leq 1-\varepsilon , \ Y \in \mathcal{E}_0) \xrightarrow[n\to\infty]{} 0.
	\end{displaymath}
\end{proof}

\begin{lemma}
	\label{lem:lbd}
	Let $\kappa = \{1, 2\}$.
	Let 
	\begin{displaymath}
	P_n (\BY_n) = P_{\Btheta^\ast}(Y \in \mathcal{E}_\kappa | \ \BY_n).
	\end{displaymath}
	Then, $\forall \varepsilon > 0$, we have
	\begin{displaymath}
	P(P_{n} (\BY_n) \leq 1 - \varepsilon \ | \ Y \in \mathcal{E}_\kappa) \xrightarrow[n\to\infty]{} 0.
	\end{displaymath}	
\end{lemma}
\begin{proof}[Proof]
	The proof is the same as that of \cref{lem:lb}. In particular, we remark that $\mathds{1}_{Y \in \mathcal{E}_\kappa}$ is a measurable random variable, as $Y$ has $C^\kappa$ trajectories.
\end{proof}

\begin{lemma}
	\label{lem:nonzero}
	Let $\kappa = 0$. Assume that Condition \ref{cond:Balls} is satisfied.
	Then 
	\begin{displaymath}
	P(Y \in \mathcal{E}_0) > 0, \quad \mbox{for} \quad -\infty \leq \ell < u \leq \infty.
	\end{displaymath}
\end{lemma}
\begin{proof}[Proof]
	We first prove that for all $\delta>0$
	\begin{displaymath}
	P(\forall \Bx \in \mathbb{X}: \ |Y(\Bx)| \leq \delta) > 0.
	\end{displaymath}
	This result is true and appears implicitly in the literature about small ball estimates for GP \citep{Li1999BallEstimates}. We nevertheless provide a proof of it for self-consistency. Let $(\Bv_i)_{i\in \mathds{N}}$ be a dense sequence in $\mathbb{X}$.
	Let $ \BY_v = [\begin{smallmatrix} Y(\Bv_1), & \cdots, & Y(\Bv_n) \end{smallmatrix}]^\top$.	
	Let $\mu_n$ and $k_n$ be the mean and the covariance function of $Y$ given $\BY_v$. Then we let
	\begin{displaymath}
	d_{k_n}^2 (\Bx,\Bx') = \var{(Y(\Bx)-Y(\Bx'))| \mathcal{F}_n},
	\end{displaymath}
	where $\mathcal{F}_n = \sigma(Y(\Bv_1), \cdots, Y(\Bv_n))$.
	Thus 
	\begin{displaymath}
	d_{k_n}^2 (\Bx,\Bx') = \expect{\var{(Y(\Bx)-Y(\Bx'))| \mathcal{F}_n}} \leq \var{(Y(\Bx)-Y(\Bx'))} = d_{k}^2 (\Bx,\Bx'),
	\end{displaymath}
	from the law of total variance. Hence $N(\mathbb{X}, d_{k_n}, \rho) \leq N(\mathbb{X}, d_{k}, \rho) \ \forall \rho$. Also, from Theorem 2.10 in \citep{Azais2009level} (together with an union bound and using that $\max_{\Bx \in \mathbb{X}} \ Y(\Bx)$ and $\max_{\Bx \in \mathbb{X}} \ [-Y(\Bx)]$ have the same law) we have, with $C$ an universal constant,
	\begin{align*}
	\expect{\max\limits_{\Bx \in \mathbb{X}} \ \left|Y(\Bx) - \mu_n(\Bx)\right|}
	&\leq C \int_{0}^{\infty} \sqrt{\log(N(\mathbb{X}, d_{k_n},\rho))} \ d\rho \\
	&= C \int_{0}^{2\sqrt{ \sup\limits_{\Bx \in \mathbb{X}} \ k_n(\Bx,\Bx) }} \sqrt{\log(N(\mathbb{X}, d_{k_n},\rho))} \ d\rho \\
	&\leq C \int_{0}^{2\sqrt{\sup\limits_{\Bx \in \mathbb{X}} \ k_n(\Bx,\Bx) }} \sqrt{\log(N(\mathbb{X}, d_{k},\rho))} \ d\rho.
	\end{align*}
	This last integral goes to 0 as $n \to \infty$ because $\sup_{\Bx \in \mathbb{X}} \ k_n(\Bx,\Bx) \to 0$ (see the proof of \cref{lem:lb}), and because of Condition \ref{cond:Balls}.
	Hence $\max_{\Bx \in \mathbb{X}} \ |Y(\Bx) - \mu_n(\Bx)|$ goes to 0 in probability. 	Furthermore, $\mathscr{P} =	P(\forall \Bx \in \mathbb{X}, \ -\delta \leq Y(\Bx) \leq \delta)$ satisfies
	\begin{align*}
	\mathscr{P}
	&\geq P \left(\forall \Bx \in \mathbb{X}, \ -\frac{\delta}{2} \leq \mu_n(\Bx) \leq \frac{\delta}{2}, \ -\frac{\delta}{2} \leq Y(\Bx) - \mu_n(\Bx) \leq \frac{\delta}{2}\right) \\
	&= P \left(\forall \Bx \in \mathbb{X}, \ -\frac{\delta}{2} \leq \mu_n(\Bx) \leq \frac{\delta}{2}\right) P \left(\forall \Bx \in \mathbb{X}, \ -\frac{\delta}{2} \leq Y(\Bx) - \mu_n(\Bx) \leq \frac{\delta}{2}\right),
	\end{align*}
	since the distribution of $Y-\mu_n$ does not depend on $\BY_v$.
	We now fix $n \in \mathds{N}$ for which the second probability is non-zero (the existence is guaranteed from above). Then, the first probability is non-zero by continuity since, when $\BY_v = \Bzero$, then $\mu_n$ is the zero function. 
	Hence we have 
	\begin{equation*}
	P(\forall \Bx \in \mathbb{X}: \ |Y(\Bx)| \leq \delta) > 0.
	\end{equation*}
	Let $f$ be a $C^\infty$ function on $\realset{d}$, square integrable, satisfying
	\begin{equation*}
	\forall \Bx \in \mathbb{X}, \quad \ell+\delta \leq f(\Bx) \leq u-\delta,
	\end{equation*}
	for $\delta >0$. ($f$ exists for $\delta >0$ small enough,  and can be taken for instance as $f(\Bx) = \exp\{ -\tau \|\Bx-\Bx_0\|^2\} \left[\frac{u+\ell}{2}\right]$ with $\tau>0$ small enough, and for any $\Bx_0 \in \mathbb{X}$).
	Let $Z$ be a GP with covariance function $k$ and mean function $f$. Then, from what we have shown before, we have
	\begin{displaymath}
	P(\forall \Bx \in \mathbb{X}: \ |Z(\Bx)-f(\Bx)| \leq \delta) > 0,
	\end{displaymath}
	so that
	\begin{displaymath}
	P(\forall \Bx \in \mathbb{X}: \ \ell \leq Z(\Bx) \leq u) > 0.
	\end{displaymath}
	From \citep{Yadrenko1983SpectralTheory} (p.138), as discussed by \citep{stein1999interpolation} (p.121), the Gaussian measures of $Y$ and $Z$ are equivalent. Thus
	\begin{displaymath}
	P(Y \in \mathcal{E}_0) = P(\forall \Bx \in \mathbb{X}: \ \ell \leq Y(\Bx) \leq u) > 0.
	\end{displaymath}
\end{proof}

\begin{lemma}
	\label{lem:nonzeroC1}
	Let $\kappa = 1$. Assume that Condition \ref{cond:Balls} is satisfied.
	Then
	\begin{displaymath}
	P\left(Y \in \mathcal{E}_1 \right) > 0.
	\end{displaymath}
\end{lemma}
\begin{proof}[Proof]
	We first prove that for all $\delta>0$
	\begin{displaymath}
	P\left(\forall i = 1,\cdots,d, \ \forall \Bx \in \mathbb{X} : \ \left|\frac{\partial}{\partial x_i} Y(\Bx)\right| \leq \delta \right) > 0. 
	\end{displaymath}
	We let $(\Bv_i)_{i \in \mathds{N}}$ and $\BY_v$ be defined as in the proof of \cref{lem:nonzero}. 
	Then, as in this proof we can show that for all $i=1,\cdots,d$
	\begin{displaymath}
	\max\limits_{\Bx \in \mathbb{X}} \left| \frac{\partial}{\partial x_i} Y(\Bx) - \expect{\frac{\partial}{\partial x_i} Y(\Bx) \bigg| \BY_v} \right| \xrightarrow[n \to \infty]{P} 0.
	\end{displaymath}	
	Furthermore, $\mathscr{P} = P\left(\forall i = 1,\cdots,d, \ \forall \Bx \in \mathbb{X} : \ \left|\frac{\partial}{\partial x_i} Y(\Bx)\right| \leq \delta \right)$ satisfies
	\begin{align*}
	\mathscr{P}
	&\geq P\left(\forall i = 1,\cdots,d, \ \forall \Bx \in \mathbb{X}, \ -\frac{\delta}{2} \leq \expect{\frac{\partial}{\partial x_i} Y(\Bx)\bigg|\BY_v} \leq \frac{\delta}{2}, \right. \\
	& \hspace{28pt} \left. \forall i = 1,\cdots,d, \ \forall \Bx \in \mathbb{X}, \ -\frac{\delta}{2} \leq \frac{\partial}{\partial x_i} Y(\Bx) -  \expect{\frac{\partial}{\partial x_i} Y(\Bx)\bigg|\BY_v} \leq \frac{\delta}{2} \right) \\
	& = P\left(\forall i = 1,\cdots,d, \ \forall \Bx \in \mathbb{X}, \ -\frac{\delta}{2} \leq \expect{\frac{\partial}{\partial x_i} Y(\Bx)\bigg|\BY_v} \leq \frac{\delta}{2} \right) \\
	&\times P\left(\forall i = 1,\cdots,d, \ \forall \Bx \in \mathbb{X}, \ -\frac{\delta}{2} \leq \frac{\partial}{\partial x_i} Y(\Bx) -  \expect{\frac{\partial}{\partial x_i} Y(\Bx)\bigg|\BY_v} \leq \frac{\delta}{2} \right).
	\end{align*}
	Notice that the last equality holds because the distribution of the process $\Bx \to \frac{\partial}{\partial x_i} Y(\Bx) -  \expect{\frac{\partial}{\partial x_i} Y(\Bx) | \BY_v}$ does not depend on $\BY_v$. 	
	We now fix $n\in \mathds{N}$ so that the second probability is non-zero (the existence is guaranteed from above). 
	Then, the first probability is non-zero by continuity since, when $\BY_v = \Bzero$, then for all $i = 1, \cdots, d$, $\ \expect{\frac{\partial}{\partial x_i} Y|\BY_v}$ is the zero function.
	Hence, we have  obtained
	\begin{displaymath}
	P\left(\forall i = 1,\cdots,d, \ \forall \Bx \in \mathbb{X} : \ \left|\frac{\partial}{\partial x_i} Y(\Bx)\right| \leq \delta \right).
	\end{displaymath}
	We now conclude the proof in the same way as for \cref{lem:nonzero}.
	We consider the mean function
	\begin{displaymath}
	f(\Bx) = \left[\sum_{i=1}^{d} x_i\right]\exp\{ -\tau \|\Bx-\Bx_0\|^2\},
	\end{displaymath}
	with $\Bx_0 \in \mathbb{X}$ and $\tau > 0$.
	For $\tau$ small enough, $f$ is $C^\infty$, square integrable, and satisfies
	\begin{displaymath}
	\forall i = 1,\cdots,d, \quad \forall \Bx \in \mathbb{X}, \quad \frac{\partial}{\partial x_i} f(\Bx) \geq \frac{1}{2}.
	\end{displaymath}	
	Then, we conclude the proof as in the proof of \cref{lem:nonzero}.
\end{proof}

\begin{lemma}
	\label{lem:nonzeroC2}
	Let $\kappa = 2$. Assume that Condition \ref{cond:Balls} is satisfied. Then,
	\begin{displaymath}
	P\left(Y \in \mathcal{E}_2 \right) > 0. 
	\end{displaymath}
\end{lemma}
\begin{proof}[Proof]
	We first prove that for all $\delta>0$
	\begin{displaymath}
	P\left(\forall i,j = 1,\cdots,d, \ \forall \Bx \in \mathbb{X} : \ \left|\frac{\partial^2}{\partial x_i \partial x_j} Y(\Bx)\right| \leq \delta \right) > 0. 
	\end{displaymath}
	This is done in a similar way as for showing $P\left(\forall i = 1,\cdots,d, \ \forall \Bx \in \mathbb{X} : \ \left|\frac{\partial}{\partial x_i} Y(\Bx)\right| \leq \delta \right) > 0$ in the proof of \cref{lem:nonzeroC1}.
	We then conclude similarly as the rest of the proof this Lemma. In particular, we consider the mean function
	\begin{displaymath}
	f(\Bx) = \left[\sum_{i=1}^{d} x_i^2\right]\exp\{ -\tau \|\Bx-\Bx_0\|^2\},
	\end{displaymath}
	with $\Bx_0 \in \mathbb{X}$ and $\tau>0$.
	Let $\lambda_\inf(M)$ be the smallest eigenvalue of a symmetric matrix $M$. Then, for $\tau$ small enough, $f$ is $C^\infty$, square integrable, and satisfies
	\begin{displaymath}
	\forall \Bx \in \mathbb{X}, \quad \lambda_\inf \left(\frac{\partial^2}{\partial \Bx^2} f(\Bx)\right) \geq 1.
	\end{displaymath}
\end{proof}

\begin{lemma}
	\label{lem:continuity}
	Let $\kappa \in \{0,1,2\}$.
	Assume that Condition \ref{cond:Balls} holds.
	%	Let $\BTheta$ be compact in $(0,\infty)^d$.
	Let $Y_\Btheta$ be the GP defined by
	\begin{displaymath}
	Y_\Btheta(t) = \sigma Y(\theta_1 t_1, \cdots, \theta_d t_d).
	\end{displaymath}
	Let $P_\Btheta^\kappa = P(Y_\Btheta \in \mathcal{E}_\kappa)$ (see \Cref{eq:convexsetFun2}). Then,
	\begin{displaymath}
	\inf\limits_{\Btheta \in \BTheta} \ P_\Btheta^\kappa >0.
	\end{displaymath}
\end{lemma}
\begin{proof}[Proof]
	We do the proof for $\kappa=2$. The proof for $\kappa=0,1$ is similar.
	Let $\varepsilon > 0$ and let 
	\begin{displaymath}
	P_{\Btheta,\varepsilon}^\kappa = \expect{\mathds{1}_{I(Y_\Btheta)\geq \varepsilon} + \frac{I(Y_\Btheta)}{\varepsilon} \mathds{1}_{0 \leq I(Y_\Btheta)\leq \varepsilon}},
	\end{displaymath}
	with $I(Y_\Btheta) = \inf_{\Bx \in \mathbb{X}} \ \lambda_\inf \left(\frac{\partial^2}{\partial \Bx^2} Y_\Btheta (\Bx)\right)$.
	We have $P_{\Btheta,\varepsilon}^\kappa \leq P_{\Btheta}^\kappa$ for all $\varepsilon > 0$.
	With the proof of \cref{lem:nonzeroC2}, we also obtain for $\varepsilon>0$ small enough
	\begin{displaymath}
	\forall \Btheta \in \BTheta \quad P_{\Btheta,\varepsilon}^\kappa > 0.
	\end{displaymath}
	Hence, the proof is concluded, by compacity, if we show that $\Btheta \to P_{\Btheta,\varepsilon}^\kappa$ is a continuous function on $\BTheta$. Let us show this. Let $\Btheta = (\sigma_1^2, \theta_1, \cdots, \theta_d) \in (0,\infty)^{d+1}$ and $\Btheta_{n} = (\sigma_{n}^2, \theta_{n1}, \cdots, \theta_{nd}) \to \Btheta$. We have
	\begin{displaymath}
	\frac{\partial^2}{\partial x_i x_j} Y_{\Btheta_n}(\Bx) = \sigma_n \left((\theta_n)_i (\theta_n)_j \frac{\partial^2}{\partial x_i x_j} Y(\theta_{n1} x_1, \cdots, \theta_{nd} x_d) \right).
	\end{displaymath}
	Hence, because $Y$ is $C^2$, we have a.s.
	\begin{displaymath}
	\sup\limits_{\Bx \in \mathbb{X}} \left|\left| \frac{\partial^2}{\partial \Bx^2} Y_{\Btheta_n}(\Bx) - \frac{\partial^2}{\partial \Bx^2} Y_{\Btheta}(\Bx) \right|\right| \xrightarrow[n \to \infty]{} 0,
	\end{displaymath}
	for any matrix norm $\|\cdot\|$. Hence also since $Y$ is $C^2$, we can show, a.s.
	\begin{displaymath}
	\left(\inf\limits_{\Bx \in \mathbb{X}} \ \lambda_\inf \left(\frac{\partial^2}{\partial \Bx^2} Y_{\Btheta_n}(\Bx)\right) - \inf\limits_{\Bx \in \mathbb{X}} \ \lambda_\inf \left(\frac{\partial^2}{\partial \Bx^2} Y_{\Btheta}(\Bx) \right)\right) \xrightarrow[n \to \infty]{} 0.
	\end{displaymath}
	Hence, we conclude by dominated convergence observing that $t \to (\mathds{1}_{t \geq \varepsilon} + \frac{t}{\varepsilon} \mathds{1}_{0 \leq t \leq \varepsilon})$ is a continuous function on $\realset{}$.
\end{proof}

\section*{Acknowledgement}
This research was conducted within the frame of the Chair in Applied Mathematics OQUAIDO, gathering partners in technological research (BRGM, CEA, IFPEN, IRSN, Safran, Storengy) and academia (CNRS, Ecole Centrale de Lyon, Mines Saint-Etienne, University of Grenoble, University of Nice, University of Toulouse) around advanced methods for Computer Experiments. We thank Yann Richet (IRSN) for providing the nuclear criticality safety data. 

%bibliography------------------------------------------------------------------------------------
%\nocite{*}
%\newpage
\bibliographystyle{apa}
\bibliography{arXiv2017}

\end{document}